%% file: paper-arxiv.tex
\documentclass[a4]{article}

\input{packages}
\usepackage{graphicx}
\usepackage{amssymb}
\input{macros}

\usepackage{geometry}

\def\arxiv{} 
\def\ackn{} 

\ifdefined\draft
\newcommand{\note}[1]{}
\newcommand{\remove}[1]{}
\else
\newcommand{\note}[1]{\red{(Memo: #1)}} 
\newcommand{\remove}[1]{
\begingroup
\color{blue}
#1
\endgroup
}
\fi

\title{
\input{title}
}

\author[1]{Satoshi~Katayama}
\author[1]{Shoyo~Hunt}
\author[1]{Shintaro~Masuda}
\author[1]{Masayuki~Karasuyama\thanks{karasuyama@nitech.ac.jp}}
\affil[1]{Nagoya Institute of Technology}

\date{}

\newcommand{\biblstyle}{\bibliographystyle{apalike}}

\begin{document}

\maketitle

\input{abstract}
\input{main}

\end{document}

%% file: packages.tex
\usepackage{amsthm}
\usepackage{bm}
\usepackage{amsmath}
\usepackage{mathtools}
\usepackage{subfig}
\usepackage{xcolor}
\usepackage{here}
\usepackage{natbib}
\usepackage{cite}
\usepackage{url}
\usepackage{authblk}
\usepackage{listings}
\usepackage{wrapfig}

\usepackage[globalcitecopy]{bibunits}

\defaultbibliographystyle{abbrvnat}
\defaultbibliography{ref}

\usepackage[ruled,linesnumbered]{algorithm2e}
\SetKwBlock{Function}{function}{end function}



%% file: macros.tex
\newcommand{\memo}[1]{} 

\def\eqref#1{equation~\ref{#1}}

\ifdefined\submit
\newcommand{\red}[1]{#1}
\else
\newcommand{\red}[1]{\textcolor{red}{#1}}
\fi

\def\*#1{\boldsymbol{#1}}

\newtheorem{theo}{Theorem}[section]

\newtheorem{lem}{Lemma}[section]
\newtheorem{coro}{Corollary}[section]

\newtheorem{prop}{Proposition}[section]

\newcommand{\prm}[0]{\prime}
\newcommand{\tw}[0]{\textwidth}
\newcommand{\ig}[2]{\includegraphics[clip,width=#1\tw]{#2}}

\newcommand{\veps}[0]{\varepsilon}

\newcommand{\mr}[1]{\mathrm{#1}}

\newcommand{\argmax}{\mathop{\text{argmax}}}
\newcommand{\argmin}{\mathop{\text{argmin}}}
\newcommand{\eq}[1]{(\ref{#1})}

\newcommand{\lw}[1]{\smash{\lower2.ex\hbox{#1}}}

\newcommand{\RR}{\mathbb{R}}

\newcommand{\EE}{\mathbb{E}}

\newcommand{\cD}{{\mathcal D}}

\newcommand{\cG}{{\mathcal G}}

\newcommand{\cN}{{\mathcal N}}

\newcommand{\cP}{{\mathcal P}}

\newcommand{\cU}{{\mathcal U}}

\newcommand{\cX}{{\mathcal X}}

%% file: title.tex
Pre-training of Bayesian Optimization Algorithm through \\
Bayesian Optimization 


%% file: abstract.tex
\begin{abstract}
Bayesian optimization (BO) is widely used as a standard approach for expensive black-box optimization. 
However, BO algorithms often involve parameters that must be specified in advance, and their performance can strongly depend on these choices. 
We propose a framework for optimizing such parameters using sample paths drawn from a Gaussian process (GP) inferred from the information available at the start of BO.
We use cumulative regret as the performance metric for a BO algorithm. 
By running the BO algorithm on the generated sample paths, we obtain an empirical estimate of its expected cumulative regret for a given parameter configuration. 
Optimizing this estimate allows us to identify parameter configurations that, given the currently available information, are expected to achieve low cumulative regret. 
Since this parameter optimization is itself a black-box optimization problem, we employ another BO procedure to solve it, which we refer to as outer BO. 
Through experiments, we demonstrate that the proposed framework can effectively select parameter configurations that achieve strong performance among a range of candidate configurations.
\end{abstract}

%% file: main.tex
\section{Introduction}


Bayesian optimization (BO) is a widely used approach for optimizing expensive black-box functions with a limited number of function evaluations.
It has been successfully applied to a broad range of problems, including experimental design, hyperparameter optimization, and scientific and engineering design.
%
%
However, while BO automates the optimization of the objective function, BO algorithms themselves often involve various algorithmic configuration parameters that must be specified before optimization begins.
Examples include parameters controlling the exploration--exploitation trade-off in acquisition functions and design parameters of search strategies.
%
Such choices can substantially affect optimization performance.
Unlike the surrogate model, i.e., usually a Gaussian process (GP), hyperparameters, these algorithmic parameters cannot generally be estimated by the (marginal) likelihood of the surrogate and are therefore often set to default or heuristic values.


The basic question in this study is whether the same criterion used to evaluate the performance of a BO algorithm can also be used to select its configuration parameters.
We consider cumulative regret over a finite horizon as the performance criterion.
The cumulative regret with respect to the unknown true objective function is, of course, unavailable before running BO.
We assume that, at the beginning of BO, we have a probability distribution over the objective function that reflects the available initial observations and/or prior knowledge.
We use the expected cumulative regret under this distribution, i.e., the Bayes cumulative regret, as the criterion. 
By sampling objective functions from this distribution and actually running a candidate BO algorithm on each sample path, we can obtain a Monte Carlo (MC) approximation of the expected cumulative regret. 
This allows us to formulate the selection of BO algorithm parameters as an optimization problem that directly targets finite-horizon performance under the currently available knowledge about the objective function.


An important feature of this approach is that prior knowledge about the objective function is used not only for prediction by the surrogate model but also for configuring the BO algorithm itself.
In the simplest setting, the GP posterior conditioned on an initial dataset can be used as the distribution over objective functions.
When data from related past tasks are available, a GP prior pre-trained on such data can instead be incorporated.
Our framework therefore does not rely on a particular method for constructing the prior distribution; rather, it can be viewed as a general mechanism for translating an available distribution over objective functions into a configuration of a BO algorithm.
Recent work on learning to optimize and meta-learning has proposed methods that learn optimizers or acquisition policies themselves from multiple optimization tasks \citep[e.g.,][]{chen2022learning,volpp2020meta}.
In contrast, our approach retains the structure of an existing BO algorithm and tunes only a small number of its algorithmic parameters.
%
This design aims to adapt the algorithm to the target problem distribution while \emph{retaining established acquisition rules and principles developed in the BO literature}.
We discuss these connections in more detail in Section~\ref{sec:related-work}.


The MC approximation of the cumulative regret is generally a complicated black-box function of the parameter configuration, for which an analytical expression or gradient is difficult to obtain.
%
Moreover, evaluating a single configuration requires repeatedly running finite-horizon BO on multiple sample paths, making exhaustive evaluation impractical.
We therefore solve this algorithm-configuration problem using another BO procedure.
We refer to the conventional BO procedure for optimizing the objective function as \emph{inner BO}, and to the BO procedure for optimizing its algorithmic parameters as \emph{outer BO}.
This construction makes the framework applicable to arbitrary parameterized BO algorithms without requiring an analytical characterization of either their internal mechanisms or the dependence of regret on their parameters.
We further extend the framework to a multi-fidelity setting, where evaluations obtained with shorter inner-BO horizons are used as lower-fidelity observations, and to a multi-stage configuration setting, where different parameter values can be used at different stages of the optimization.


We also analyze the statistical properties of the proposed MC-based evaluation of BO algorithms.
We show that the MC estimator of cumulative regret has approximation error $O_p(T/\sqrt{N})$, where $T$ is the number of inner-BO iterations and $N$ is the number of MC samples.
When the candidate parameter configurations form a finite set, we further establish a guarantee of the same order for the difference in expected cumulative regret between the configuration selected by minimizing the MC approximation and the configuration minimizing the true Bayes cumulative regret.
These results do not depend on the specific form of the acquisition function and provide theoretical justification for our simulation-based evaluation framework.


Our main contributions are summarized as follows:
\begin{itemize}
\setlength{\itemsep}{0pt}
    \item We propose a general framework for selecting pre-specified parameters of a BO algorithm using its expected cumulative regret under a probability distribution over objective functions.
	  %
	  We approximate this criterion by running BO on sample paths and use outer BO to optimize the resulting black-box objective.
	  We further extend the framework to multi-fidelity optimization using the inner-BO horizon as a fidelity variable and to multi-stage configurations in which algorithmic parameters can change during optimization.

    \item 
	  We establish the $O_p(T/\sqrt{N})$ convergence of the MC cumulative-regret estimator.
	  For a finite candidate configurations, we further provide a guarantee on the expected cumulative regret of the configuration selected using the MC estimator.

    \item Through experiments, we demonstrate that the proposed framework can identify parameter configurations with strong finite-horizon BO performance among a set of candidate configurations.
\end{itemize}

\section{Bayesian Optimization (BO)}

\subsection{Preliminary}

We consider the maximization problem of an unknown function $f: \cX \rightarrow \RR$, whose input domain is $d$-dimensional compact space 
$\cX \subset \RR^d$, i.e., $\max_{\*x \in \cX} f(\*x)$.
%
%
An observation for $\*x_i$ is denoted as $y_i = f(\*x_i) + \varepsilon$, where $\veps$ is an observation noise that follows
$\veps \sim \cN\left(0, \sigma_{\text {noise}}^2\right)$.

In Bayesian optimization (BO), the Gaussian process (GP) is typically used to model the unknown objective function $f$, denoted as
\begin{align*}
 f \sim \cG\cP(m, k),
\end{align*}
where $m$ is a prior mean function and $k$ is a prior covariance function (kernel function), both of which are assumed to be continuous.
%
Let 
$\cD_t$
be the dataset at the $t$-th iteration of BO.
%
The posterior of $f$ is 
\begin{align*}
 f \mid \cD_t \sim \cG\cP(\mu_t, \Sigma_t),
\end{align*}
%
where 
$\mu_t(\*x) = \EE[ f(\*x) \mid \cD_t ]$
and
$\Sigma_t(\*x,\*x^\prm) = \mr{Cov}(f(\*x), f(\*x^\prm) \mid \cD_t)$
are the posterior mean and covariance functions, respectively.
%
By maximizing an acquisition function $a$, the next observation point is determined as
\begin{align*}
 \max_{\*x \in \cX} a(\*x). 
\end{align*}
%
Well known acquisition functions are for example EI (Expected Improvement) and GP-UCB (Gaussian process Upper Confidence Bound) \citep{Srinivas2010-Gaussian}.
%
Throughout the paper, we assume that the number of the maximum iterations of BO is $T$.


\subsection{Parameter Configuration of BO}

For a variety of problem settings, a lot of BO algorithms have been proposed \citep{wang2023recent}. 
However, many of BO algorithms need configurations of parameters that are required to determine before starting the optimization.  
%
Let 
$\*\theta \in \Theta \subset \RR^p$
be those parameters of a BO algorithm.
%
Examples of $\*\theta$ are as follows.

{\bf EI.} EI is often defined as the expected improvement compared with the maximum value among the observed $y_i$ in the current dataset. 
       %
       However, this threshold actually can be set arbitrary \citep{brochu2010tutorial}, which obviously has substantial effect on the results. 
       %

{\bf TuRBO.} 
TuRBO (Trust Region BO) proposed by \citep{eriksson2019scalable} is a well-known scalable BO method that performs efficient global optimization by adaptively managing multiple local trust regions (TRs), each modeled by an independent GPs.
       %
       However, for example, the sizes of TRs and their change scheduling can be tuning parameters that affect the final performance. 

{\bf LSBO.}
LSBO (Latent Space BO) \citep{gomez2018automatic} performs BO in a low-dimensional latent space learned by a generative model, enabling efficient optimization over complex structured objects (such as molecules) or a higher dimensional original space.
       %
       For the latent space embedding, variational auto-encoder (VAE) \citep{kingma2014autoencoding} is typically used.
       %
       However, the estimation in a VAE involves several design choices, such as balancing parameters of the reconstruction loss and the regularizations (such as a Gaussianity and independence in the latent space), and the dimensionality of the latent space.

%
%
Note that these parameters are configurations of a BO algorithm itself, and are not parameters of GPs, by which the standard hyper-parameter tuning by the marginal likelihood is not applicable. 
%
Because of this difficulty, these configurations are often set heuristically.
%
However, an appropriate way of a configuration selection is obviously important for maximally accelerating the optimization.
%
Algorithm~\ref{alg:inner-BO} shows the process of BO algorithm.
%
In the scenarios described above, typically, line 3-4 of the algorithm (surrogate fitting and acquisition function maximization) depends on $\*\theta$.

\begin{algorithm}[t] 
 \caption{BO algorithm} \label{alg:inner-BO}
 \Function(BO{(}$f, \*\theta, T${)}){
  \For{$i = 0, \ldots, T$}{ %
   Fit surrogate model (GP) to $\cD_i$ \\
   $\*x_t \leftarrow \argmax_{\*x \in \cX} a(\*x)$ \\
   Query observation $y_t \leftarrow f(\*x_t) + \veps$ \\
   $\cD_{t+1} \leftarrow \cD_t \cup \{(\*x_t, y_t)\}$
   } 
  } 
\end{algorithm}

\section{Proposed Method: Pre-trained BO}

\subsection{Evaluation of BO Algorithm}

%
The basic idea of our approach is that the evaluation criterion of an entire BO algorithm should be optimized to select an appropriate $\*\theta$.
%
For this purpose, we employ the cumulative regret, which is often used in empirical evaluation of BO algorithms, 
\begin{align*}
 R_T \coloneqq \sum_{t = 1}^T (f(\*x^\star) - f(\hat{\*x}_t)),
\end{align*}
where 
$\*x^\star = \argmax_{\*x \in \cX} f(\*x)$
and 
$\hat{\*x}_t$
is the recommended point at the $t$-th iteration. 
%
For evaluation of BO algorithms, typically, $\hat{\*x}_t$ is defined as the maximizer of the acquisition function, or the maximizer of the posterior mean is also often used (called inference regret). 
%
%

Our proposed method, called pre-trained BO, considers minimizing $R_T$ with respect to $\*\theta$.
%
However, for this purpose, we need to evaluate $R_T$ before starting the actual BO procedure.
%
Therefore, we use the expected cumulative regret (also called Bayes regret) measured by the initial posterior 
$f \mid \cD_0 \sim \cG\cP(\mu_0, \Sigma_0)$, 
written as
\begin{align*}
 \mr{BR}_T(\*\theta) 
 \coloneqq 
 \EE[R_T \mid \cD_0]. 
\end{align*}
%
The expectation here is about $f$, $\veps$, and other randomness contained in the BO algorithm (if exists).
%
Since $\mu_0$ and $\Sigma_0$ are estimated by the initial dataset $\cD_0$, they are available before starting the BO loop.
%
The following optimization problem characterizes the best $\*\theta$ in a sense of the expected cumulative regret defined by the current knowledge of the optimization task. 
\begin{align}
 \min_{\*\theta \in \Theta} \ \mr{BR}_T(\*\theta) 
 \label{eq:maximize-expected-regret}
\end{align}
%
The posterior $\mu_0$ and $\Sigma_0$ can be purely estimated from $\cD_0$, but richer prior knowledge can also be incorporated through the prior mean $m$ and covariance $k$. 
%
We consider the following two typical scenarios.
\begin{itemize}
 \setlength{\itemsep}{0pt}
 \item 
The first scenario is that $m$ are manually created from some expert knowledge. 
%
In some application problems, human experts sometimes have background knowledge about $f$.
%
For example, when $f$ is an expensive black-box function determined by quantum mechanics calculation, simpler explicit function from classical mechanics can be used as $m$ \citep[e.g.,][]{deringer2021gaussian}. 
 \item 
The second scenario is when prior (past) datasets are available. 
Pre-trained GP (Wang et al., 2024) is proposed for this setting, estimating $m$ and $k$ under the assumption that the past datasets share a common prior.
For example, Wang et al. (2024) demonstrate this setting in AutoML, where historical hyperparameter configurations and their performance on related tasks are used to learn a prior for a new task. 
Another plausible setting is scientific experimentation, where conditions and outcomes from related past experiments could similarly provide a prior for optimizing a new experiment.
\end{itemize}

\subsection{Monte-Calro Approximation and Optimization}

Since the expectation in 
$\mr{BR}_T$
is computationally intractable, we employ the Monte-Carlo (MC) approximation. 
%
Let 
$f^{(1)}, \ldots, f^{(N)}$
be $N$ sample paths generated from 
$\cG\cP(\mu_0, \Sigma_0)$.
%
By performing BO $T$ iterations for each of $f^{(i)}$, we can obtain a realization of $R_T$, denoted as 
$R_T^{(i)} = \sum_{t = 1}^T (f^{(i)}(\*x_{(i)}^\star) - f^{(i)}(\hat{\*x}_t))$,
where
$\*x_{(i)}^* = \argmax_{\*x \in \cX} f^{(i)}(\*x)$.
%
%
Then, our MC estimator for
$\mr{BR}_T(\*\theta)$
is
\begin{align}
 \widehat{\mr{BR}}_T(\*\theta) 
 \coloneqq \frac{1}{N} \sum_{i = 1}^N  R_T^{(i)}. 
 \label{eq:empirical-regret}
\end{align}


As a result, we solve the following optimization problem instead of \eq{eq:maximize-expected-regret}:
\begin{align}
 \min_{\*\theta \in \Theta} \ 
 \widehat{\mr{BR}}_T(\*\theta). 
 \label{eq:maximize-empirical-regret}
\end{align}
%
This objective function 
$\widehat{\mr{BR}}_T(\*\theta)$
is a complicated black-box function for which analytical representation and derivative are difficult to obtain.
%
Further, to evaluate 
$\widehat{\mr{BR}}_T(\*\theta)$
for a given $\*\theta$, we need to run a $T$-iteration BO for $N$ times, performing exhaustive or massive evaluations is not practical.
%
Therefore, we also apply BO to the optimization of \eq{eq:maximize-empirical-regret}. 
%
We call the optimization of $\*\theta$ outer BO and the optimization of $\*x$ inner BO.

Algorithm~\ref{alg:outer-BO} shows the procedure of outer BO.
%
A dataset for outer BO is a set of 
$(\*\theta, \widehat{\mr{BR}}_T(\*\theta))$, 
denoted as $\cD^{\rm outer}$.
%
By fitting a GP to $\widehat{\mr{BR}}_T(\*\theta)$, an acquisition function for outer BO
$a^{\rm outer}(\*\theta)$
can be defined.
%
By performing inner BO $N$ times, the cumulative regret $R_T^{(j)}$ for the $j$-th sample path is evaluated at line 7.
%
Although 
$\widehat{\mr{BR}}_T(\*\theta) = \frac{1}{N} \sum_{i = 1}^N R_T^{(i)}$
is optimized here, we can replace $R_T^{(i)}$ with $-\sum_{t=1}^T f^{(i)}(\hat{\*x}_t)$. 
%
This replacement only shifts  
$\widehat{\mr{BR}}_T$
with a constant 
$\frac{1}{N}\sum_{i=1}^N T f^{(i)}(\*x_{(i)}^*)$, 
and it does not affect the optimal $\*\theta$, while we can omit the calculation of $\*x_{(i)}^*$ by the replacement.
%
The same essential conclusion holds for the following theoretical analysis, regardless of which version is considered.

\begin{algorithm}[t] 
 \caption{Pre-trained BO} \label{alg:outer-BO}
  Generate $N$ sample paths $f^{(1)}, \ldots, f^{(N)}$ from $\cG\cP(\mu_0, \Sigma_0)$ \\
  Create initial $\cD^{\rm outer}$ \\
  \For{$i = 1, \ldots, T^{\rm outer}$}{ 
   Fit a GP to $\cD^{\rm outer}$ \\
   $\*\theta \leftarrow \argmax_{\*\theta \in \Theta} \ a^{\rm outer}(\*\theta)$ \\
   \For{$j = 1, \ldots, N$}{ 
     Evaluate $R_T^{(j)}$ by running $\mr{BO}(f^{(j)}, \*\theta, T)$ \\
   } 
   $\widehat{\mr{BR}}_T = \frac{1}{N} \sum_{j = 1}^N R_T^{(j)}$ \\
   $\cD^{\rm outer} \leftarrow \cD^{\rm outer} \cup \{ (\*\theta, \widehat{\mr{BR}}_T) \}$
   } 
\end{algorithm}











\subsection{Theoretical Properties}

We first consider the convergence of the MC estimator \eq{eq:empirical-regret}.
%
In general, the approximation error of the MC estimator is typically $O(1/\sqrt{N})$ under some regularity condition such as the finite variance \citep[e.g.,][]{glasserman2004monte}. 
%
Since a value of $f$ that follows a GP is not finitely bounded, it is not immediately evident whether the variance of $\mr{BR}_T$ is finite.
%
The following proposition guarantees the convergence of $\widehat{\mr{BR}}_T$: 
%
\begin{prop} \label{prop:one-shot-error}
Assume that the kernel function is continuous, $\cX$ is compact, and there exist constants $L < \infty$ and $\alpha > 0$ that satisfy 
\begin{align}
 d(\*x,\*x^\prm)
 \coloneqq
 \sqrt{
 \EE[
 \{f(\*x)-f(\*x^\prm)\}^2
 ]
 }
 \leq
 L \|\*x - \*x^\prm\|_2^\alpha
 \label{eq:metric}
\end{align}
for any sufficiently small $\| \*x - \*x^\prm \|_2$.
%
Then, for a given $\*\theta$, 
\begin{align*}
 \EE\left[
 \left.
 \left(
 \widehat{\mr{BR}}_{T}(\*\theta)
 -
 \mr{BR}_T(\*\theta)
 \right)^2
 \right|
 \mathcal D_0
 \right]
 \le
 \frac{4 C_{\cD_0} T^2}{N} 
\end{align*}
and 
\begin{align*}
 \widehat{\mr{BR}}_{T}(\*\theta)
 -
 \mr{BR}_T(\*\theta)
 =
 O_p\left( \frac{T}{\sqrt N} \right), 
\end{align*}
where 
$C_{\cD_0}
 \coloneqq
 \EE\left[
 \left.
 \left(
 \sup_{\*x \in \cX}| f(\*x) |
 \right)^2
 \right|
 \cD_0
 \right]
 < \infty$. 
\end{prop}
\noindent
The proof is in Appendix~\ref{app:proof-one-shot-error}.
%
The metric $d(\*x,\*x^\prm)$ in the assumption \eq{eq:metric} is called canonical pseudo-metric \citep[e.g.,][]{adler2007random}, and Appendix~\ref{app:bounds-canonical-metric} shows Gaussian kernel and Mat{\'e}rn kernel satisfy this condition.
%
This proposition shows that the estimator $\widehat{\mr{BR}}_{T}$ has a finite variance, and, as a consequence, the stochastic order is derived as $O_p(T / \sqrt{N})$.
%
The dependence on $N$ is same as the standard MC estimator, and on $T$ is the first order.
%
In this proposition, since we do not assume any specific acquisition function of inner BO, the result is generally applicable to a variety of BO algorithms.

For the case of the finite candidate $\Theta$, proposition~\ref{prop:one-shot-error} derives the following corollary: 
%
\begin{coro} \label{coro:optimization-error}
Let $\Theta$ be a finite set of candidate parameters. 
%
Define
\[
 \*\theta^\star \in \argmin_{\*\theta \in \Theta} \mr{BR}_T(\*\theta),
 \qquad
 \widehat{\*\theta}_N
 \in
 \argmin_{\*\theta \in \Theta} \widehat{\mr{BR}}_{T}(\*\theta)
\]
as the optimizers of the true and estimated cumulative regret, respectively.
Then
\[
 \mr{BR}_T(\widehat{\*\theta}_N) - \mr{BR}_T(\*\theta^\star)
 =
 O_p\left(\frac{T}{\sqrt{N}}\right).
\]
\end{coro}
\noindent
The proof is in Appendix~\ref{app:proof-optimization-error}.
%
Proposition~\ref{prop:one-shot-error} only shows the MC estimation error on a fixed $\*\theta$.
This corollary guarantees the difference between 
the expected cumulative regret of the true optimal configuration
$\mr{BR}_T(\*\theta^\star)$
and 
those of achieved by estimated $\widehat{\*\theta}_N$ defined by the maximizer of the MC estimator $\widehat{\mr{BR}}_T$.
%
For a continuous parameter space, the pointwise convergence in Proposition~\ref{prop:one-shot-error} does not directly extend Corollary~\ref{coro:optimization-error}, since it requires uniform convergence over $\Theta$, i.e., 
$\sup_{\*\theta\in\Theta}
\left|
\widehat{\mr{BR}}_{T,N}(\*\theta)-\mr{BR}_T(\*\theta)
\right|
=
O_p \left(\frac{T}{\sqrt{N}}\right)$, 
for which additional regularity conditions may be required for the continuous case. 
We leave a complete analysis of the continuous case for future work.



\section{Extensions}

\subsection{Multi-fidelity Approach for Outer BO}

While our optimization target is the $T$ iteration cumulative regret \eq{eq:empirical-regret}, the cumulative regret with earlier $t < T$ iterations can be seen as a cheaper approximation of \eq{eq:empirical-regret}.
%
It is possible to regard the optimization problem \eq{eq:maximize-empirical-regret} as a multi-fidelity optimization problem, in which $\widehat{\mr{BR}}_T(\*\theta)$ is the highest fidelity function and $\widehat{\mr{BR}}_t(\*\theta)$ for $t < T$ is a lower fidelity function.
%
Therefore, we can exploit multi-fidelity BO (MFBO) to accelerate the optimization \eq{eq:maximize-empirical-regret}.

A variety of methods have been proposed for MFBO \citep[e.g.,][]{Huang2005-Sequential,kandasamy2016gaussian,Takeno2020-Multifidelity}, in which a pair of an input point and its fidelity level is selected at each iteration.
%
Here, we consider the case that there exists only one low fidelity level is defined by the $T_{\rm{low}} < T$ iteration inner BO (actually, it is possible to create any $H$ fidelity levels
$T_1 < T_2 < \cdots < T_H = T$).
%
In a typical setting of MFBO, a correlated multi-output GP \citep{alvarez2012kernels} is constructed to capture the relation among different fidelity levels.
%
In our case, a two-output GP represents $\widehat{\mr{BR}}_t(\*\theta)$ for $t = T_{\rm low}$ and $T$. 
%
MFBO can select a pair of 
$\*\theta$
and its high ($t = T$) or low ($t = T_{\mr{low}}$) fidelity level based on an MFBO acquisition function. 

%
%

\subsection{Optimizing Multi-Stage Configurations}
\label{ssec:multi-stage}

We have set a fixed $\*\theta$ during inner BO.
%
On the other hand, $\*\theta$ can be changed during inner BO, by which altering the search strategy at some iteration becomes possible.
%
For brevity, we here assume that we change $\*\theta$ only once at the iteration $T_{\mr{middle}} (< T)$. 
%
Let 
$\*\theta_{(1)}$
be the parameter of earlier $T_{\mr{middle}}$ iterations and 
$\*\theta_{(2)}$
be that of later iterations from $T_{\mr{middle}}+1$ to $T$.
%
Then, by re-defining 
$\*\theta = (\*\theta_{(1)}^\top, \*\theta_{(2)}^\top)^\top$,
we can optimize both parameters simultaneously by minimizing
$\widehat{\mr{BR}}_T(\*\theta)$
though outer BO. 
%
In this case, MFBO can also be applied by regarding 
$\widehat{\mr{BR}}_{T_{\mr{middle}}}(\*\theta)$
and
$\widehat{\mr{BR}}_T(\*\theta)$
as low and high fidelity, respectively (i.e., $T_{\mr{low}} = T_{\mr{middle}}$). 
%
However, unlike $\widehat{\mr{BR}}_T(\*\theta)$,
$\widehat{\mr{BR}}_{T_{\mr{middle}}}(\*\theta)$ only depends on 
$\*\theta_{(1)}$
and 
does not depend on 
$\*\theta_{(2)}$.
%
Thus, the lower fidelity function can be written as
$\widehat{\mr{BR}}_{T_{\mr{middle}}}(\*\theta_{(1)})$, i.e., low and high fidelity function have different input dimensions. 
%
To apply MFBO, this relation should be incorporated into the surrogate model.

We use a model inspired by the autoregressive co-kriging model \citep{kennedy2000predicting}. 
%
The following definition can handle low and high fidelity functions having different input dimensions, while modeling their correlated structure:
\begin{align*}
 \widehat{\mr{BR}}_t(\*\theta_{(1)}) &
 \approx f^{(\mr{low})}(\*\theta_{(1)}),
 \\
 \widehat{\mr{BR}}_T(\*\theta) &\approx f^{(\mr{high})}(\*\theta_{(1)},\*\theta_{(2)}) 
 \! \coloneqq \!
 f^{(\mr{low})}(\*\theta_{(1)}) + g(\*\theta_{(1)},\*\theta_{(2)}),
\end{align*}
where $f^{(\mr{low})}$ and $g$ follow independent GPs and $\rho$ is a .
%
%
%
This model enables us to evaluate the posterior from the dataset containing both 
$\widehat{\mr{BR}}_{T_{\mr{middle}}}(\*\theta_{(1)})$
and
$\widehat{\mr{BR}}_T(\*\theta)$.
See Appendix~\ref{app:MFBO-MS-config} for detail.

\section{Related Work}
\label{sec:related-work}

\paragraph{Learning optimization policies.}
Our work is closely related to learning to optimize (L2O) and meta-learning, where optimization algorithms are learned from a collection or distribution of optimization problems.
L2O aims to learn optimization procedures from their performance on training problems so that they generalize to new problems drawn from a similar distribution \citep{chen2022learning}.
For black-box optimization, \citet{chen2017learning} learn a derivative-free optimizer parameterized by a recurrent neural network using training functions sampled from function distributions, including GPs, while \citet{tv2019metalearning} similarly learn a recurrent black-box optimizer with a regret-based meta-objective.
More specifically in the context of BO, MetaBO \citep{volpp2020meta} meta-learns an acquisition function using reinforcement learning on related tasks, yielding a BO policy specialized to a given task distribution, while FSAF \citep{hsieh2021reinforced} combines reinforcement learning and meta-learning to learn acquisition strategies that can rapidly adapt to new black-box functions.
Related ideas have also been explored more broadly in meta-black-box optimization, where the search strategy itself is learned across optimization problems.
For example, \citet{lange2023discovering} meta-learn evolution strategies parameterized by attention-based models, and more recent work has investigated pre-training optimization models for transfer to unseen black-box optimization tasks \citep{li2024pretrained} and learning neural optimizers for black-box optimization under limited evaluation budgets \citep{li2025b2opt}.
Our approach shares with these methods the general principle of exploiting a distribution of optimization problems to improve optimization performance on future tasks.
However, rather than learning an optimizer or acquisition policy itself, we retain the structure of an established BO algorithm and optimize only its pre-specified algorithmic parameters.
%
%
By restricting the search to a structured class of established BO algorithms, our approach is designed to leverage knowledge accumulated in the BO literature, including established acquisition rules and their principles for balancing exploration and exploitation, while still allowing adaptation to the target problem distribution.

\paragraph{Learning distributions over objective functions.}
A complementary line of work uses data from related tasks to learn surrogate models or prior distributions that can transfer information to new optimization tasks.
HyperBO \citep{wang2024pretrain} pre-trains a GP prior using data from related functions and applies the learned prior to BO on new tasks.
More generally, learned stochastic-process models such as Neural Processes \citep{garnelo2018neural} learn distributions over functions from collections of related functions and condition their predictions on observations from a new task.
While these approaches focus on constructing a transferable surrogate model or distribution over objective functions from past data, our work addresses how to configure a BO algorithm given such a distribution.
The two directions are therefore complementary.
For example, a pre-trained GP prior can directly serve as the distribution over objective functions in our framework, from which sample paths are generated to evaluate and tune BO parameters.
More generally, our framework can be combined with probabilistic surrogate models from which suitable function samples can be generated.


\begin{figure*}[t]

\subfloat[][$(d, \text{initial data}) = (5, \text{Sphere})$]{
\ig{.30}{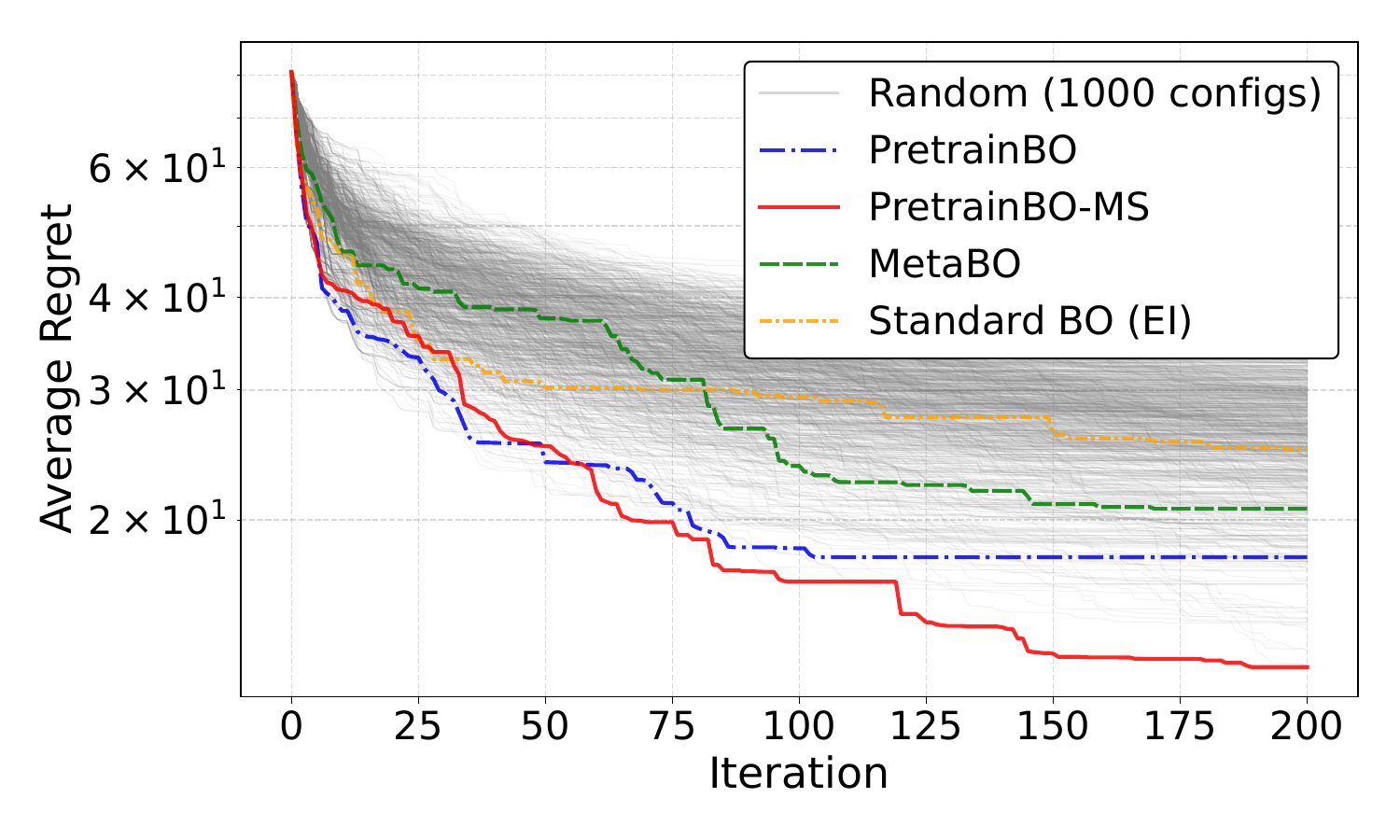}}
\subfloat[][$(d, \text{initial data}) = (10, \text{Sphere})$]{
\ig{.30}{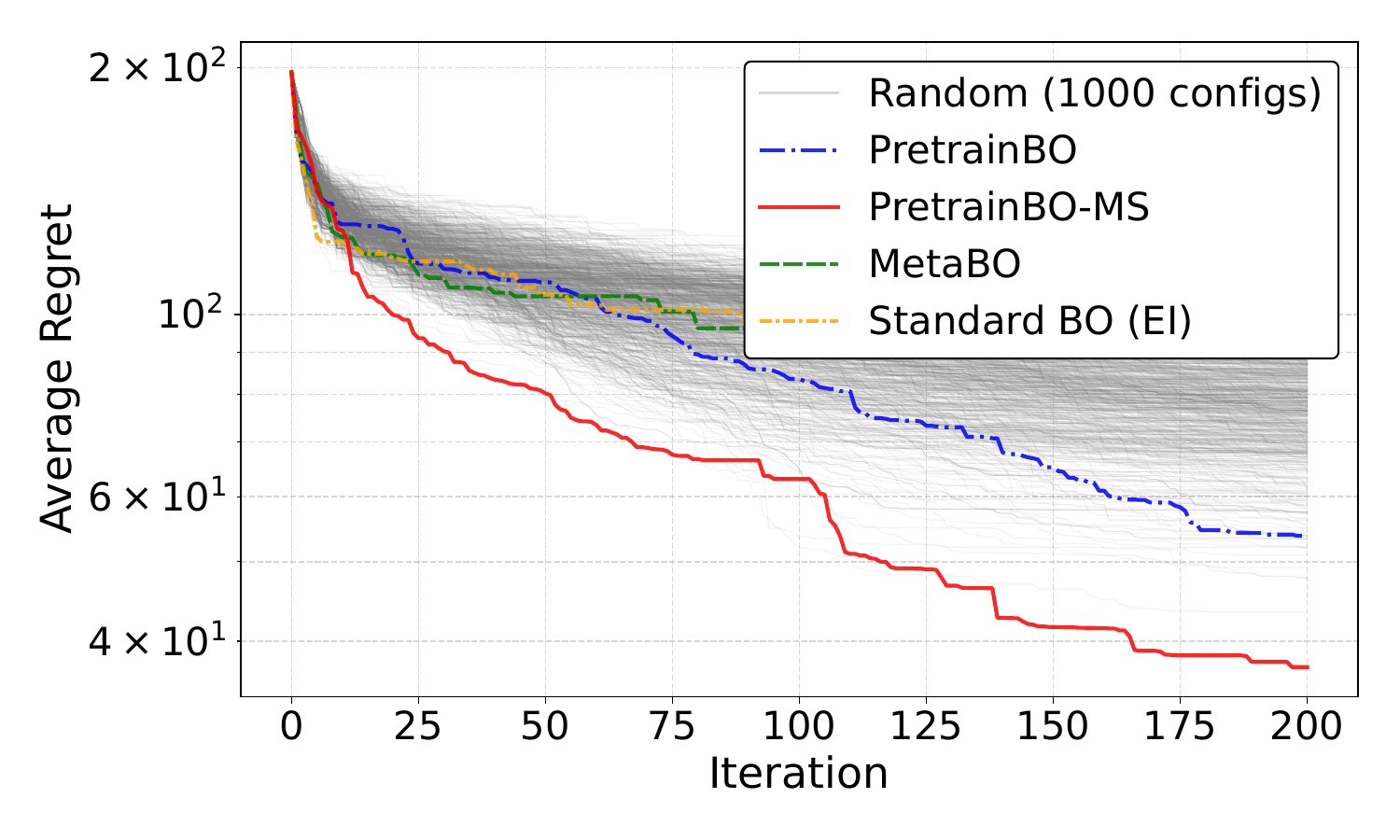}}
\subfloat[][$(d, \text{initial data}) = (20, \text{Sphere})$]{
\ig{.30}{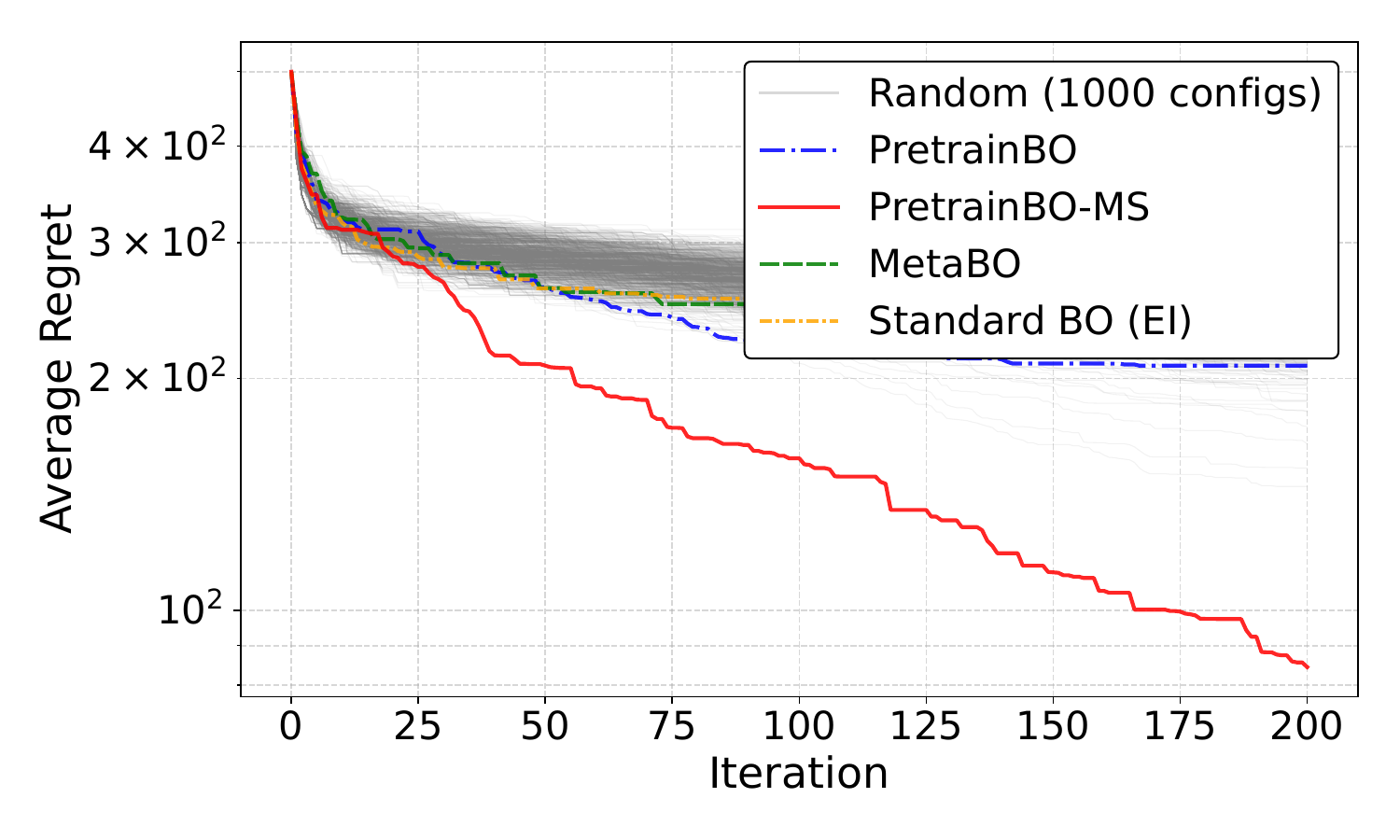}}

\subfloat[][$(d, \text{initial data})  \! \! =  \! \! (5, \text{Rosenbrock})$]{
\ig{.30}{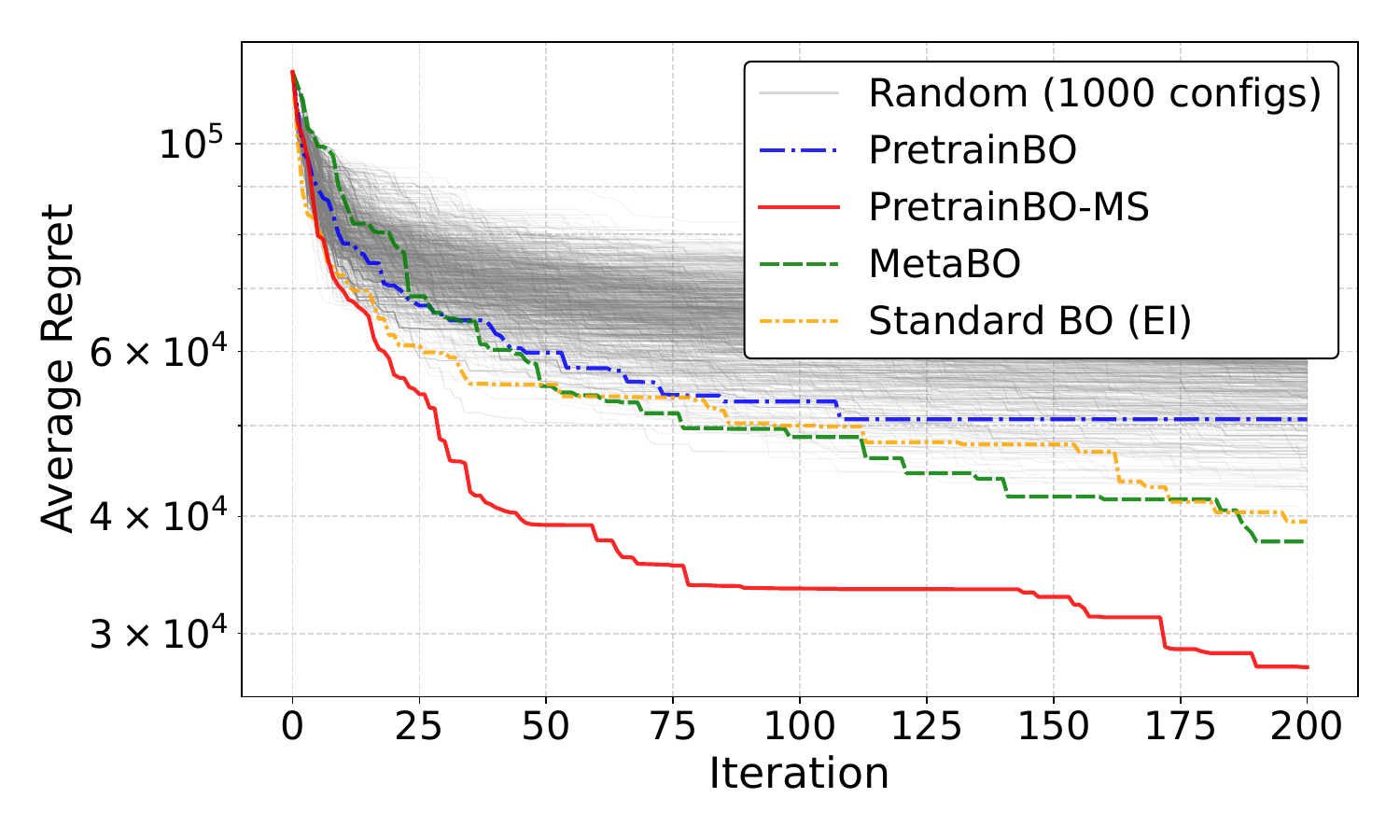}}
\subfloat[][$(d, \text{initial data}) \! \! = \! \! (10, \text{Rosenbrock})$]{
\ig{.30}{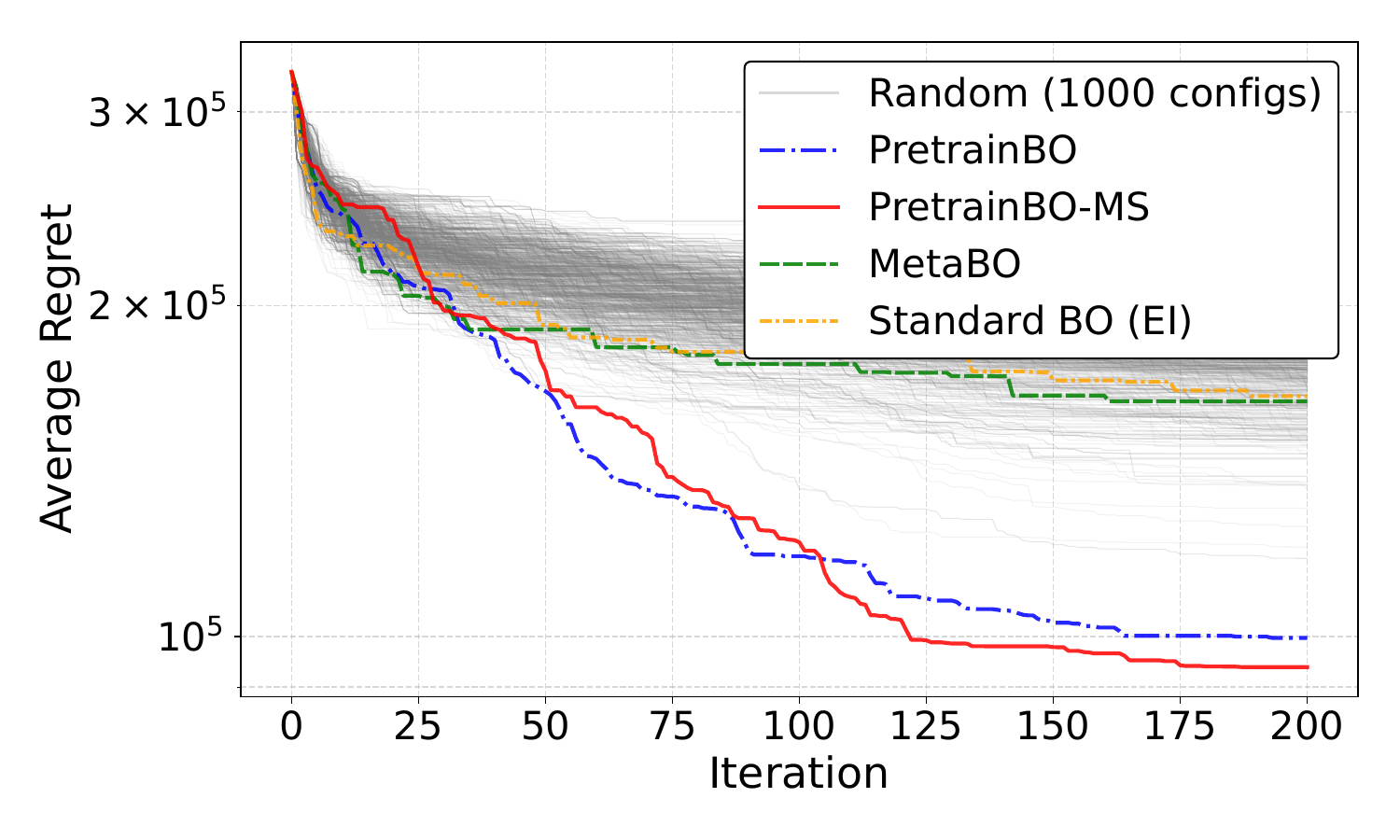}}
\subfloat[][$(d, \text{initial data})  \! \! =  \! \! (20, \text{Rosenbrock})$]{
\ig{.30}{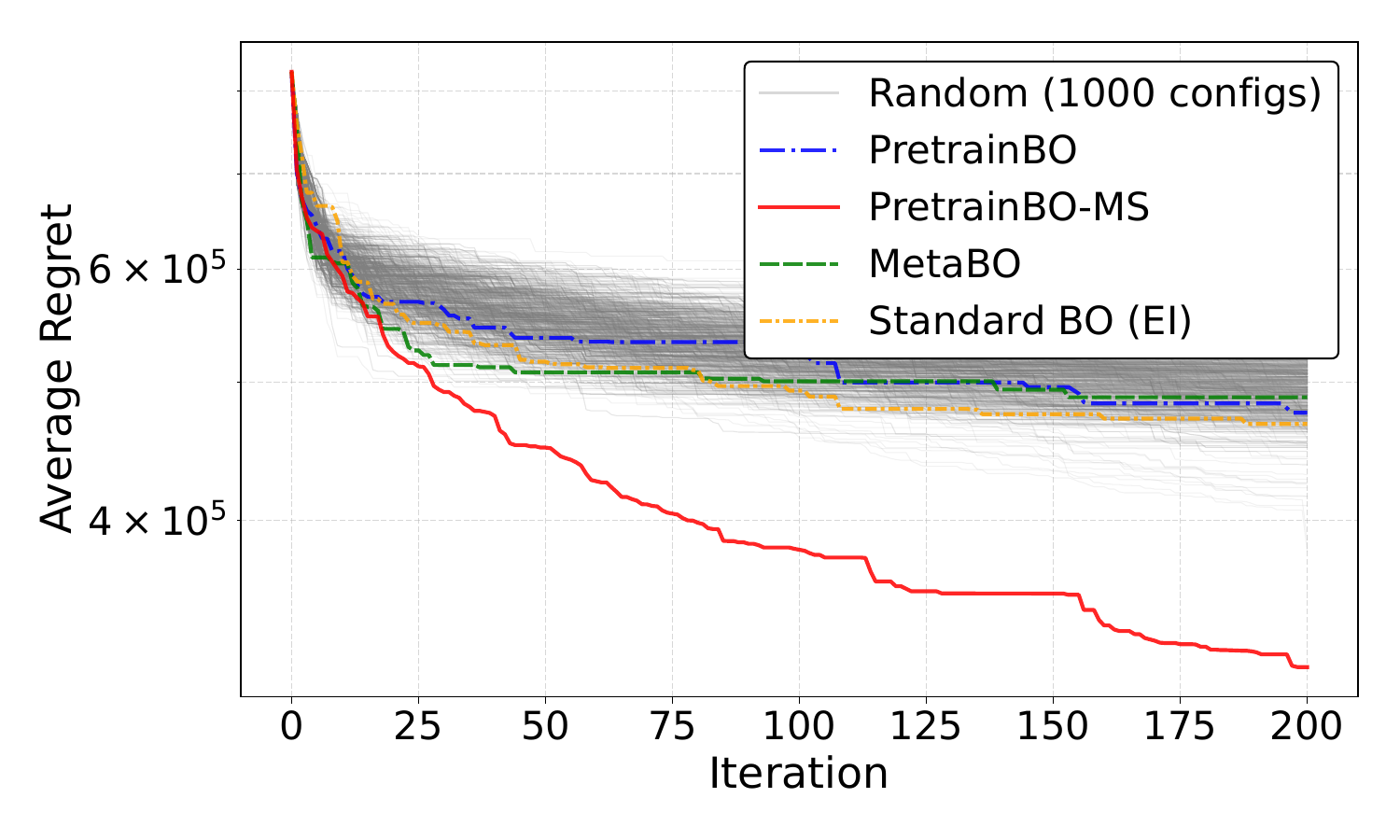}}

\subfloat[][$(d, \text{initial data}) = (5, \text{Levy})$]{
\ig{.30}{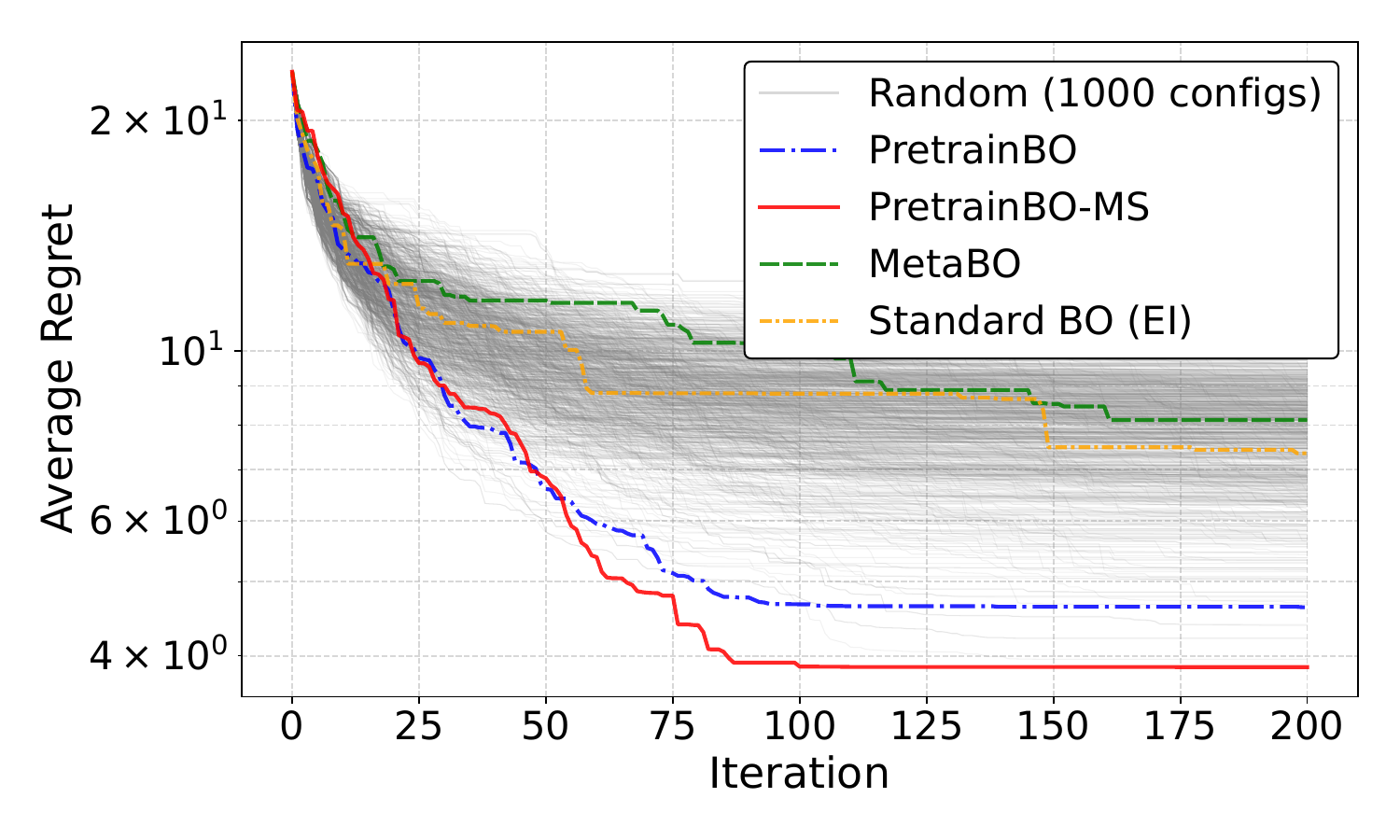}}
\subfloat[][$(d, \text{initial data}) = (10, \text{Levy})$]{
\ig{.30}{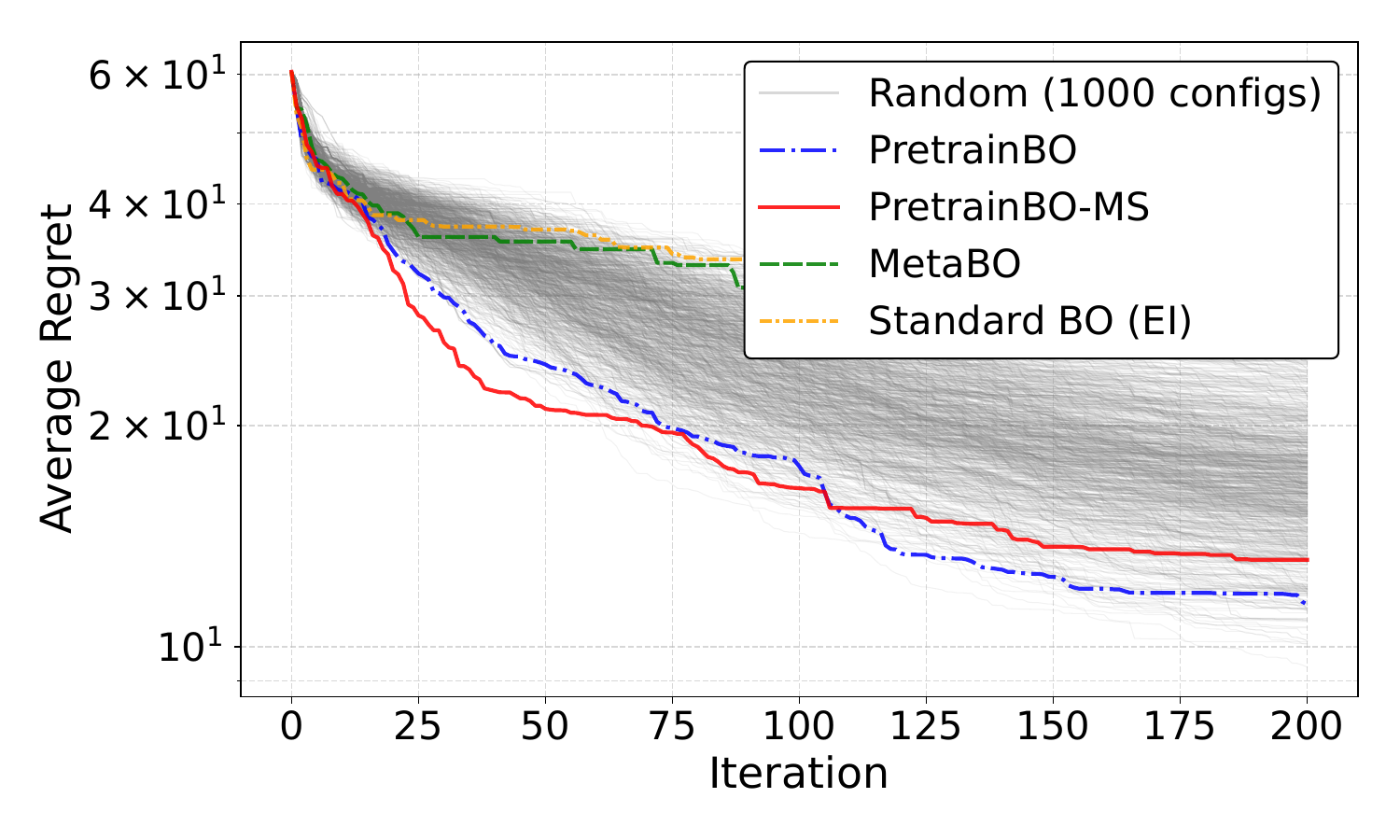}}
\subfloat[][$(d, \text{initial data}) = (20, \text{Levy})$]{
\ig{.30}{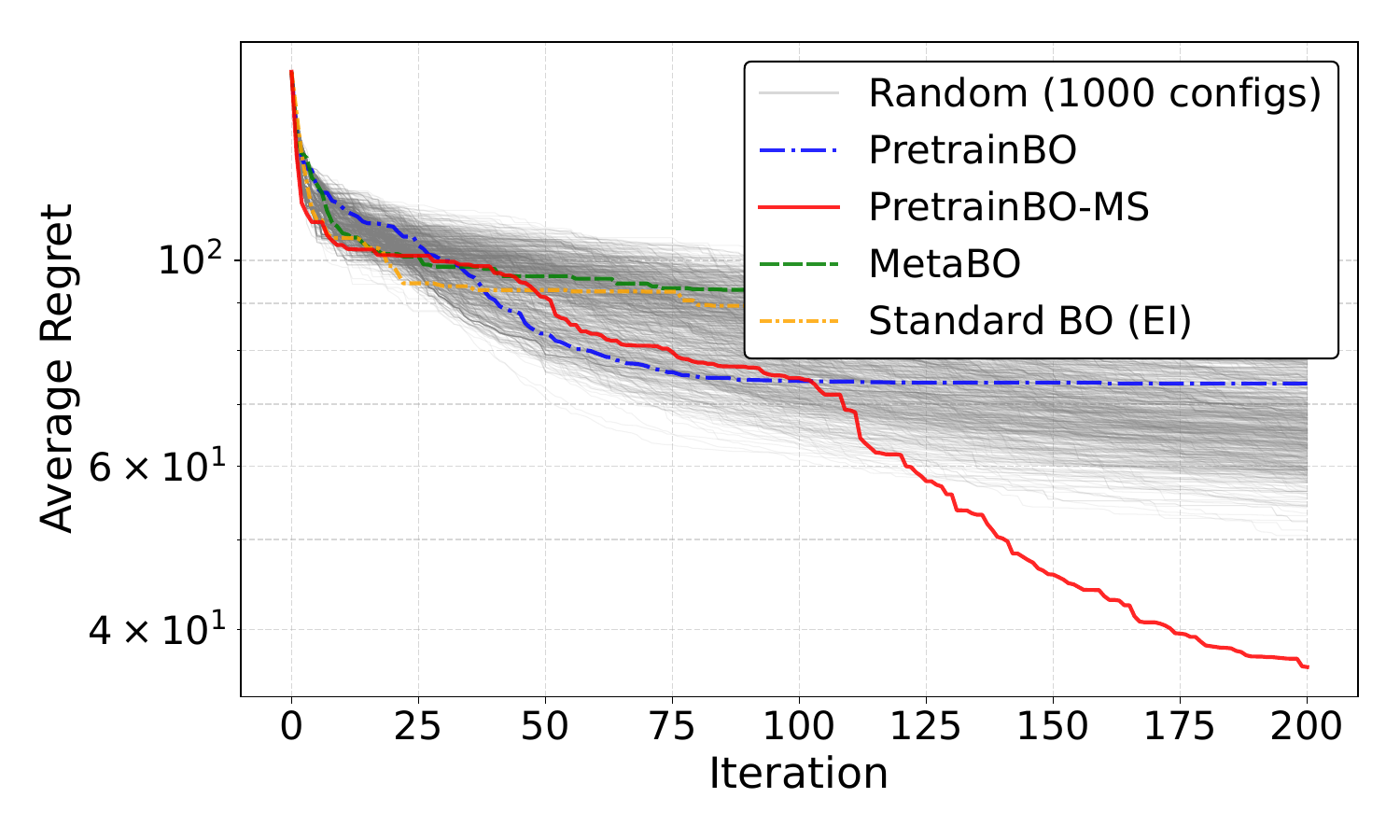}}

\caption{
Comparison of test function performance.
}
\label{fig:test-BO-trajectory}

\end{figure*}

\section{Experiments}
\label{sec:experiments}

%

We empirically evaluate the performance of the proposed framework.
We refer to the proposed method as \emph{Pretrained BO} and to its multi-stage extension described in Section~\ref{ssec:multi-stage} as \emph{Pretrained BO-MS}.
For MFBO in outer BO, we set $T_{\mr{low}} = T / 2$ and use MFSKO \citep{Huang2005-Sequential} as an acquisition function.
We use the configuration of TuRBO~\citep{eriksson2019scalable} and LSBO~\citep{gomez2018automatic} for demonstrating the performance of our framework.







\begin{figure*}[t]

\subfloat[][
Pretrained BO \\
$\qquad (d, \text{initial data}) = (10, \text{Sphere})$
]{
\ig{.30}{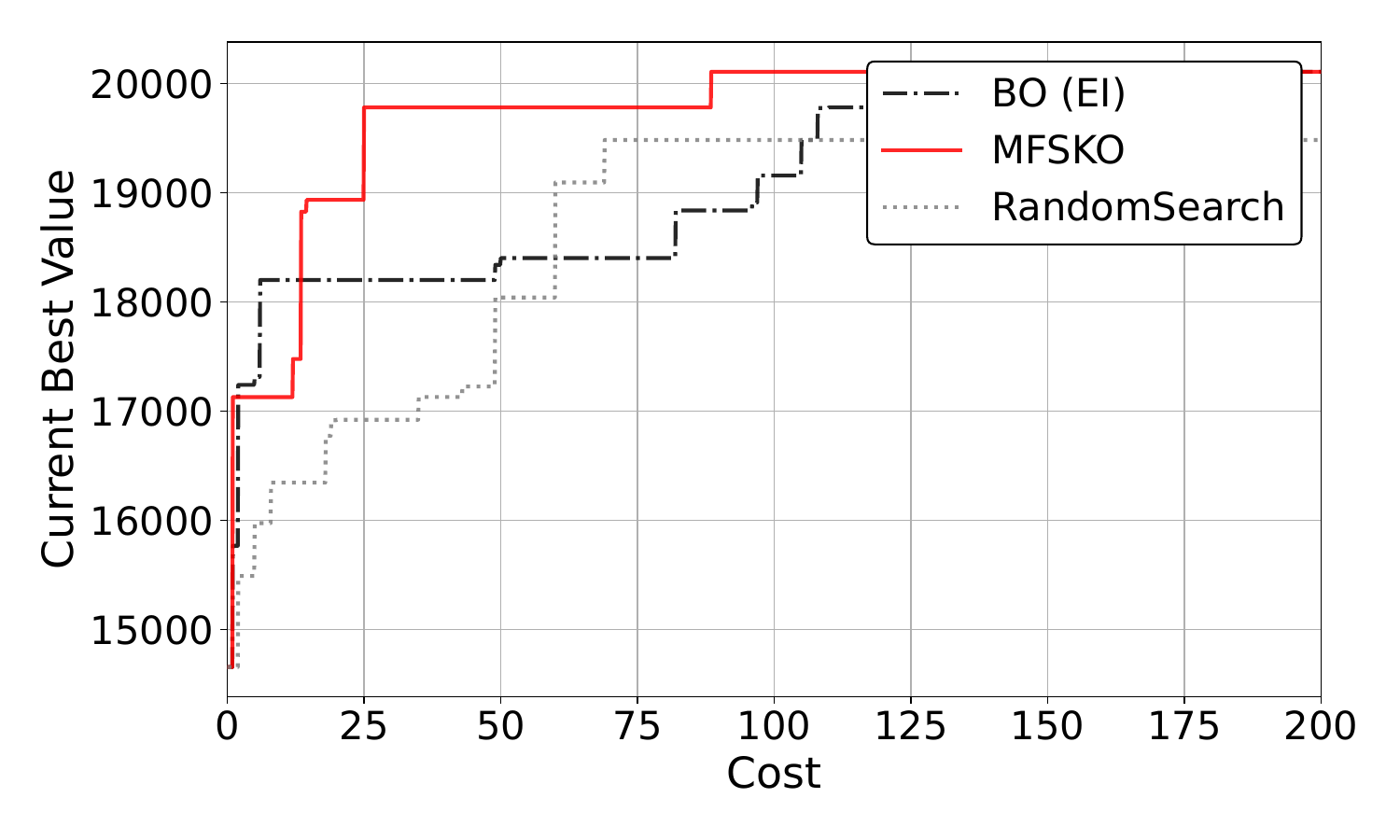}}
\subfloat[][
Pretrained BO \\
$\quad (d, \text{initial data}) = (10, \text{Rosenbrock})$
]{
\ig{.30}{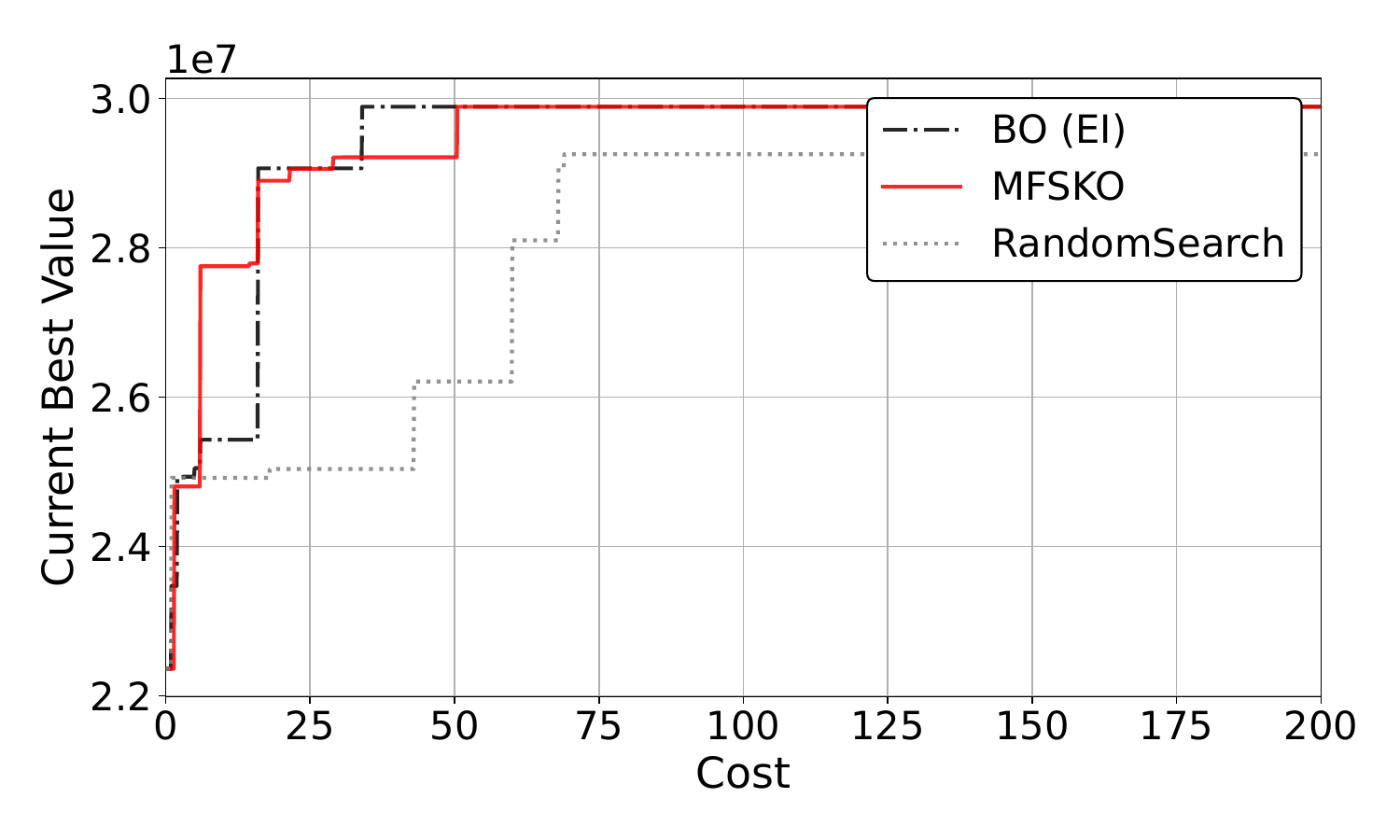}}
\subfloat[][
Pretrained BO \\
$\qquad (d, \text{initial data}) = (10, \text{Levy})$
]{
\ig{.30}{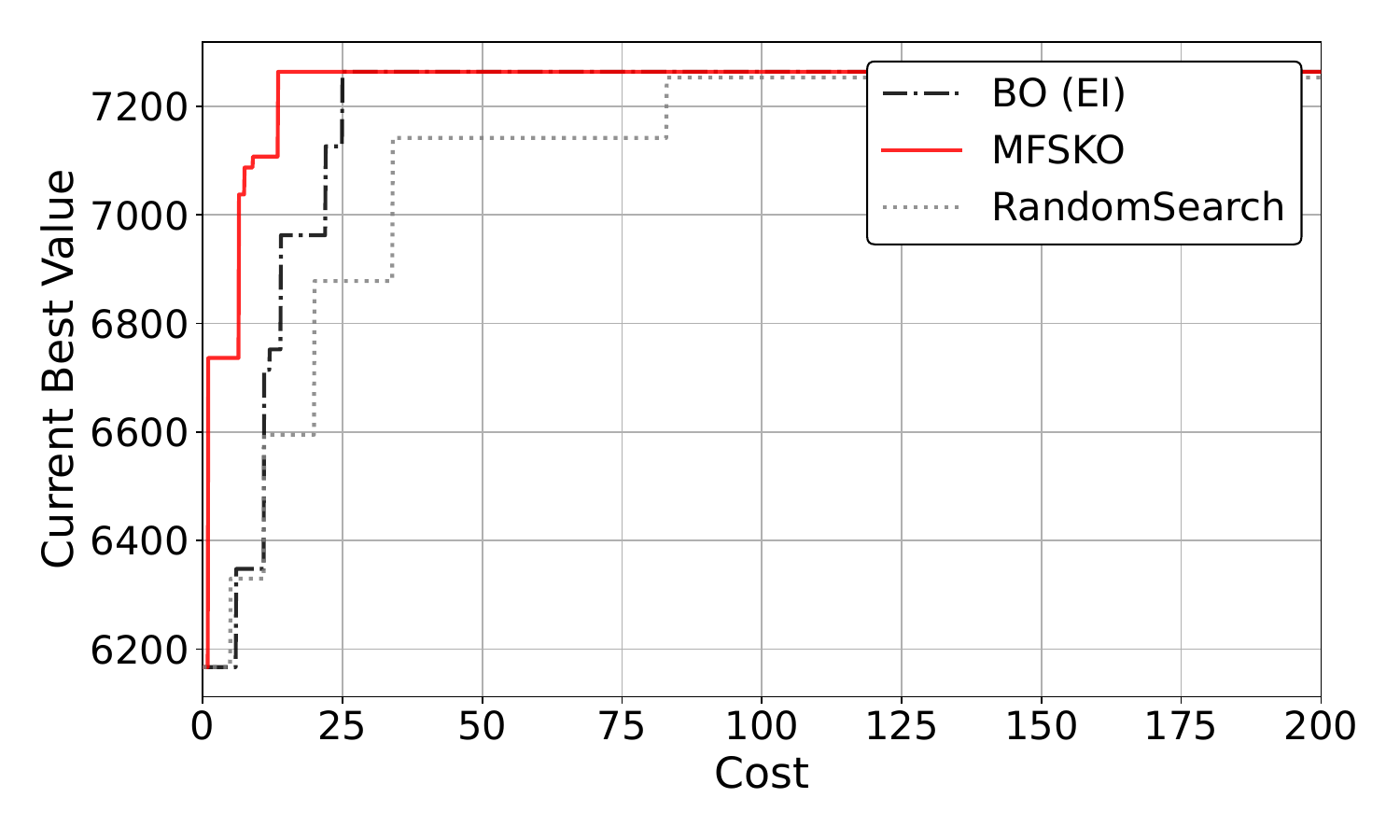}}

\subfloat[][
Pretrained BO-MS \\
$\qquad (d, \text{initial data}) = (10, \text{Sphere})$
]{
\ig{.30}{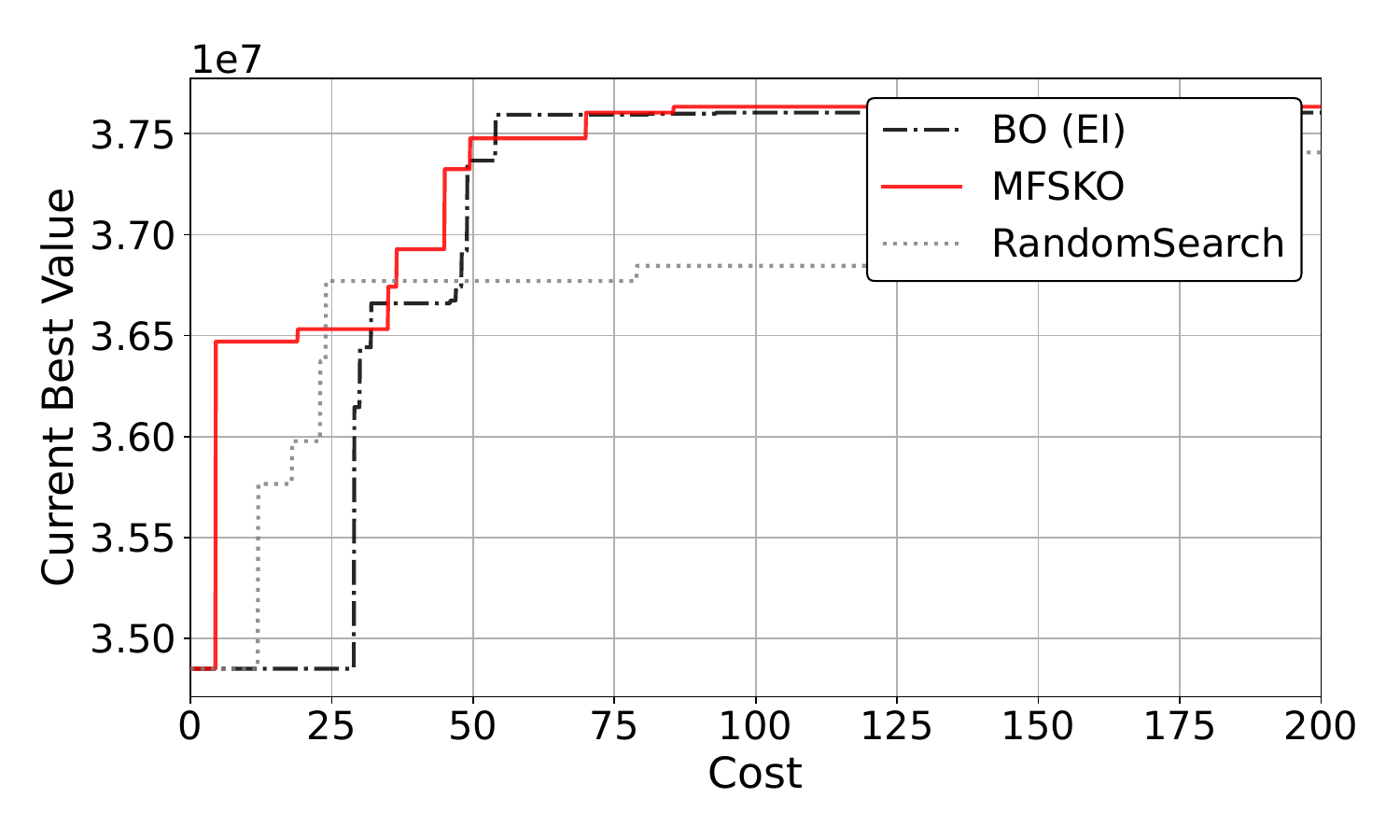}}
\subfloat[][
Pretrained BO-MS \\
$\quad (d, \text{initial data}) = (10, \text{Rosenbrock})$
]{
\ig{.30}{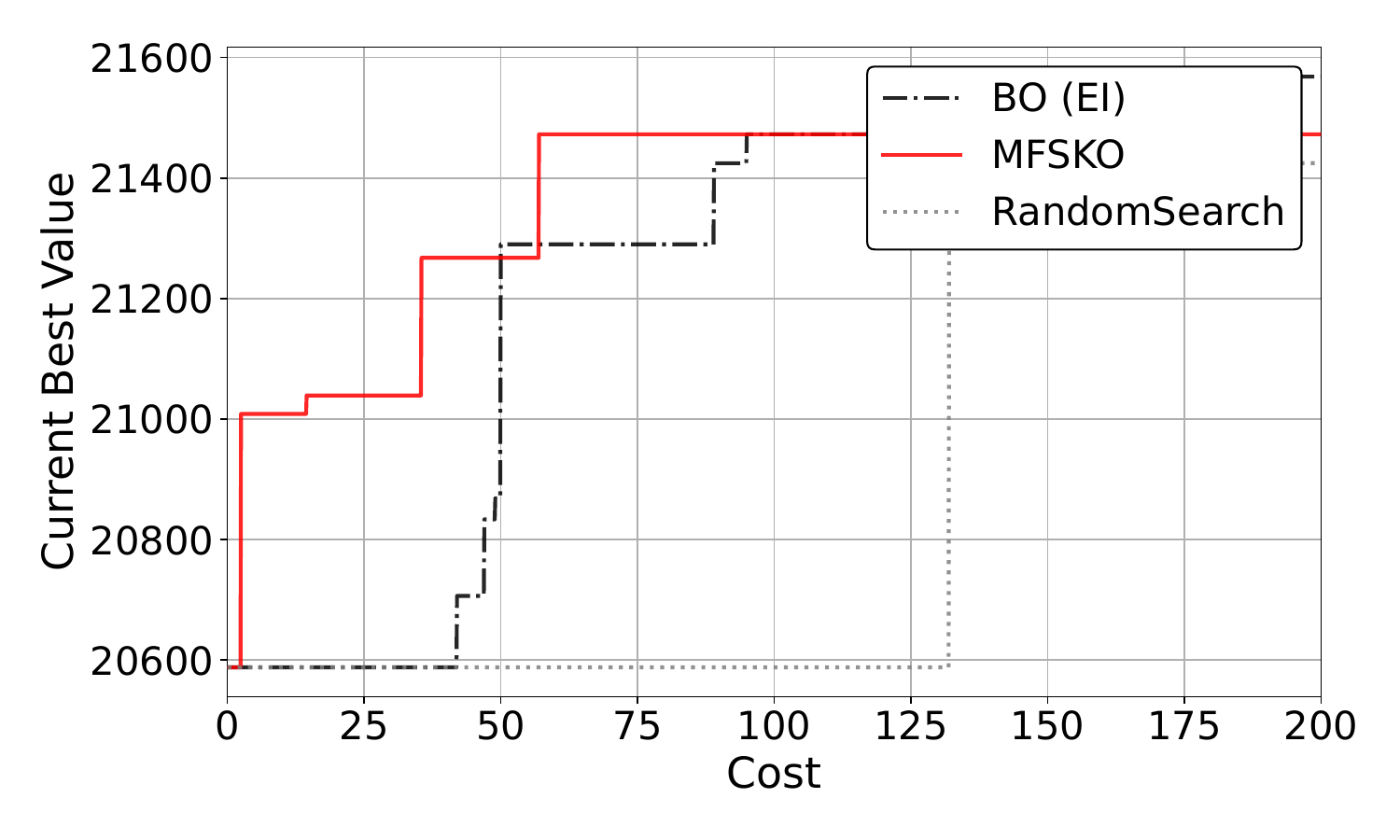}}
\subfloat[][
Pretrained BO-MS \\
$\qquad (d, \text{initial data}) = (10, \text{Levy})$
]{
\ig{.30}{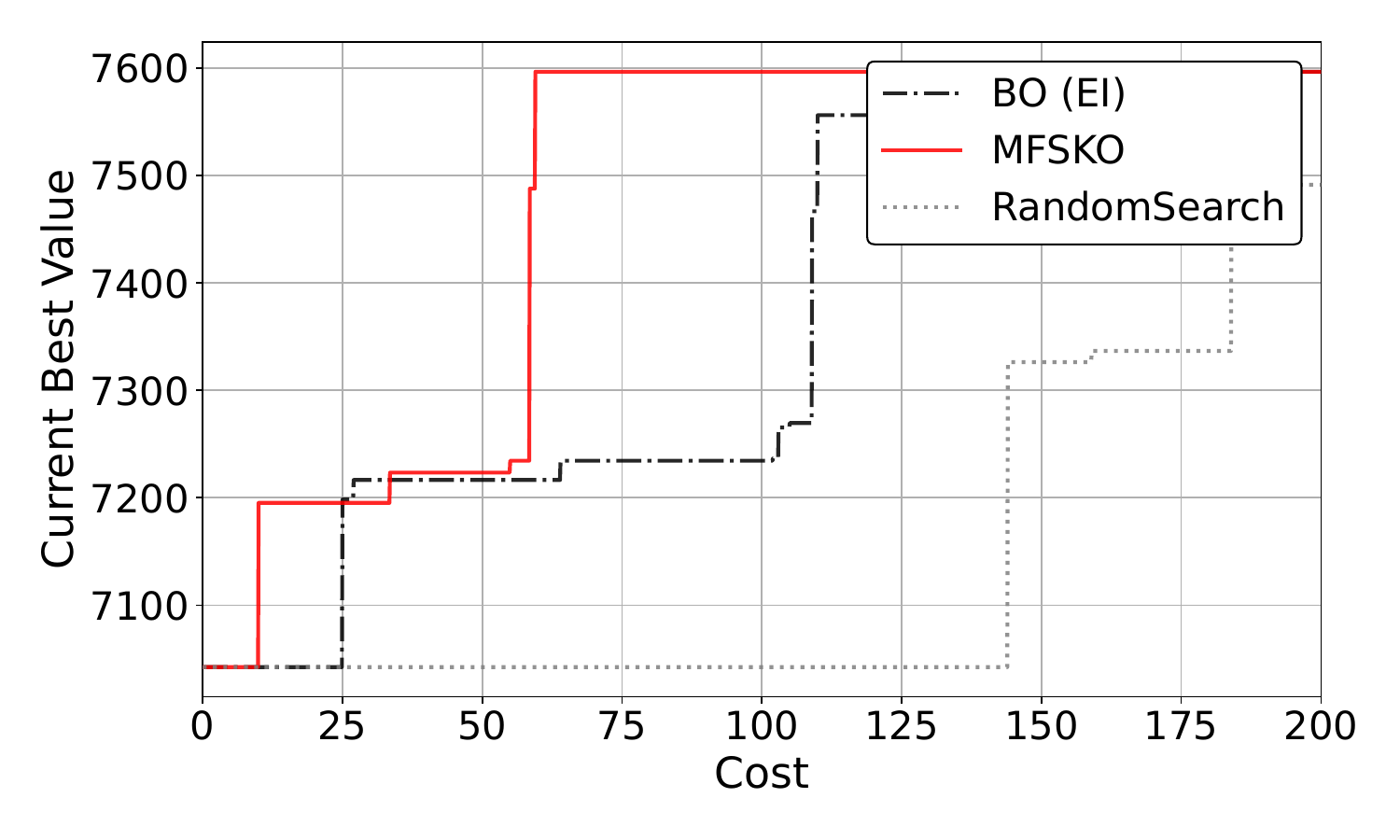}}


 \caption{%
 Transition of $\frac{1}{N} \sum_{t = 1}^T f^{(i)}(\*x_t)$.
 (a)-(c): Pretrained BO, (d)-(f): Pretrained BO-MS
 }
 \label{fig:train-best-score}

\end{figure*}

\begin{figure}[t]

 \centering

 \ig{.33}{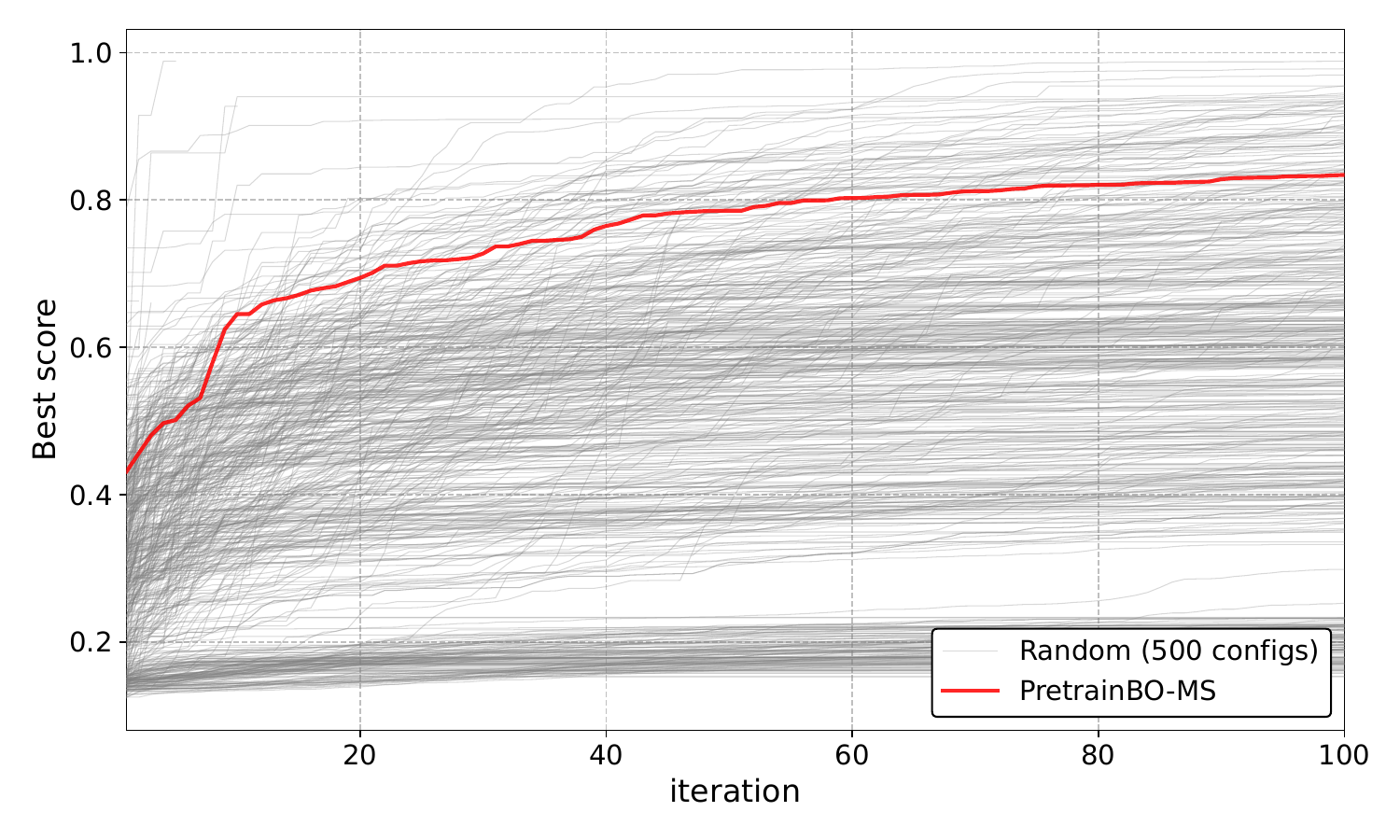}


 \caption{Score transition of LSBO.}
 \label{fig:lsbo}

\end{figure}

\subsection{TuRBO}


We first consider the optimization of GP sample-path functions using TuRBO.
For Pretrained BO, we optimize seven TuRBO parameters, including thresholds for the numbers of successes and failures and the number of trust regions.
The resulting configuration space contains 3,888 discrete candidates. 
For Pretrained BO-MS, using separate configurations for the first and second stages doubles the number of parameters.
To reduce the computational cost, we therefore restrict the configuration to three parameters for each stage.
Each stage has 105 candidate configurations, resulting in $105^2 = 11{,}025$ multi-stage candidates.
See Appendix~\ref{app:setting-turbo} for further details.

We use a zero prior mean and a Gaussian kernel for the GP, with observation-noise variance $\sigma_\epsilon^2 = 10^{-6}$.
The number of inner-BO iterations is set to $T=200$. 
We first randomly select $n_0=100$ input points from $\cX$ and construct the initial dataset by evaluating a benchmark function at these points.
We consider Sphere, Rosenbrock, and Levy as benchmark functions \citep{simulationlib}.
From the GP posterior constructed using this initial dataset, we generate both the training sample paths
$f^{(1)},\ldots,f^{(N)}$ 
used by Pretrained BO and test functions $f^{\mathrm{test}}$ used to evaluate the performance of TuRBO with the estimated configuration $\hat{\theta}_N$.
We independently generate ten test functions and report the average performance over them.
For outer BO, we also use a Gaussian kernel. 
The number of outer-BO iterations is set to
$T_{\mr{outer}}=200$.


Figure~\ref{fig:test-BO-trajectory} shows the results.
As baselines, we include MetaBO (from \url{https://github.com/boschresearch/MetaBO}, which is AGPL-3.0 license), standard BO with expected improvement (EI), and 1,000 randomly selected configurations.
MetaBO is trained using the same training functions as the proposed method and learns an acquisition function from them.
Standard BO (EI) uses the conventional EI acquisition function without pre-training.
For the random-configuration baseline, 1,000 configurations are sampled from the candidate set $\Theta$ (the multi-stage setting is not applied to this baseline).


We first compare Pretrained BO with the 1,000 randomly selected configurations.
In many settings, the black dashed curve corresponding to Pretrained BO remains below the collection of 1,000 gray curves.
Since these configurations are sampled uniformly from the candidate set $\Theta$, when the Pretrained BO curve is below all 1,000 random configurations, its selected configuration can roughly be regarded as belonging to the top $0.1\%$ of the candidate set in terms of performance.
Pretrained BO also outperforms standard BO and MetaBO in many of the settings.
Even in settings such as Figs.~\ref{fig:test-BO-trajectory} (d) and (f), where such an improvement over these existing methods is not observed, comparison with the gray curves indicates that Pretrained BO still identifies a high-performing configuration within $\Theta$.
Pretrained BO-MS yields substantial additional improvements in several settings, while in others, such as Figs.~\ref{fig:test-BO-trajectory} (e) and (h), its performance is comparable to that of Pretrained BO without the multi-stage extension.
One possible reason is that, for computational tractability, we reduced the dimensionality of the configuration space for Pretrained BO-MS.
The results with $n_0 = 50$ and $n_0 = 300$ are in Appendix~\ref{app:results-n0}, and those with std error plots are in Appendix~\ref{app:error-bars}.


Figure~\ref{fig:train-best-score} illustrates the optimization process of outer BO.
The vertical axis shows 
$\frac{1}{N} \sum_{t = 1}^T f^{(i)}(\*x_t)$
and the horizontal axis represents the total evaluation cost.
%
%
For each setting, we compare standard BO, MFBO (MFSKO), and random search for optimizing $\*\theta$ (Note that in Fig.~\ref{fig:test-BO-trajectory}, standard BO was used for Pretrained BO, whereas MFBO was used for Pretrained BO-MS. See Appendix~\ref{app:additional-detail-mfbo} for detail of MFBO setting here). 
%
Figure~\ref{fig:train-best-score} shows a general tendency for BO to search more efficiently than random search and for MFBO to further improve the search efficiency over standard BO.

\subsection{LSBO}



We next consider a synthetic latent-space optimization problem based on an image-classification CNN.
This setting is inspired by \citet{boyar2024latent}.
The CNN is pre-trained on Fashion-MNIST \citep{xiao2017fashion}, which is released under the MIT License.
We assign a score to each Fashion-MNIST class and aim to generate images with high scores.
Let $s_y$ denote the score assigned to class $y$, and let $p_1,\ldots,p_K$ denote the class probabilities predicted by the CNN, where $K$ is the number of classes.
The expected score of an image is then defined as
$\sum_{y \in [K]} p_y s_y$.

We first train a VAE to construct a 20-dimensional latent space and then perform BO in this latent space to maximize the expected score.
Specifically, BO selects a point in the latent space, the point is decoded into an image by the VAE, and the objective value is obtained by evaluating the expected score of the decoded image using the CNN.
This setup can be viewed as a simplified analogue of, for example, searching a molecular latent space for compounds with desirable properties.

We treat the VAE hyperparameters as the configuration $\*\theta$. 
Specifically, we optimize three weighting parameters controlling properties of the latent representation \citep{chen2018isolating}, such as independence, together with the dimensionality of the latent embedding, resulting in four configuration parameters.
We set the inner-BO horizon to $T=100$. 
The training functions are generated by perturbing the class scores $s_y$.
Further details are provided in Appendix~\ref{app:setting-lsbo}.

Figure~\ref{fig:lsbo} shows the results, where the vertical axis represents the expected score.
For clarity, we show only Pretrained BO-MS and the results for 500 randomly selected configurations.
The proposed method identifies a configuration with high optimization performance among the candidate configurations.

\section{Conclusion}

We proposed Pretrained BO for configuring a BO algorithm
by minimizing its expected cumulative regret estimated from GP sample
paths.
We established theoretical guarantees for the MC approximation and
extended the framework to multi-fidelity and multi-stage settings.
Experiments on TuRBO and LSBO demonstrated that the proposed method
can identify high-performing configurations.
Our results suggest that prior knowledge about objective functions can
be exploited not only for surrogate modeling but also for configuring
the BO algorithm itself.
A limitation is that 
misspecification of the estimated GP posterior may lead to suboptimal BO parameter choices. Evaluating the robustness under such model misspecification is an important direction for future work.

\clearpage

\ifdefined\ackn
\subsection*{Acknowledgement}
This work was partially supported by MEXT KAKENHI (25K03182), MEXT Supporting Pioneering Re-search through AI for 1,000 Discovery challenges Program (SPReAD) Japan Grant Number JPMXP1726275243, and Data Creation and Utilization Type Material Research and Development Project (Grant No. JP-MXP1122712807) of MEXT.
\fi

\biblstyle
\bibliography{ref}



\clearpage

\appendix

\ifdefined\arxiv
\begin{center}
 {\Large{\bf 
 Supplementary Appendix of \\
 \input{title}
 \par}}
\end{center}
\fi

\ifdefined\aistats
\thispagestyle{empty}

\onecolumn
\aistatstitle{
 Supplementary Appendix of \\
\input{title}
}
\fi

\section{Details of MFBO for Multi-Stage Configurations}
\label{app:MFBO-MS-config}

Let
\begin{align*}
 f^{\mr{low}} &\sim \cG\cP(0, k_l(\*\theta_{(1)},\*\theta_{(1)}^\prm)), \\
 g &\sim \cG\cP(0, k_g((\*\theta_{(1)}, \*\theta_{(2)}), \ (\*\theta^\prm_{(1)}, \*\theta^\prm_{(2)})) ),
\end{align*}
where $k_l$ and $k_g$ are kernel functions.
From the definition of $f^{\mr{(low)}}$ and $f^{\mr{(high)}}$, 
\begin{align*}
 \mr{Cov}(f^{\mr{(low)}}(\*\theta_{(1)}), \ f^{\mr{(low)}}(\*\theta^\prm_{(1)})) &= k_l(\*\theta_{(1)}, \*\theta^\prm_{(1)}),
 \\
 \mr{Cov}(f^{\mr{(low)}}(\*\theta_{(1)}), \ f^{\mr{(high)}}(\*\theta^\prm_{(1)}, \*\theta^\prm_{(2)})) &= k_l(\*\theta_{(1)}, \*\theta^\prm_{(1)}),
 \\
 \mr{Cov}(f^{\mr{(high)}}(\*\theta_{(1)}, \*\theta_{(2)}), \ f^{\mr{(high)}}(\*\theta^\prm_{(1)}, \*\theta^\prm_{(2)})) &=
 k_l(\*\theta_{(1)}, \*\theta^\prm_{(1)})
 +  k_g( (\*\theta_{(1)}, \*\theta_{(2)}), \ (\*\theta^\prm_{(1)}, \*\theta^\prm_{(2)}) ). 
\end{align*}
Let
$\tilde{\*\theta}$
be 
$\*\theta_{(1)}$ for the low fidelity input
and 
$\*\theta$ for the high fidelity input.
Then, the kernel for multi-stage observation is 
\begin{align*}
 k_{\mr{ms}}(\tilde{\*\theta}, \tilde{\*\theta}^\prm)
 =
 \begin{cases}
  k_l(\*\theta_{(1)}, \*\theta^\prm_{(1)})
  +
  k_g( (\*\theta_{(1)}, \*\theta_{(2)}), \ (\*\theta^\prm_{(1)}, \*\theta^\prm_{(2)}) ), 
  & \text{ if both of } \tilde{\*\theta} \text{ and } \tilde{\*\theta}^\prm \text{ are from high fidelity}\\
  k_l(\*\theta_{(1)}, \*\theta^\prm_{(1)}), & \text{ otherwise }
 \end{cases}
\end{align*}


The acquisition function of MFSKO \citep{Huang2005-Sequential}
$\alpha(\*\theta, \mr{fidelity})$
requires full 
$\*\theta$
including 
$\*\theta_{(2)}$
even when the acquisition function value for $\mr{fidelity} = \mr{low}$ is calculated.
To avoid this issue, we first optimize  
$(\bar{\*\theta}_{\mr{(1)}}, \bar{\*\theta}_{\mr{(2)}}) = \argmax_{\*\theta} \alpha(\*\theta, \mr{fidelity} = \mr{high})$.
Next, by fixing $\bar{\*\theta}_{\mr{(2)}}$, we optimize 
$\max_{\*\theta_{(1)}} \alpha([\*\theta_{\mr{(1)}}, \bar{\*\theta}_{\mr{(2)}}], \mr{fidelity} = \mr{low})$. 
The same procedure is applicable to MFMES \citep{Takeno2020-Multifidelity}.

\section{Experimental Detail}
\label{app:experimental-detail}

\subsection{Setting of TuRBO Experiment}
\label{app:setting-turbo}

In Pretrained BO, $\Theta$ is defined as follows
\begin{itemize}
 \item Success/failure tolerance $\in [1,2,3,4]$: The number of consecutive successes/failures required to expand/shrink the trust region.
 \item 
    Initial TR size $\in [0.8, 1.0, 1.2, 1.4]$: The initial side length of the trust region.
 \item 
    Maximum TR size $\in [2, 4, 8]$: The maximum allowed side length of the trust region.
 \item 
    Minimum TR size $\in [0.01, 0.1, 0.5]$: The minimum allowed side length of the trust region.
 \item 
    Number of TRs $\in [1, 3, 5]$: The number of trust regions maintained in parallel.
 \item 
    Expansion factor $\in [1.5, 2, 3]$: The factor by which the trust region is expanded after sufficient successes.
 \item 
    Shrinkage factor $\in [0.3, 0.5, 0.7]$: The factor by which the trust region is shrunk after sufficient failures.
\end{itemize}
For Pretrained BO-MS, candidates are 
success/failure tolerance $\in [1,2,3,4,5]$, 
initial TR size $\in [0.8, 1.0, 1.2, 1.4]$, and 
minimum TR size $\in [0.01, 0.1, 0.5]$.


Let $v^2$ be the variance of initial $n_0$ points.
To generate initial sample path functions $f^{(1)}, \ldots, f^{(N)}$, parameters in Gaussian kernel are that the length scale was $0.5$, the variance parameter was $0.3 \times v^2$, and the noise variance is $10^{-6} \times v^2$. 


In outer BO, parameters of (Gaussian) kernel (length scale, variance parameter, noise variance, and correlation factor $\rho$ in MFGP) are optimized by marginal log-likelihood at every iteration.



In inner BO, parameters of (Gaussian) kernel (length scale, variance parameter, and noise variance) are optimized by marginal log-likelihood at every three iterations.


To generate training functions $\hat{f}^{(1)}, \ldots, \hat{f}^{(N)}$ in the continuous space, random Fourier feature \citep{Rahimi2008-Random} was used, in which the number of basis functions was set $1,000$.

\subsection{Setting of LSBO Experiment}
\label{app:setting-lsbo}

We employ a VAE in which each component loss has its own weight parameter \citep{chen2018isolating} by using the implementation \url{https://github.com/rtqichen/beta-tcvae}.
Let $p_{\mathrm{data}}(\bm{x})$ denote the empirical data distribution,
$q(\bm{z}\mid\bm{x})$ the encoder distribution, and $\bm{z}$ the latent
representation. The aggregated posterior is defined as
\[
q(\bm{z})
=
\mathbb{E}_{p_{\mathrm{data}}(\bm{x})}
[q(\bm{z}\mid\bm{x})],
\]
with $q(z_j)$ denoting its $j$-th marginal distribution.
Following the ELBO decomposition of \citet{chen2018isolating}, we consider
the weighted objective
\begin{align}
\mathcal{L}_{\alpha,\beta,\gamma}
={}&
\mathbb{E}_{p_{\mathrm{data}}(\bm{x})q(\bm{z}\mid\bm{x})}
[\log p(\bm{x}\mid\bm{z})]
-\alpha I_q(\bm{x};\bm{z})
\nonumber\\
&-\beta
D_{\mathrm{KL}}
\left(
q(\bm{z})
\,\middle\|\,
\prod_j q(z_j)
\right)
-\gamma
\sum_j
D_{\mathrm{KL}}
\left(
q(z_j)
\,\middle\|\,
p(z_j)
\right),
\label{eq:weighted_vae}
\end{align}
where the three coefficients control distinct properties of the latent
representation:
\begin{itemize}
    \item $\alpha$ controls the mutual information penalty between the data
    and the latent representation, regulating how much information about
    the data is retained in the latent representation.
    \item $\beta$ controls the total correlation penalty, encouraging
    statistical independence among the latent dimensions and thereby
    promoting disentanglement.
    \item $\gamma$ controls the dimension-wise KL penalty, encouraging the
    marginal distribution of each latent dimension to match its corresponding
    prior.
\end{itemize}

$\Theta$ consists of $\alpha, \beta, \gamma$ and the latent dimension. 
The finite candidate set is created in such a way that
$\log_{10} \alpha, \log_{10} \beta$ and $\log_{10} \gamma$
take from $-2.0$ to $2.0$ with the uniform grid of $0.5$.
The candidate latent dimensions are $4, 8, 16, \ldots, 32$. 
Log EI is use in inner BO (LSBO).

We define the score $s_y = \exp(-2 \times (y - y_{\mr{tgt}})^2)$. 
This means that class $y_{\mr{tgt}}$ (set as $4$ in the experiments) has the highest score, and the scores of the other classes are determined by the distance measured by the class label.
%
Since the class labels in Fashion-MNIST are not ordinal, this distance has no interpretation in the image space; we define it solely to create a setting in which each class is associated with a different level of preference.
The perturbation of $s_y$ to generate training functions are defined by $y_{\mr{tgt}} + y_{\epsilon}$, where $y_{\epsilon} \sim \cU_{[-0.5,0.5]}$. 



\subsection{Additional Details of MFBO}
\label{app:additional-detail-mfbo}

In Fig.~\ref{fig:train-best-score}, we evaluate BO and MFBO in Pretrained BO and Pretrained BO-MS, respectively.
In Pretrained BO, standard multi-output GPs are applicable, because it shares the same input dimensions among different fidelities unlike Pretrained BO-MS.
On the other hand, in Pretrained BO-MS, we need to use our model derived in Section~\ref{ssec:multi-stage}, in which different input dimensions can be handled.

%
%
In Pretrained BO, this time, we employed linear model of coregionalization (LMC) \citep{alvarez2012kernels}. 
Let $\*K(\*\theta,\*\theta^\prm) \in \RR^{2 \times 2}$ be the covariance matrix of 
$\*\theta$ with low fidelity and $\*\theta^\prm$ with high fidelity.
%
In LMC, it is set as
\[
  \*K(\*\theta,\*\theta^\prm) = \sum_{q = 1}^Q \*B_q k_q(\*\theta,\*\theta^\prm),
\]
where 
$\*B_q \in \RR^{2 \times 2}$ 
is the coregionalization matrix.
This means that the latent kernel $k_q$ is shared across different fidelities.
Specifically, we use the form
$\*B_q = \*w_q \*w_q^\top + \mr{diag}(\*k_q)$, 
where $\*w_q \in \RR^2$ and $\*k_q \in \RR^2$ are parameters, which we estimate though marginal log-likelihood (this is rank-1 setting of IndexKernel of GPyTorch).
%

\section{Other Results}
\label{app:other-result}

\subsection{Results on Different $n_0$}
\label{app:results-n0}

Figures~\ref{fig:test-BO-trajectory-50} and \ref{fig:test-BO-trajectory-300} are the same setting as Fig.~\ref{fig:test-BO-trajectory}, except for that their initial points are $50$ and $300$.

\begin{figure*}

\centering
 
\subfloat[][$(d, \text{initial data}) = (5, \text{Sphere})$]{
\ig{.3}{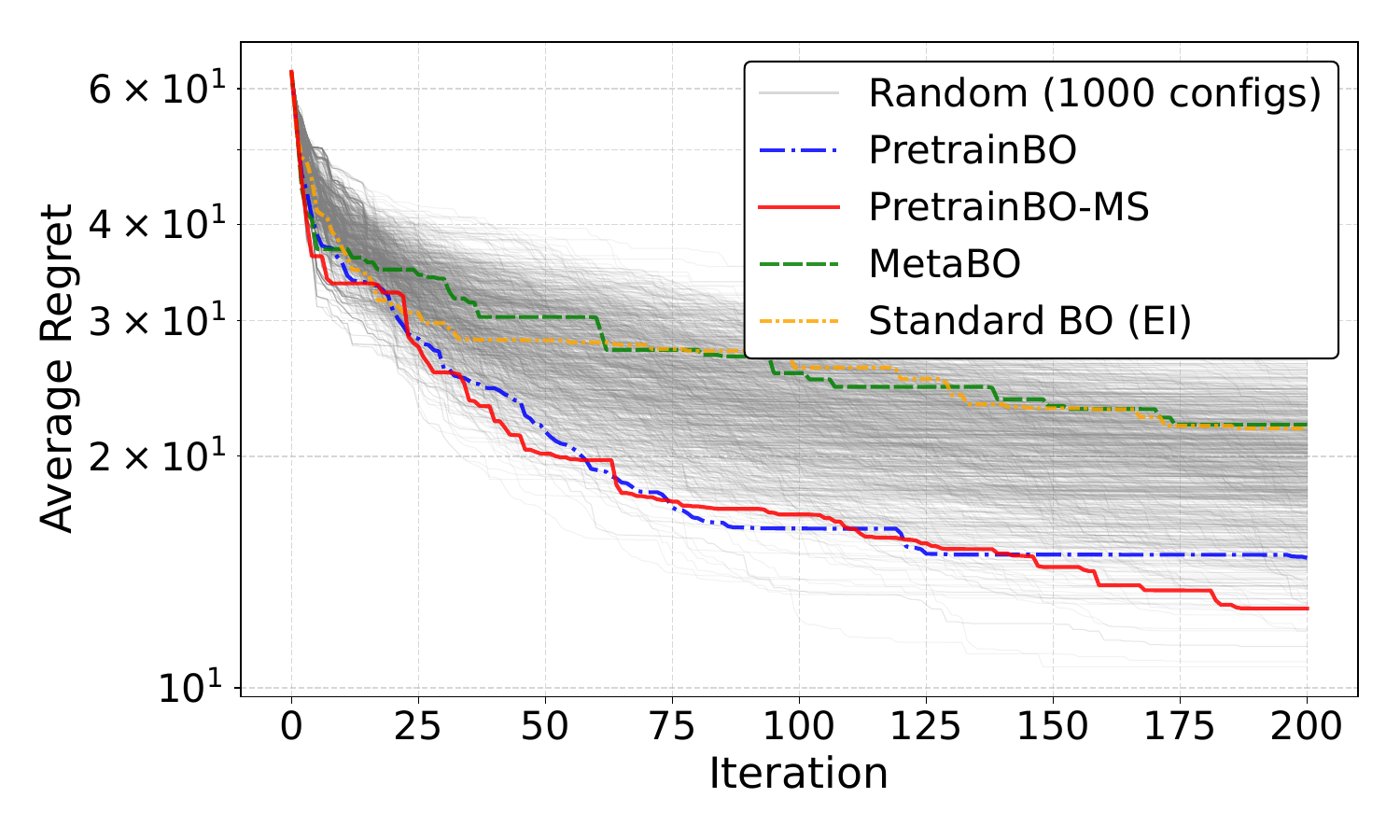}}
\subfloat[][$(d, \text{initial data}) = (10, \text{Sphere})$]{
\ig{.3}{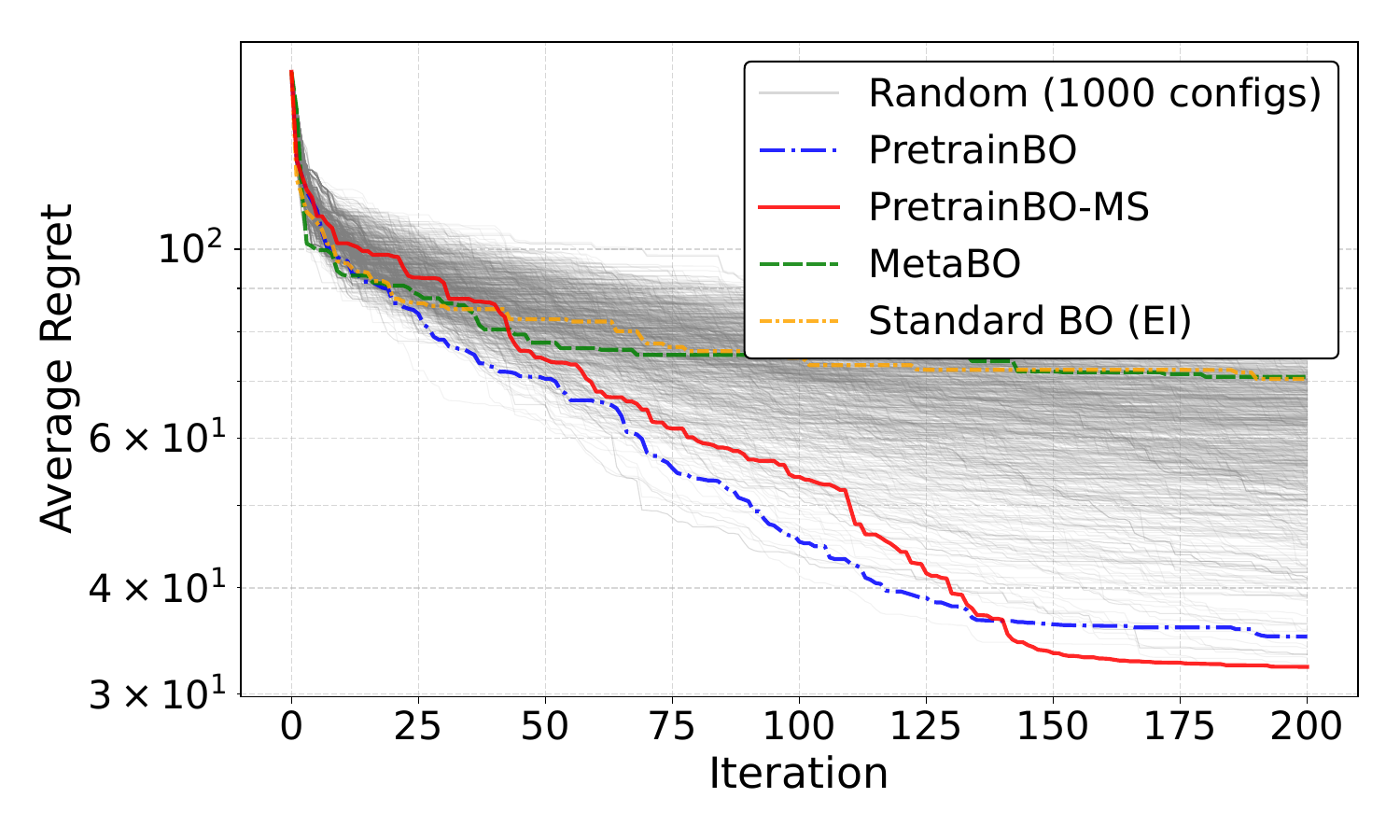}}
\subfloat[][$(d, \text{initial data}) = (20, \text{Sphere})$]{
\ig{.3}{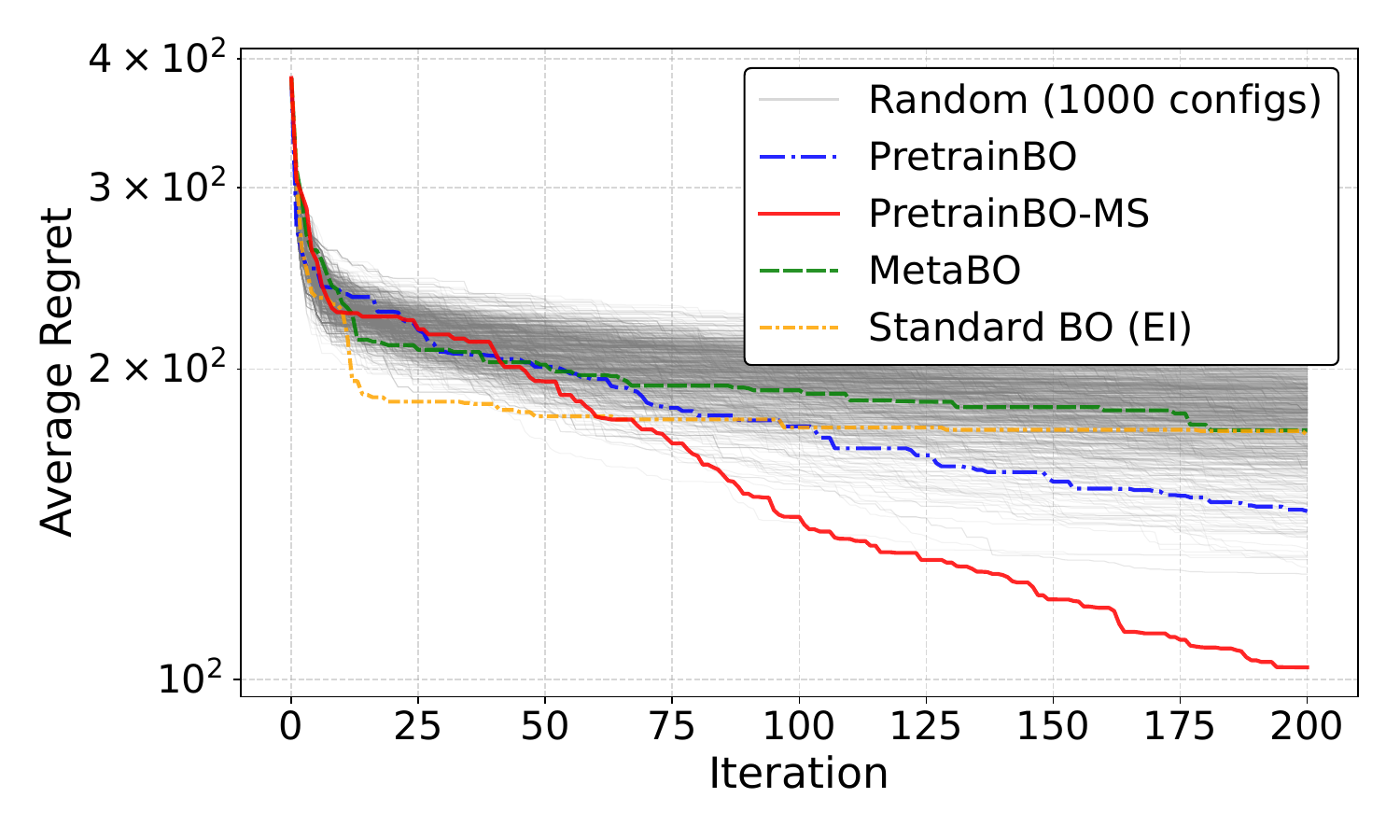}}

\subfloat[][$(d, \text{initial data})  \! \! =  \! \! (5, \text{Rosenbrock})$]{
\ig{.3}{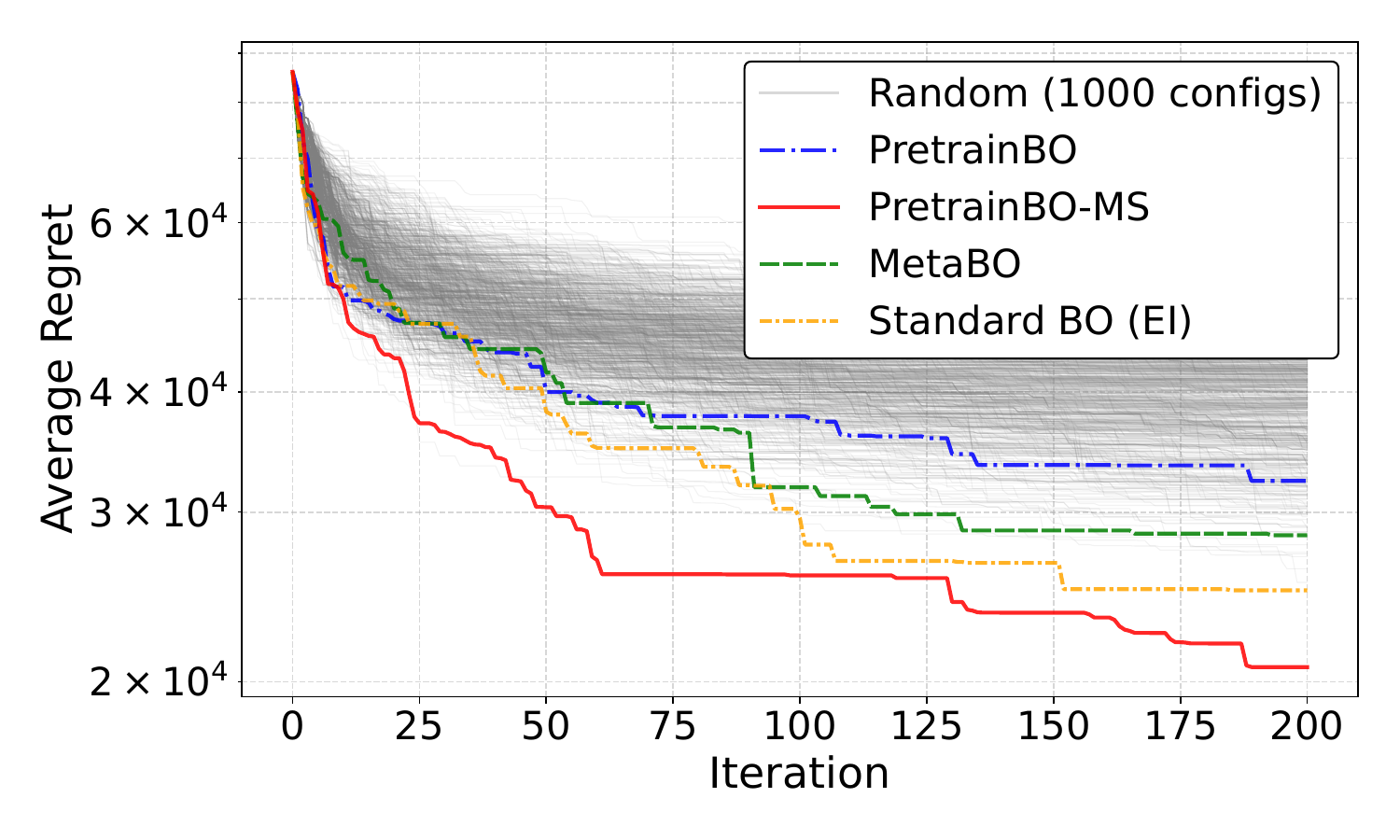}}
\subfloat[][$(d, \text{initial data}) \! \! = \! \! (10, \text{Rosenbrock})$]{
\ig{.3}{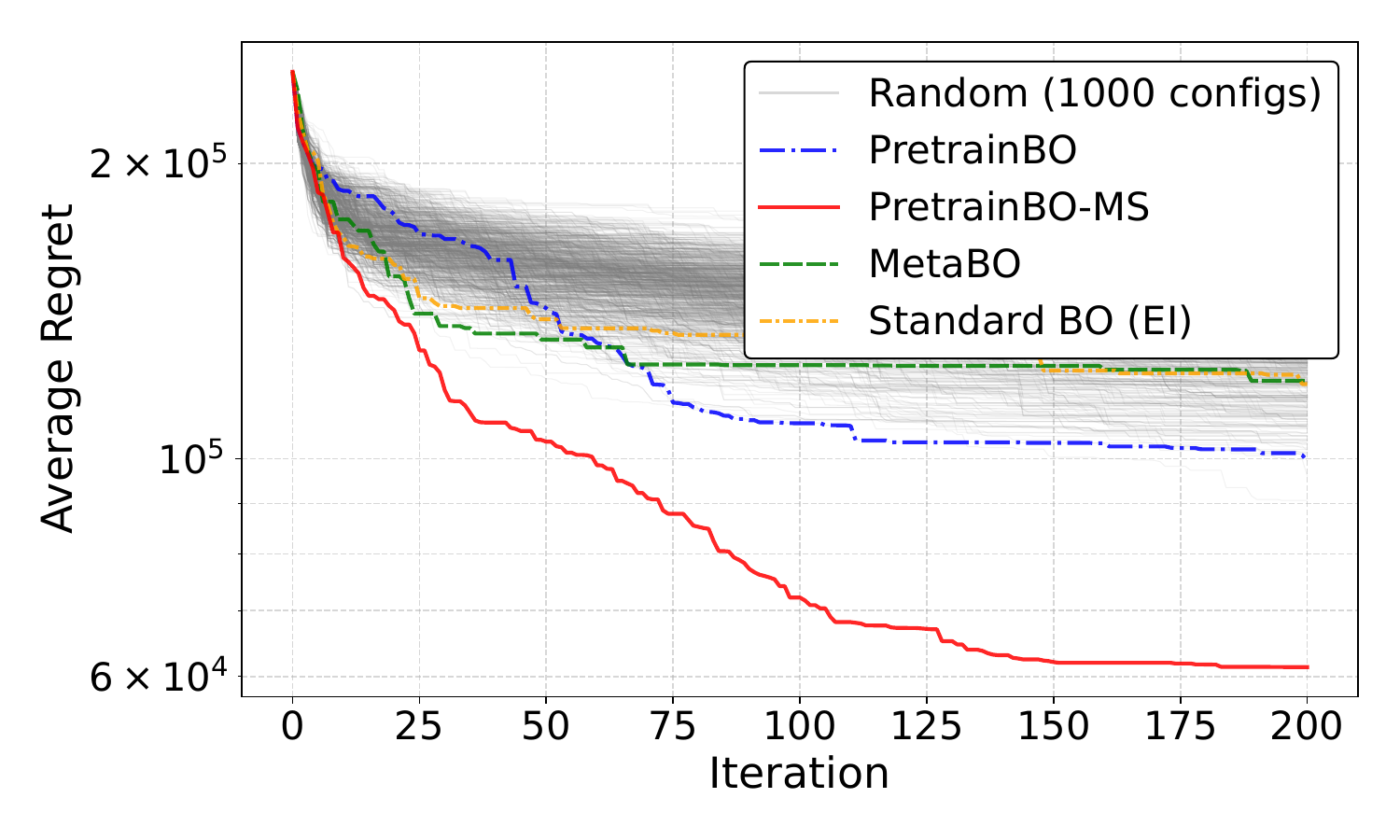}}
\subfloat[][$(d, \text{initial data})  \! \! =  \! \! (20, \text{Rosenbrock})$]{
\ig{.3}{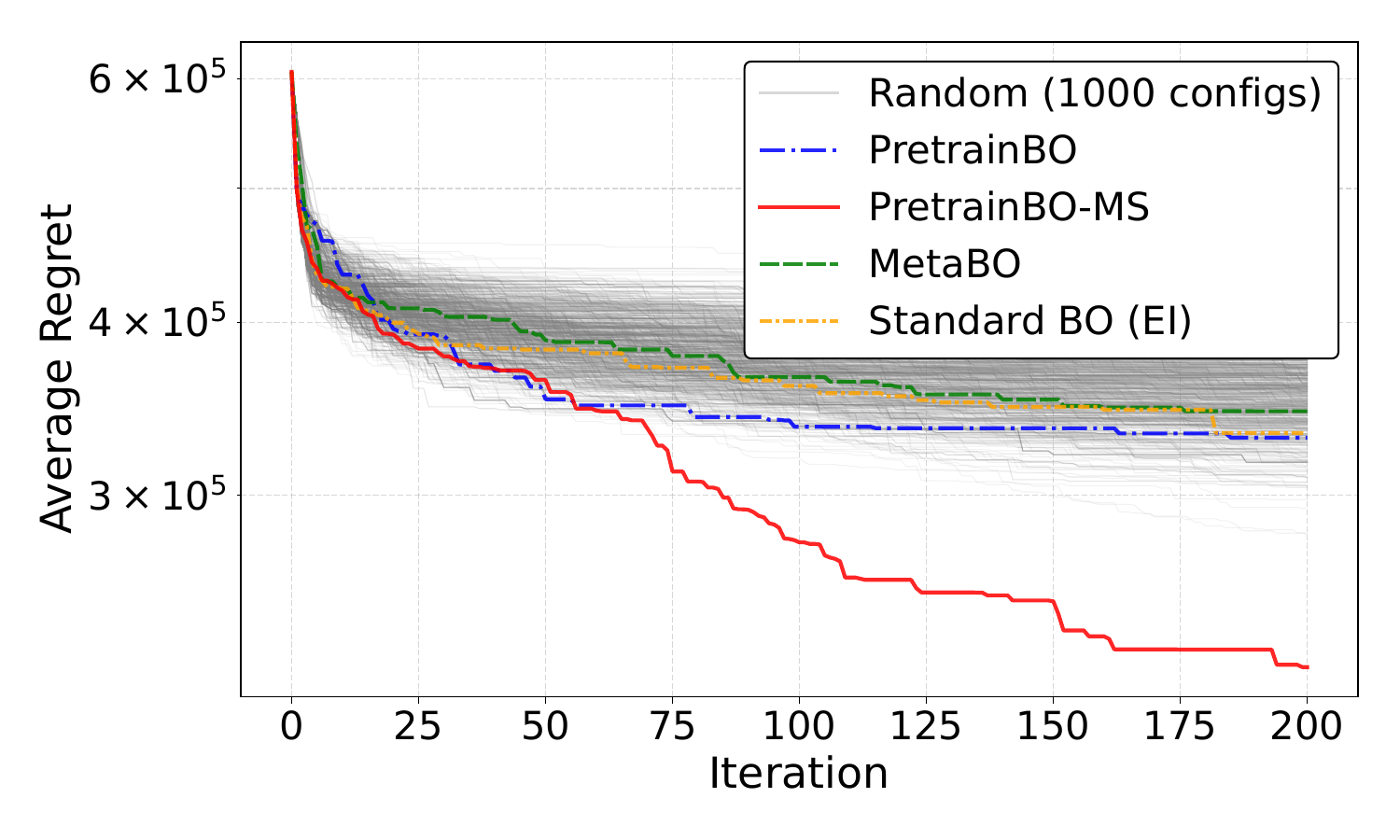}}

\subfloat[][$(d, \text{initial data}) = (5, \text{Levy})$]{
\ig{.3}{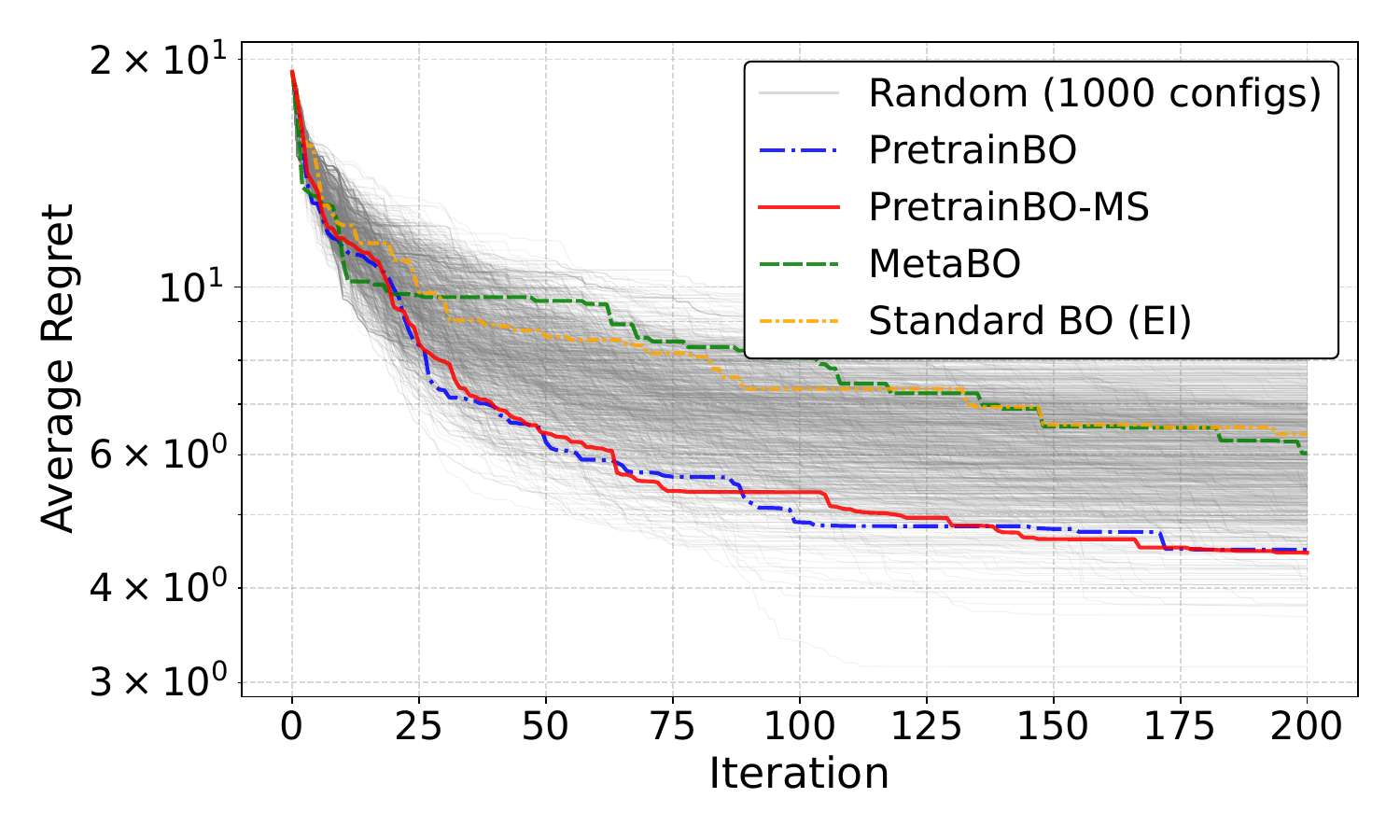}}
\subfloat[][$(d, \text{initial data}) = (10, \text{Levy})$]{
\ig{.3}{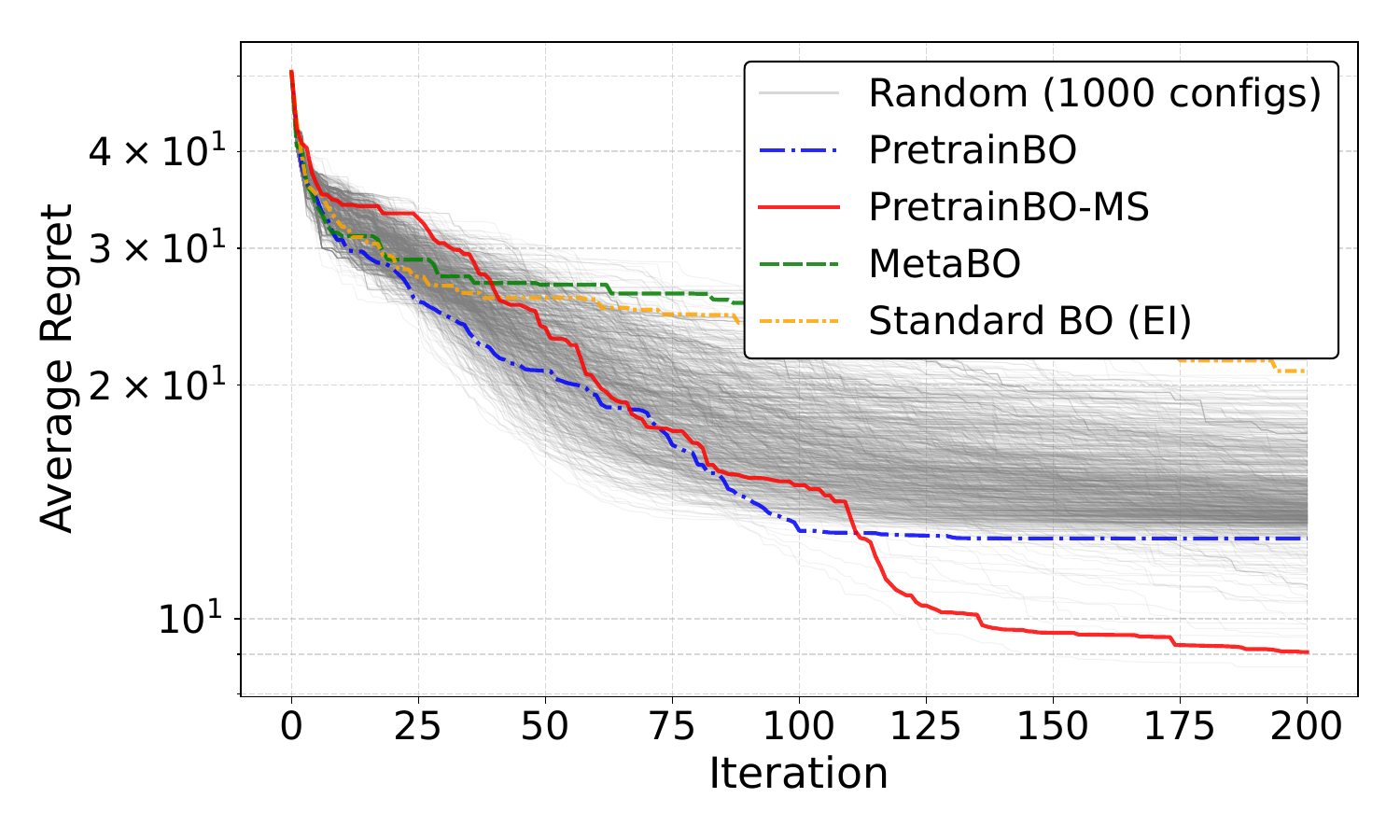}}
\subfloat[][$(d, \text{initial data}) = (20, \text{Levy})$]{
\ig{.3}{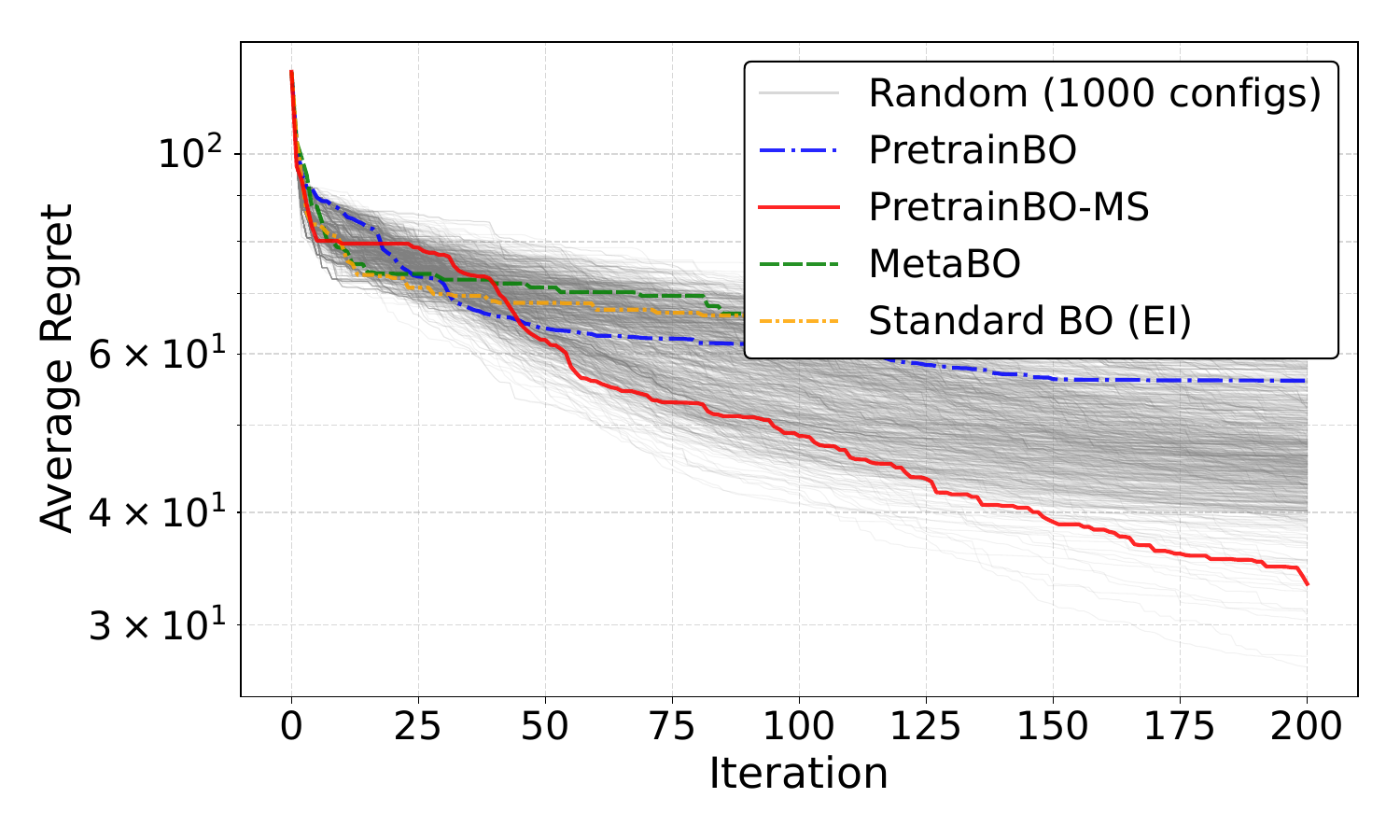}}

\caption{
Comparison of test function performance ($n_0 = 50$).
}
\label{fig:test-BO-trajectory-50}

\end{figure*}

\begin{figure*}

\centering

\subfloat[][$(d, \text{initial data}) = (5, \text{Sphere})$]{
\ig{.3}{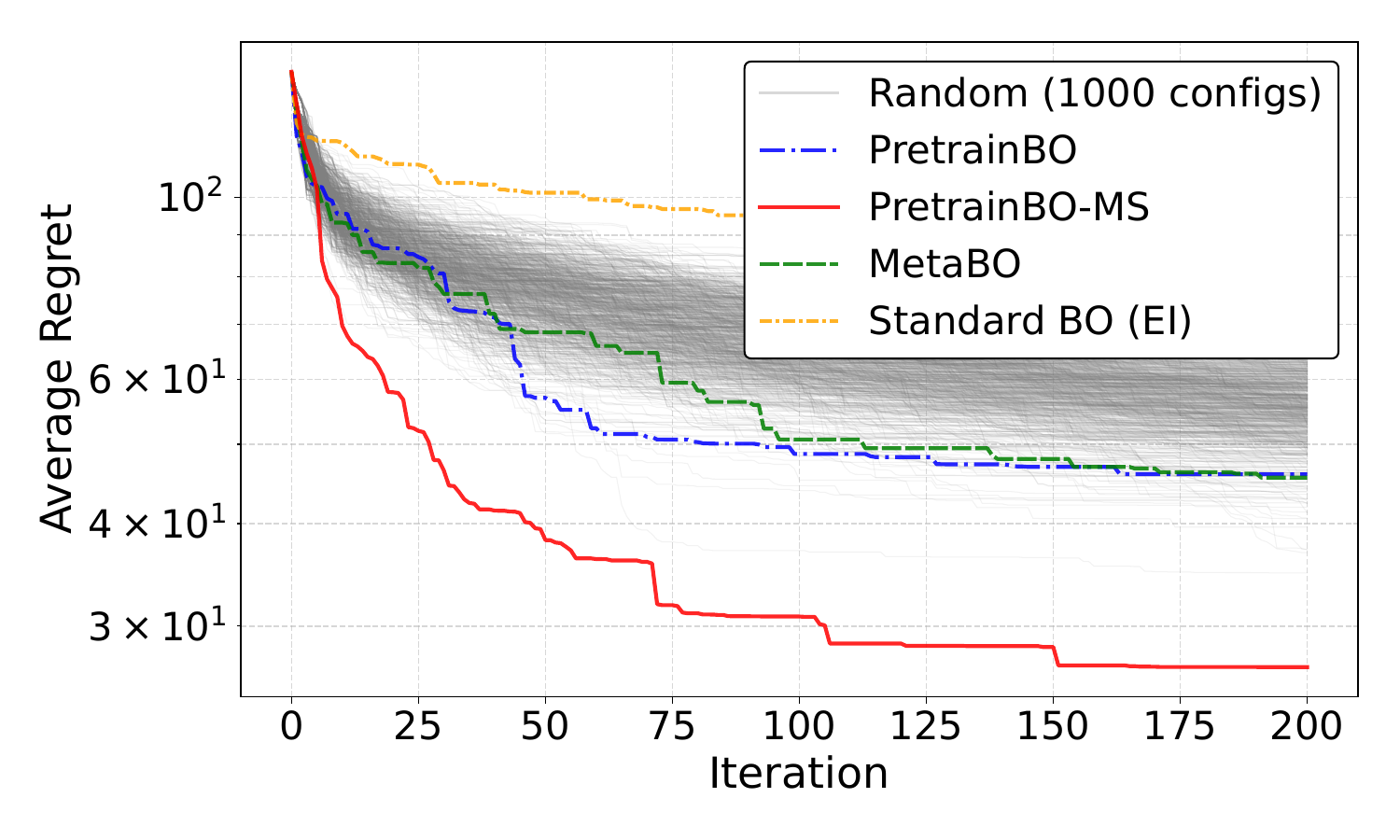}}
\subfloat[][$(d, \text{initial data}) = (10, \text{Sphere})$]{
\ig{.3}{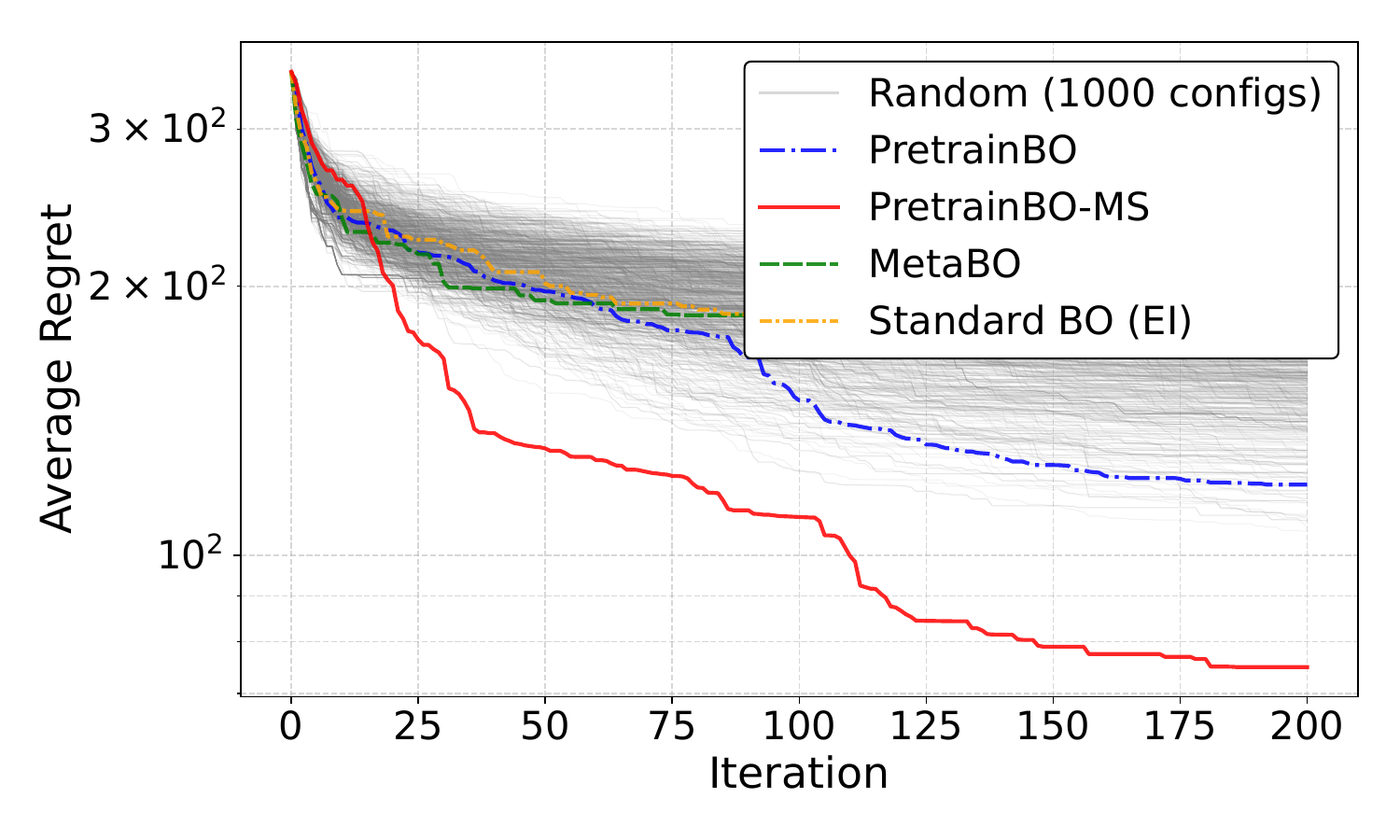}}
\subfloat[][$(d, \text{initial data}) = (20, \text{Sphere})$]{
\ig{.3}{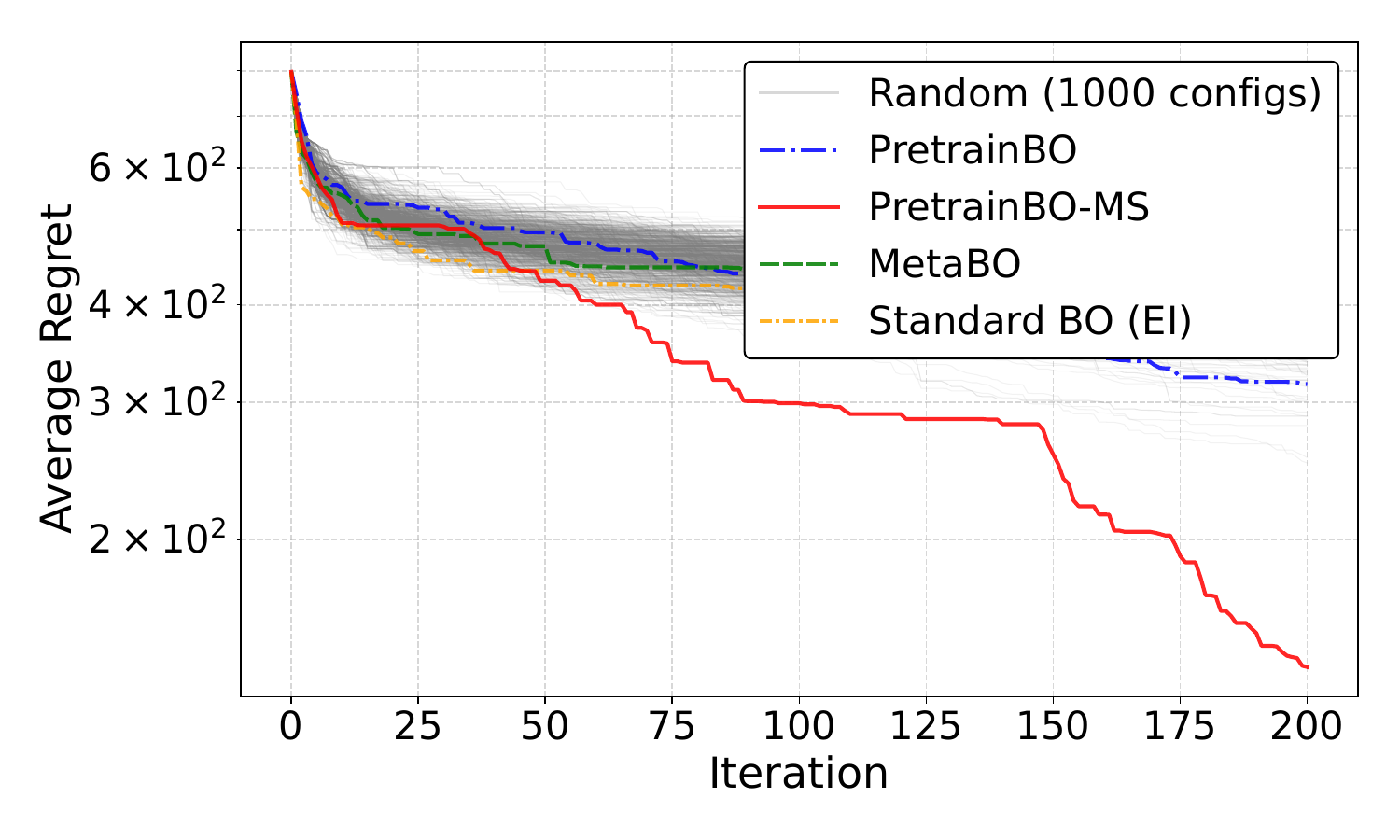}}

\subfloat[][$(d, \text{initial data})  \! \! =  \! \! (5, \text{Rosenbrock})$]{
\ig{.3}{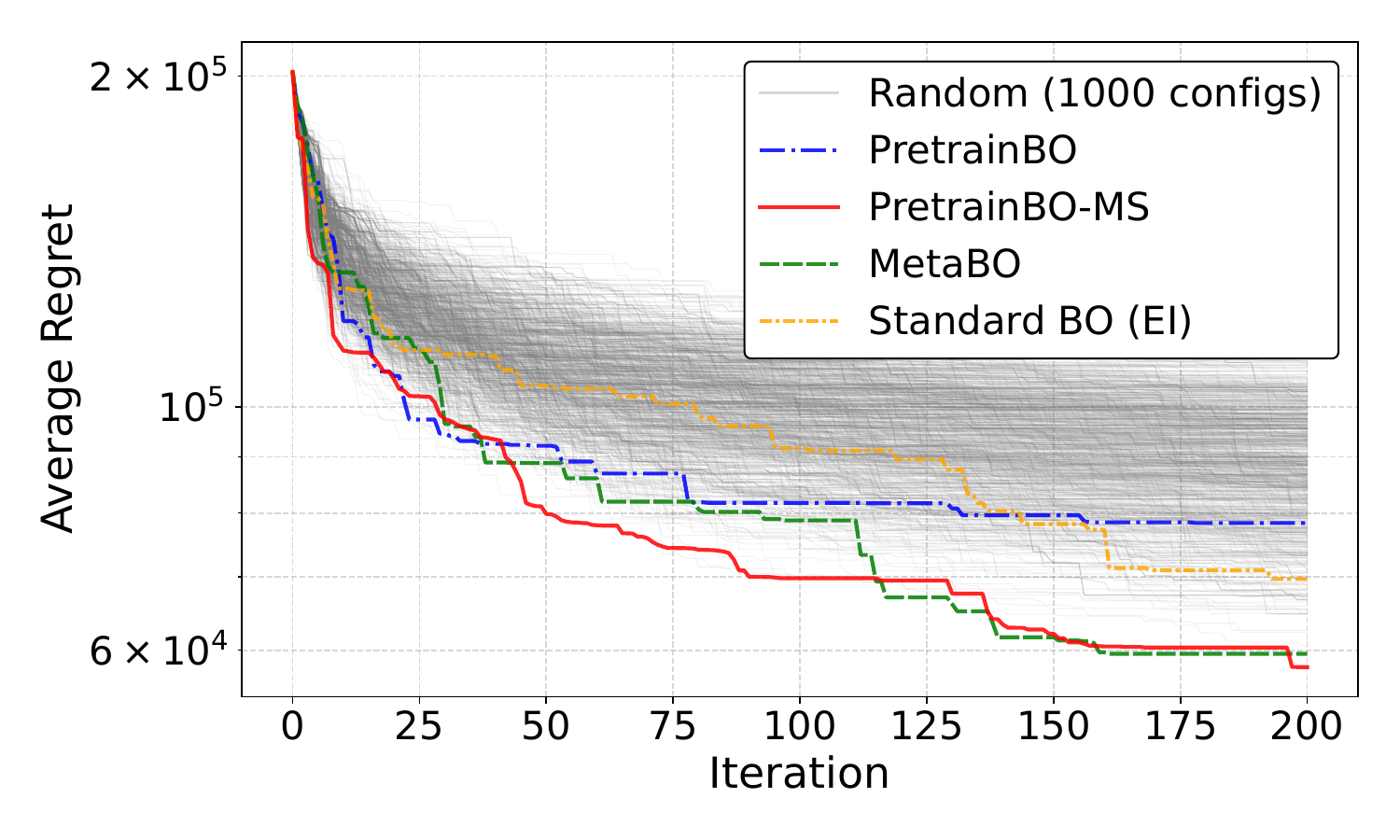}}
\subfloat[][$(d, \text{initial data}) \! \! = \! \! (10, \text{Rosenbrock})$]{
\ig{.3}{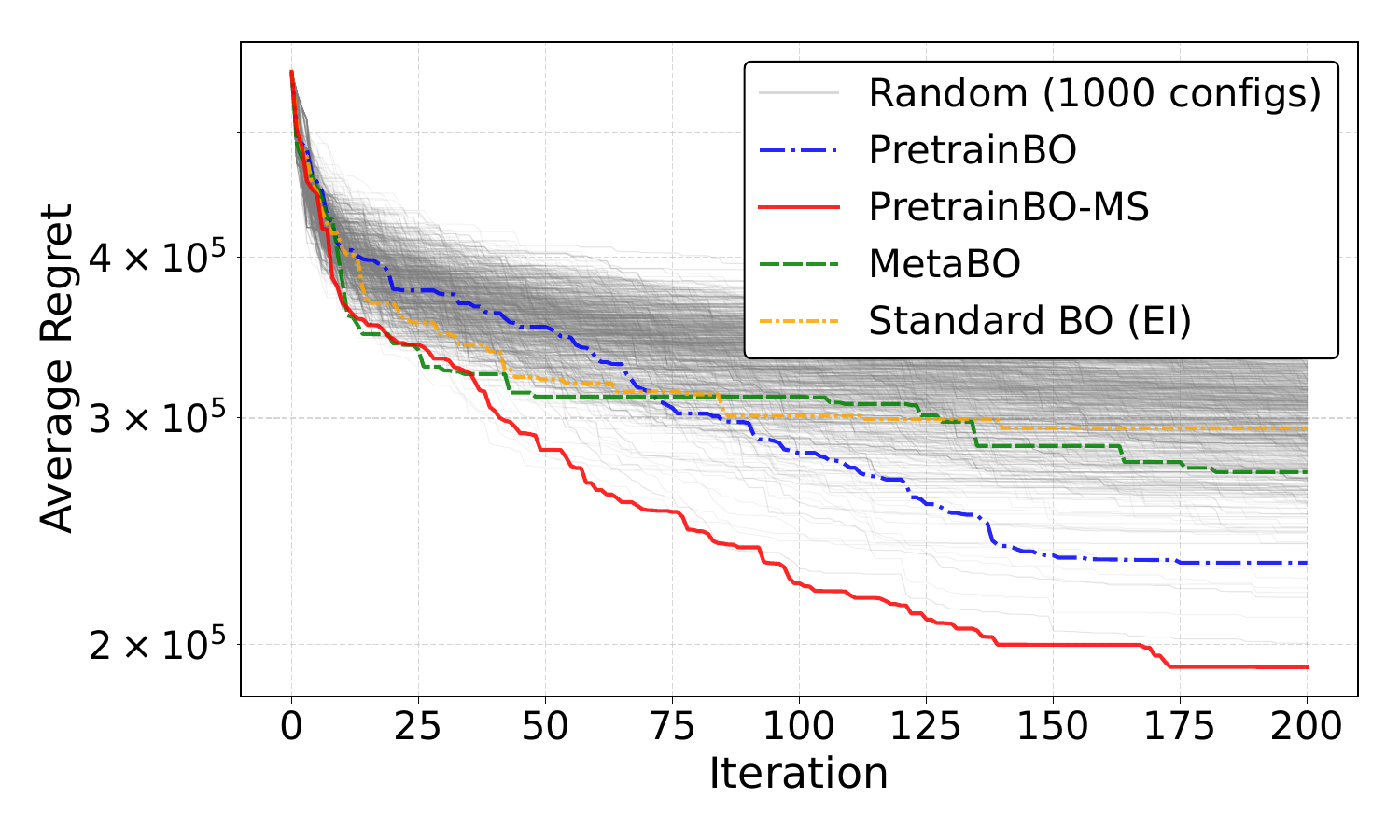}}
\subfloat[][$(d, \text{initial data})  \! \! =  \! \! (20, \text{Rosenbrock})$]{
\ig{.3}{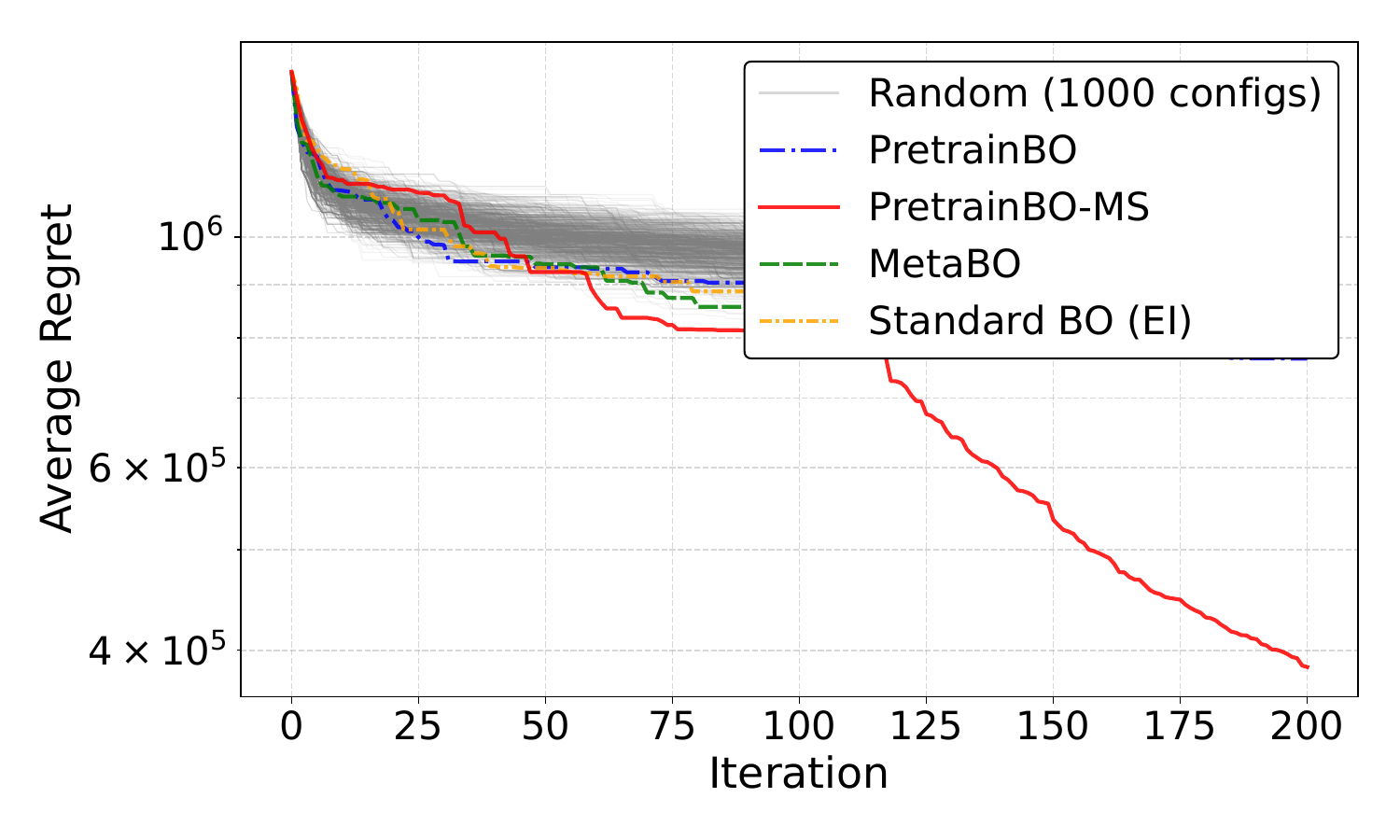}}

\subfloat[][$(d, \text{initial data}) = (5, \text{Levy})$]{
\ig{.3}{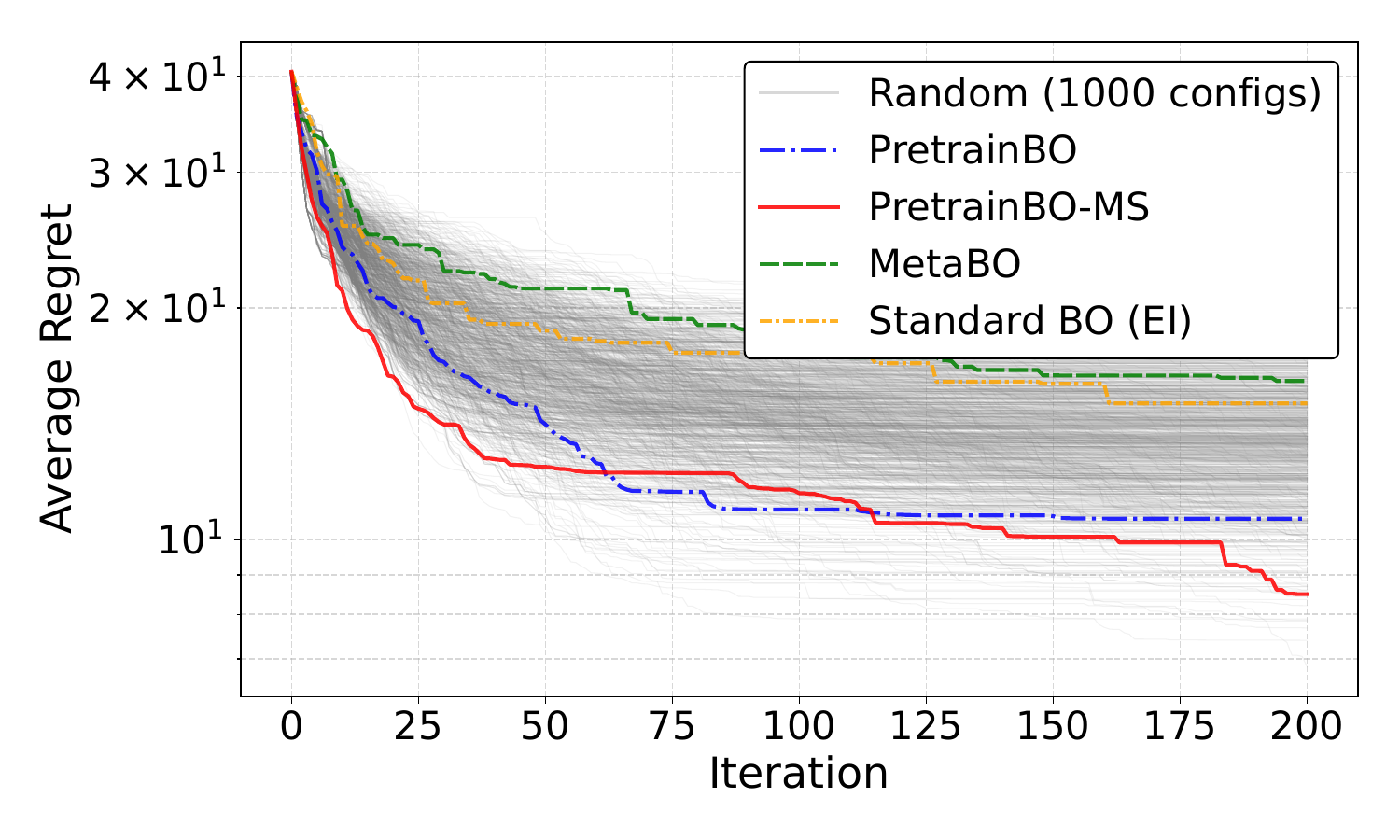}}
\subfloat[][$(d, \text{initial data}) = (10, \text{Levy})$]{
\ig{.3}{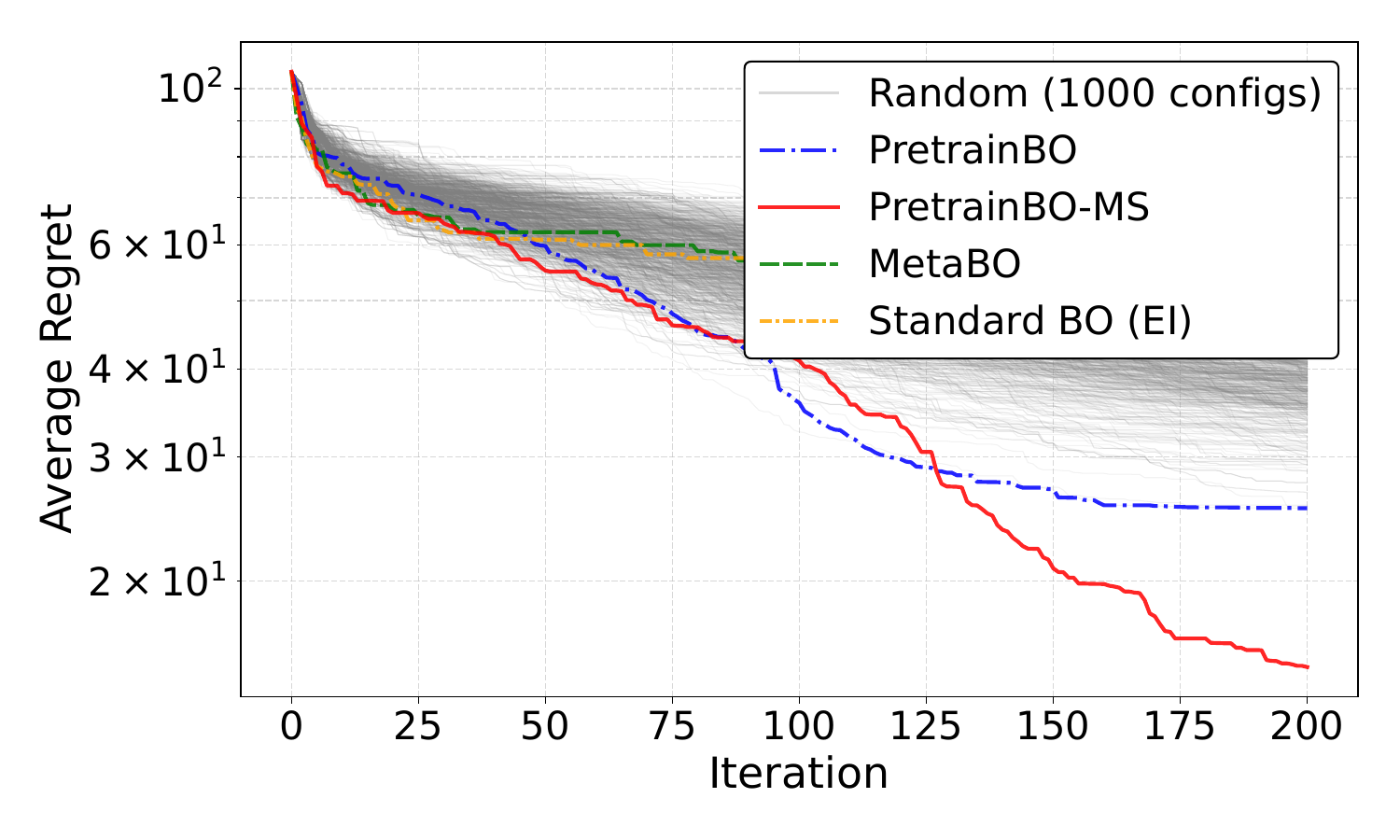}}
\subfloat[][$(d, \text{initial data}) = (20, \text{Levy})$]{
\ig{.3}{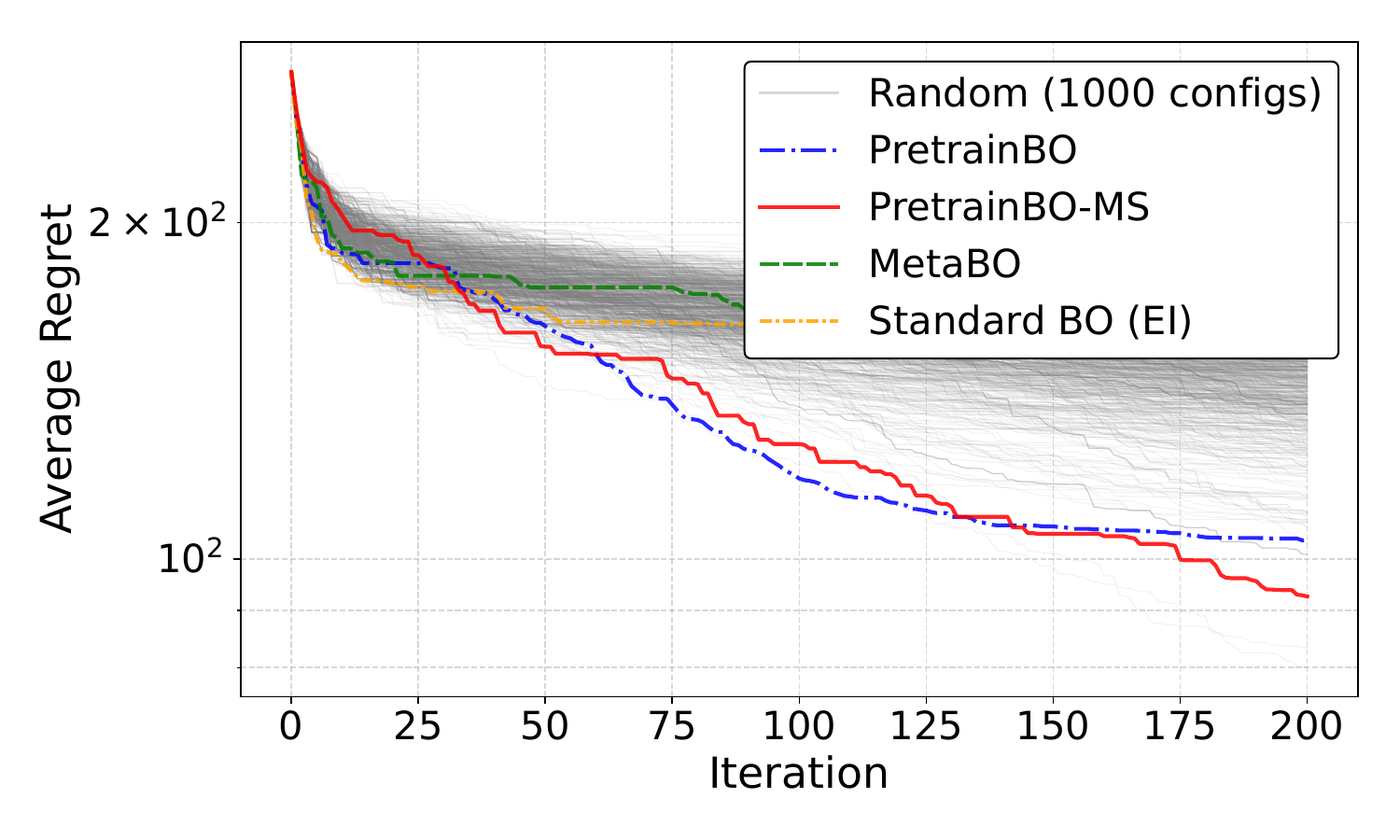}}

\caption{
Comparison of test function performance ($n_0 = 300$).
}
\label{fig:test-BO-trajectory-300}

\end{figure*}

\subsection{Results with Error Bars}
\label{app:error-bars}

Standard errors are shown in Figs.~\ref{fig:test-BO-trajectory-SE}, \ref{fig:test-BO-trajectory-50-SE}, and \ref{fig:test-BO-trajectory-300-SE}.
Note that each one of random 1000 configs has its standard error because the 10 runs were performed for each one of 1000 configs.

\begin{figure*}[t]
\subfloat[][$(d,\text{initial data}) = (5d,\text{sph})$]{
\ig{.30}{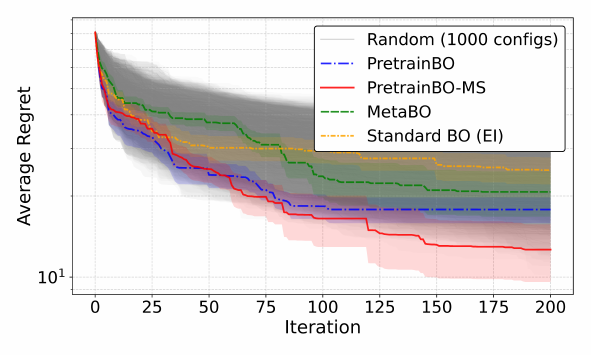}}
\subfloat[][$(d,\text{initial data}) = (10d,\text{sph})$]{
\ig{.30}{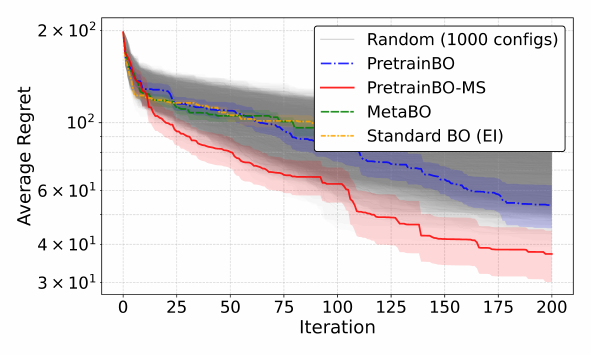}}
\subfloat[][$(d,\text{initial data}) = (20d,\text{sph})$]{
\ig{.30}{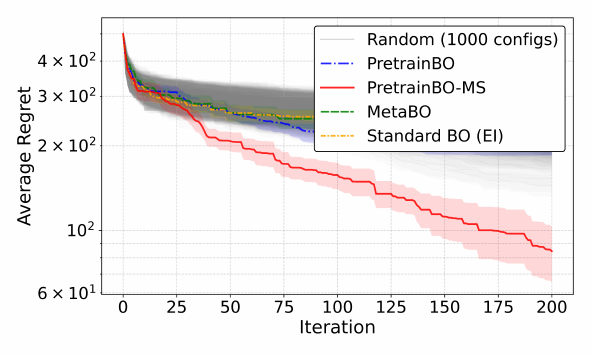}}

\subfloat[][$(d,\text{initial data}) = (5d,\text{ros})$]{
\ig{.30}{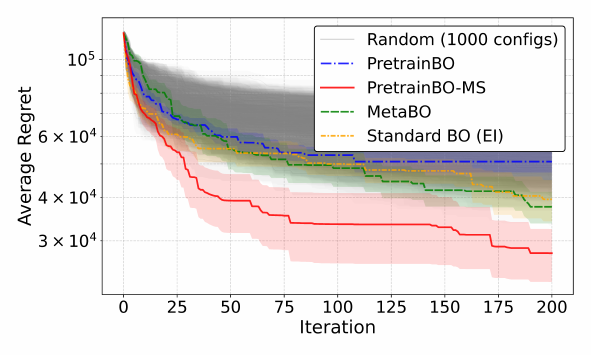}}
\subfloat[][$(d,\text{initial data}) = (10d,\text{ros})$]{
\ig{.30}{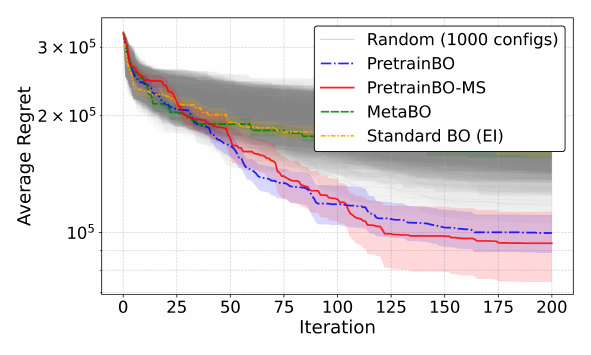}}
\subfloat[][$(d,\text{initial data}) = (20d,\text{ros})$]{
\ig{.30}{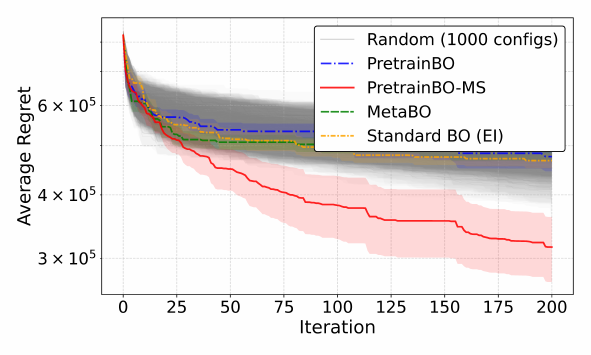}}

\subfloat[][$(d,\text{initial data}) = (5d,\text{lev})$]{
\ig{.30}{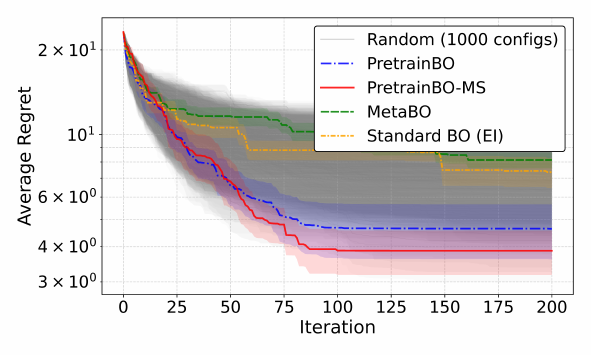}}
\subfloat[][$(d,\text{initial data}) = (10d,\text{lev})$]{
\ig{.30}{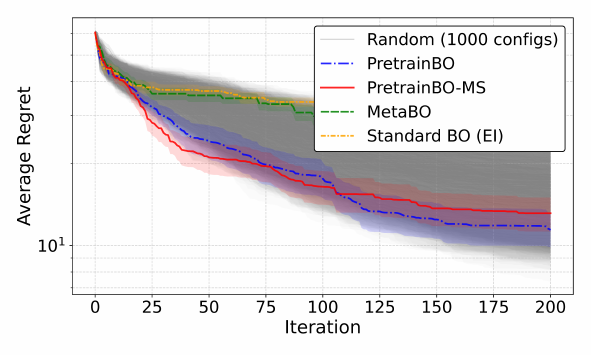}}
\subfloat[][$(d,\text{initial data}) = (20d,\text{lev})$]{
\ig{.30}{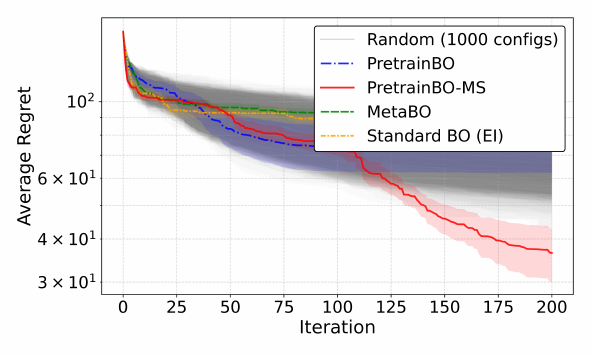}}

\caption{
Comparison of test function performance ($n_0 = 100$).
Standard error is added to Fig.~\ref{fig:test-BO-trajectory}.
}
\label{fig:test-BO-trajectory-SE}

\end{figure*}

\begin{figure*}[t]
\subfloat[][$(d,\text{initial data}) = (5d,\text{sph})$]{
\ig{.30}{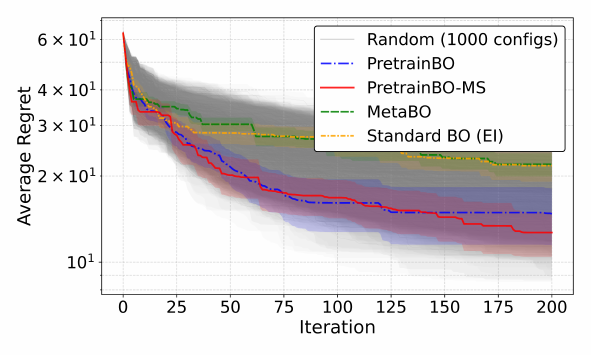}}
\subfloat[][$(d,\text{initial data}) = (10d,\text{sph})$]{
\ig{.30}{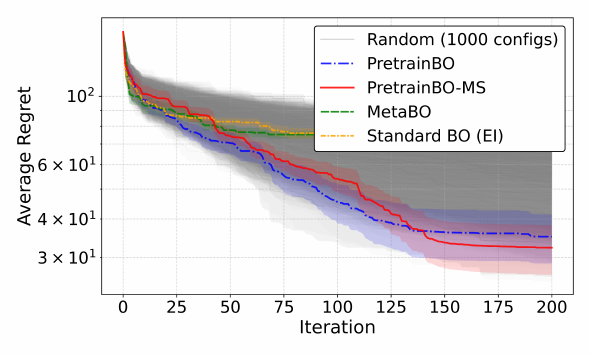}}
\subfloat[][$(d,\text{initial data}) = (20d,\text{sph})$]{
\ig{.30}{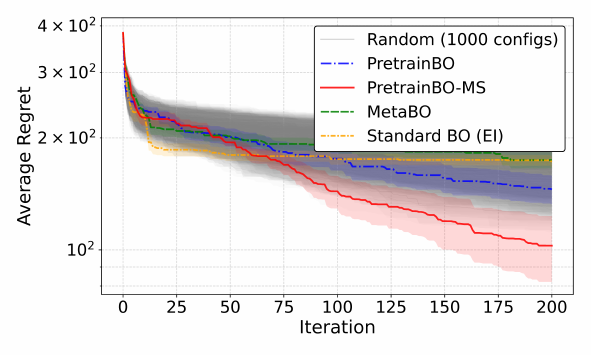}}

\subfloat[][$(d,\text{initial data}) = (5d,\text{ros})$]{
\ig{.30}{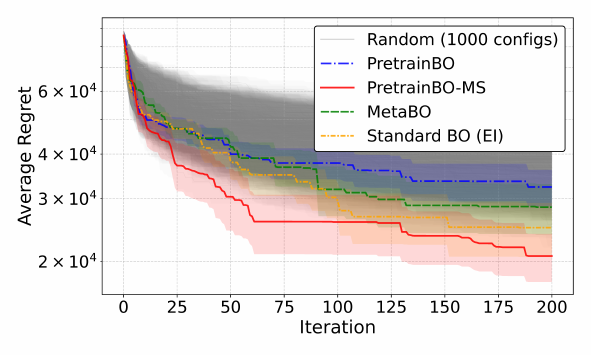}}
\subfloat[][$(d,\text{initial data}) = (10d,\text{ros})$]{
\ig{.30}{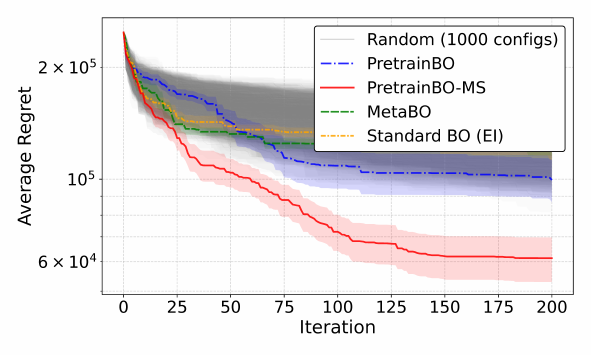}}
\subfloat[][$(d,\text{initial data}) = (20d,\text{ros})$]{
\ig{.30}{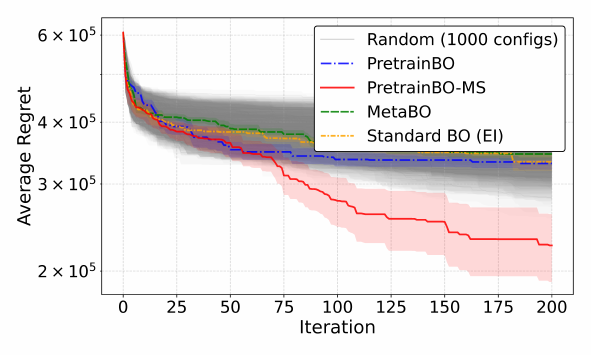}}

\subfloat[][$(d,\text{initial data}) = (5d,\text{lev})$]{
\ig{.30}{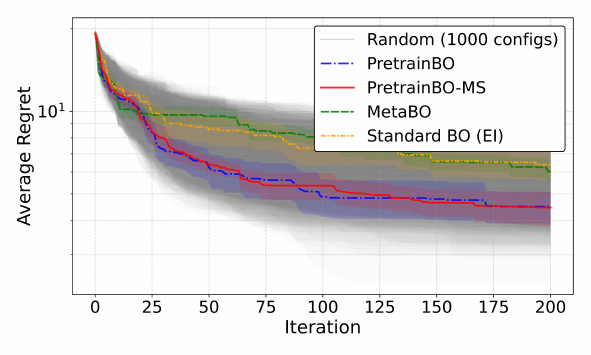}}
\subfloat[][$(d,\text{initial data}) = (10d,\text{lev})$]{
\ig{.30}{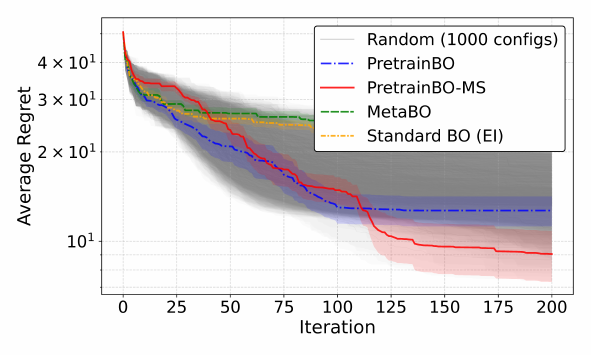}}
\subfloat[][$(d,\text{initial data}) = (20d,\text{lev})$]{
\ig{.30}{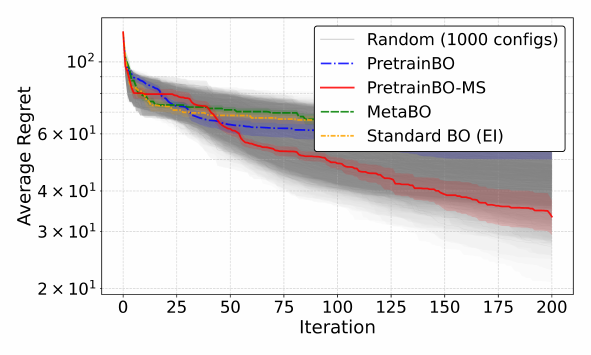}}

\caption{
Comparison of test function performance ($n_0 = 50$).
Standard error is added to Fig.~\ref{fig:test-BO-trajectory-50}.
}
\label{fig:test-BO-trajectory-50-SE}

\end{figure*}

\begin{figure*}[t]
\subfloat[][$(d,\text{initial data}) = (5d,\text{sph})$]{
\ig{.30}{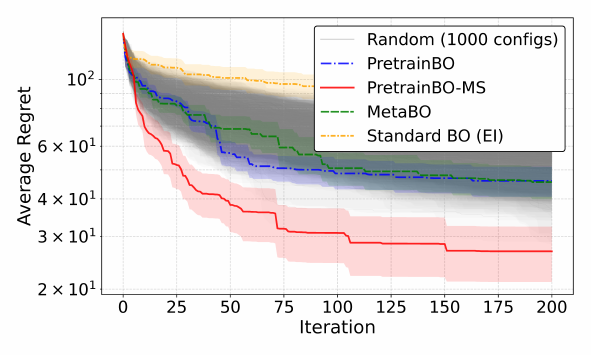}}
\subfloat[][$(d,\text{initial data}) = (10d,\text{sph})$]{
\ig{.30}{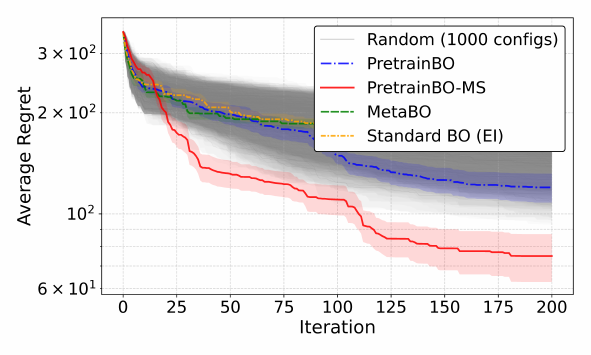}}
\subfloat[][$(d,\text{initial data}) = (20d,\text{sph})$]{
\ig{.30}{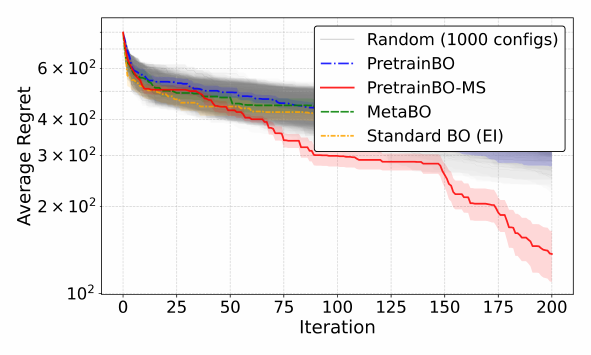}}

\subfloat[][$(d,\text{initial data}) = (5d,\text{ros})$]{
\ig{.30}{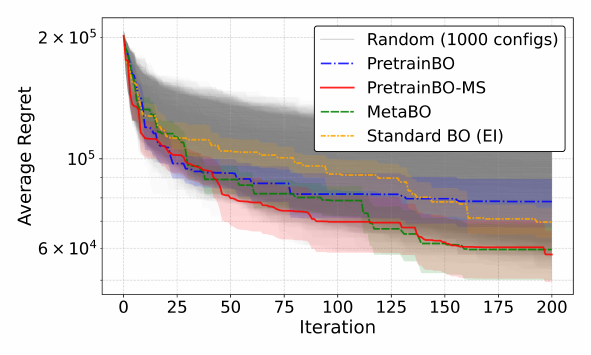}}
\subfloat[][$(d,\text{initial data}) = (10d,\text{ros})$]{
\ig{.30}{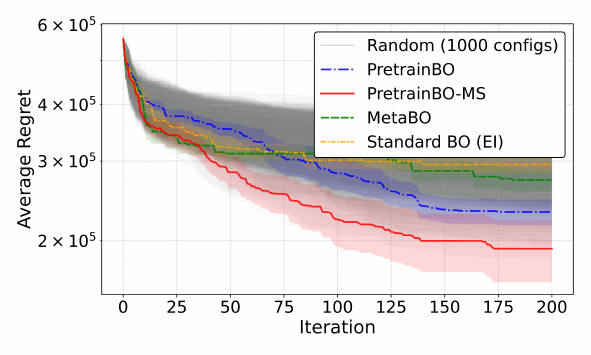}}
\subfloat[][$(d,\text{initial data}) = (20d,\text{ros})$]{
\ig{.30}{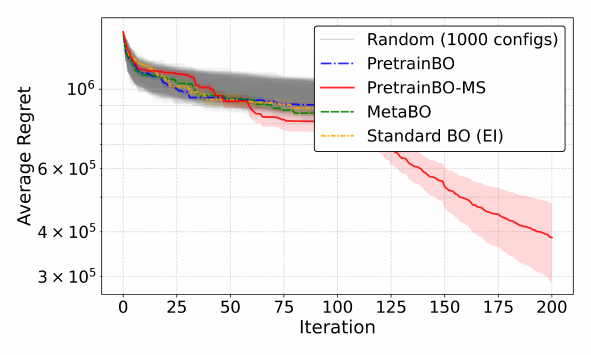}}

\subfloat[][$(d,\text{initial data}) = (5d,\text{lev})$]{
\ig{.30}{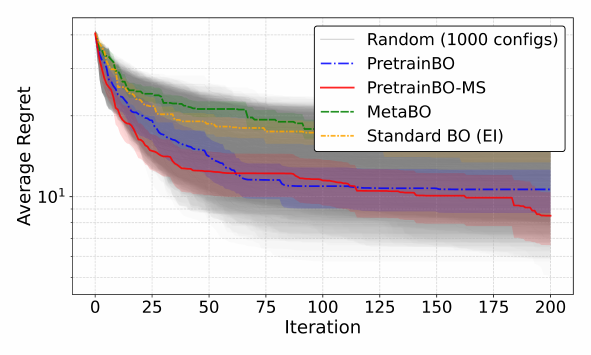}}
\subfloat[][$(d,\text{initial data}) = (10d,\text{lev})$]{
\ig{.30}{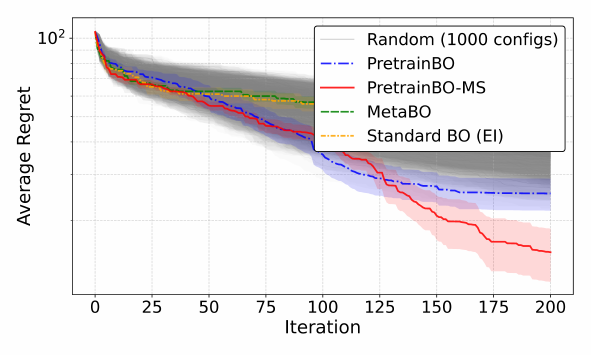}}
\subfloat[][$(d,\text{initial data}) = (20d,\text{lev})$]{
\ig{.30}{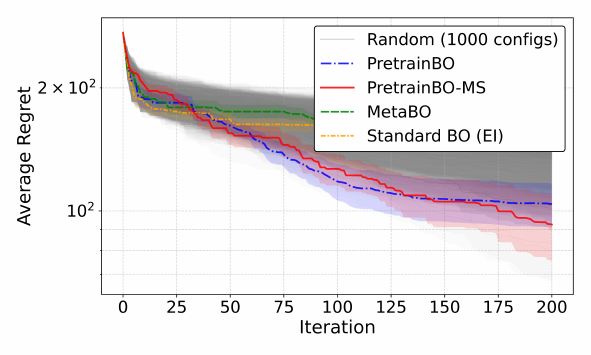}}

\caption{
Comparison of test function performance ($n_0 = 300$).
Standard error is added to Fig.~\ref{fig:test-BO-trajectory-300}.
}
\label{fig:test-BO-trajectory-300-SE}

\end{figure*}

\section{Proof of Proposition~\ref{prop:one-shot-error}}
\label{app:proof-one-shot-error}

In the proof, we use the following version of Dudley's entropy bound and Borell--TIS inequality, adapted to our notation from \citep{adler2007random}:
%
\begin{theo}[Dudley's entropy bound {\citep[adopted from][]{adler2007random}} ]
\label{thm:dudley}
Let 
$f = \{f(\*x): \*x \in \cX \}$ 
be a centered separable Gaussian process, and define its canonical pseudometric by
\[
    d(\*x,\*x')
    \coloneqq
    \left(
        \EE\left[(f(\*x)-f(\*x^\prm))^2\right]
    \right)^{1/2},
    \qquad \*x,\*x^\prm \in \cX.
\]
Suppose that $\cX$ is $d$-compact (compact with respect to the canonical pseudometric $d$).
Let
\[
    N(\cX,d,\epsilon)
\]
denote the minimum number of $d$-balls (balls under the metric $d$) of radius $\epsilon$ required to cover $\cX$.
Then there exists a universal constant $C>0$ such that
\[
    \EE\left[
        \sup_{\*x \in \cX} f(\*x)
    \right]
    \leq
    C
    \int_0^{\operatorname{diam}(\cX,d)}
    \sqrt{
        \log N(\cX,d,\epsilon)
    }\,
    d\epsilon,
\]
where
\[
    \operatorname{diam}(\cX,d)
    \coloneqq
    \sup_{\*x,\*x^\prm \in \cX} d(\*x^\prm,\*x^\prm).
\]
In particular, if the entropy integral on the right-hand side is finite, then 
\[
 \EE[\sup_{\*x \in \cX} f(\*x)] < \infty. 
\]
\end{theo}


\begin{theo}[Borell--TIS inequality {\citep[adapted from][]{adler2007random}}]
\label{thm:borell-tis}
Let $f=\{f(x):x\in\mathcal{X}\}$ be a centered Gaussian process that is almost surely bounded on $\mathcal{X}$. Define
\[
    \sigma_{\cX}^2
    \coloneqq
    \sup_{\*x \in \cX}
    \operatorname{Var}(f(\*x)).
\]
Then
\[
    \EE\left[
        \sup_{x\in\mathcal{X}} f(\*x)
    \right]
    < \infty,
\]
and, for every $u>0$,
\[
    \mathbb{P}\left(
        \sup_{\*x \in \cX} f(\*x)
        >
        \EE\left[
            \sup_{x\in\mathcal{X}} f(\*x)
        \right]
        + u
    \right)
    \leq
    \exp\left(
        -\frac{u^2}{2\sigma_{\cX}^2}
    \right).
\]
\end{theo}

Let
\begin{align*}
 \cD_0
 =
 \{(x_j^{(0)},y_j^{(0)})\}_{j=1}^{n_0}
\end{align*}
be a fixed initial dataset and
\begin{align*}
 g(x) &\coloneqq f(x) - m(x),  \\
 g_0(x) &\coloneqq f(x) - \mu_0(x)
\end{align*}
be the `centered' processes.
Then,
\begin{align*}
 g
 & \sim
 \cG\cP(0, k),
 \\
 g_0 \mid \cD_0
 & \sim
 \cG\cP(0, \Sigma_0). 
\end{align*}

We first derive the finiteness of the second moment of $g_0$ given $\cD_0$.
\begin{lem}
 \label{lem:second-moment}
 Let
 \begin{align*}
  M_0 \coloneqq
  \sup_{\*x \in \cX} |g_0(\*x)|.
 \end{align*}
 Under the assumptions of Proposition~\ref{prop:one-shot-error}, 
 \begin{align}
  C_{\cD_0} =
  \EE[
  M_0^2
  \mid \cD_0
  ]
  < \infty.
  \label{eq:second-moment-M0}
 \end{align}
\end{lem}

\begin{proof}
From the definition of the canonical pseudo-metric $d$,
\begin{align*}
 d(\*x,\*x^\prm)^2
 =
 \EE[
 \{ g(\*x)-g(\*x^\prm) \}^2
 ]
 =
 k(\*x,\*x)+k(\*x^\prm,\*x^\prm)-2k(\*x,\*x^\prm).
\end{align*}
We further define the canonical pseudo-metric conditioned on $\cD_0$ as
\begin{align*}
 d_0(\*x,\*x^\prm)
 \coloneqq
 \left(
 \EE[
 \{f(\*x) - f(\*x^\prm)\}^2
 \mid \cD_0
 ]
 \right)^{1/2}. 
\end{align*}
Then, 
\begin{align}
 d_0(\*x,\*x^\prm)^2
 =
 \EE[
 \{ g_0(\*x)-g_0(\*x^\prm) \}^2
 \mid \cD_0
 ]
 = 
 \Sigma_0(\*x,\*x)
 +
 \Sigma_0(\*x^\prm,\*x^\prm)
 -
 2\Sigma_0(\*x,\*x^\prm)
 \label{eq:d_0}
\end{align}
Let
$k_{X_0}(\*x) \coloneqq [k(\*x_1^{(0)},\*x), \ldots, k(\*x_{n_0}^{(0)},\*x)]^\top$
and 
$\*v(\*x,\*x^\prm) \coloneqq k_{X_0}(\*x) - k_{X_0}(\*x^\prm)$. 
By expanding the GP posterior covariance $\Sigma_0$ in \eq{eq:d_0}, we have
\begin{align*}
 d_0(\*x,\*x^\prm)^2
 = &
 k(\*x,\*x) + \sigma_{\veps}^2 - k_{X_0}(\*x)^\top (K_0 + \sigma_{\veps}^2 I)^{-1} k_{X_0}(\*x)
 \\
 & + k(\*x^\prm,\*x^\prm) + \sigma_{\veps}^2 - k_{X_0}(\*x^\prm)^\top (K_0 + \sigma_{\veps}^2 I)^{-1} k_{X_0}(\*x^\prm)
 \\
 & 
 - 2 \left( k(\*x,\*x^\prm) + \sigma_{\veps}^2 - k_{X_0}(\*x)^\top (K_0 + \sigma_{\veps}^2 I)^{-1} k_{X_0}(\*x^\prm) \right)
 \\
 = & 
 d(\*x,\*x^\prm) - \*v(\*x,\*x^\prm) (K_0 + \sigma_\veps^2 I)^{-1} \*v(\*x,\*x^\prm).
\end{align*}
Since
$K_0+\sigma_\varepsilon^2I\succ0$,
the second term is positive, by which we see
\[
    d_0(\*x,\*x^\prm)\le d(\*x,\*x^\prm)
    \qquad
    \forall \*x,\*x^\prm \in \cX.
\]
As a result, from the assumption \eq{eq:metric}, there exists $r_0$ that satisfies 
\begin{align}
    d_0(\*x,\*x^\prm)
    \le
    L\|\*x - \*x^\prm \|_2^\alpha
    \qquad
    \text{if }
    \|\*x - \*x^\prm\|_2 \le r_0.
 \label{eq:ub-d_0}
\end{align}


For sufficiently small $\epsilon>0$, define
\begin{align} 
 r_\epsilon
 \coloneqq
 \left(\frac{\epsilon}{L}\right)^{1/\alpha}.
 \label{eq:r-eps}
\end{align}
In particular, if
$0 < \epsilon \le L r_0^\alpha$,
then $r_\epsilon \le r_0$. 
Whenever
$\|\*x - \*x^\prm \|_2 \le r_\epsilon$, 
\[
 d_0(\*x,\*x')
 \le
 L\|\*x - \*x^\prm \|_2^\alpha
 \le
 L r_\epsilon^\alpha
 =
 \epsilon.
\]
Therefore, every Euclidean $r_\epsilon$-cover of $\cX$
is also a $d_0$-$\epsilon$-cover, and thus
\begin{align}
    N(\cX, d_0, \epsilon)
    \le
    N(\cX, \|\cdot\|_2, r_\epsilon),
 \label{eq:euc-bound-d0-covering-number}
\end{align}
where
\[
    N(\cX, \rho, \epsilon)
\]
denote the minimum number of balls of radius $\epsilon$, defined by a metric $\rho$, that covers $\cX$.

%
Since $\cX \subset \RR^d$ is compact, it is contained in a Euclidean ball $B_2^d(R)$ of some finite radius $R$.
The standard volumetric covering-number bound for Euclidean balls
\citep{vershynin2018high} 
gives
$N(B_2^d(R),\|\cdot\|_2,r) \le \left(1+\frac{2R}{r}\right)^d$.
%
For sufficiently small $r > 0$,
\begin{align}
 N(\cX,\|\cdot\|_2,r)
 &\le
 N(B_2^d(R),\|\cdot\|_2,r)
 \le
 \left(\frac{3R}{r}\right)^d
 =
 A_{\mathcal X}r^{-d},
 \label{eq:upper-bound-covering-num}
\end{align}
where $A_{\cX}:=(3R)^d<\infty$.
%
Combining \eq{eq:euc-bound-d0-covering-number}, \eq{eq:upper-bound-covering-num} and \eq{eq:r-eps}, for all sufficiently small $\epsilon>0$,
\[
\begin{aligned}
    N(\mathcal X,d_0,\epsilon)
    &\le
    A_{\mathcal X}r_\epsilon^{-d}\\
    &=
    A_{\mathcal X}
    \left(
        \frac{L}{\epsilon}
    \right)^{d/\alpha}.
\end{aligned}
\]
This indicates that $(\cX, d_0)$ is totally bounded ($N(\cX,d_0,\epsilon) < \infty$ for $\forall \epsilon > 0$) \citep{wu2020infostats}.
Then,
\begin{align} 
    \log N(\mathcal X,d_0,\epsilon)
    \le
    \log A_{\mathcal X}
    +
    \frac{d}{\alpha}
    \log\left(\frac{L}{\epsilon}\right)
 \label{eq:upper-bound-log-covering-num}
\end{align}
for all sufficiently small $\epsilon > 0$.

Define
\[
    S_0^+
    \coloneqq
    \sup_{\*x \in \cX} g_0(\*x),
    \qquad
    S_0^-
    \coloneqq
    \sup_{\*x \in \cX} \{-g_0(\*x)\}.
\]
Since $g_0$ is the centered posterior, 
we see
$S^+_0 \stackrel{d}{=} S^-_0$ 
given $\cD_0$.
Let
\[
 m_0 \coloneqq 
    \EE\left[
        \left.
        S^+_0
        \right|
        \mathcal D_0
    \right]
    =
    \EE\left[
        \left.
        S^-_0
        \right|
        \mathcal D_0
    \right].
\]
%
According to Dudley's entropy theorem~\ref{thm:dudley}, 
for totally bounded $(\cX, d_0)$, we see 
\begin{align*}
m_0 \leq C J_0 
\end{align*}
where
\begin{align*}
 J_0
 & \coloneqq
 \int_0^{\operatorname{diam}(\cX,d_0)}
 \sqrt{\log N(\cX,d_0,\epsilon)}
 \,d\epsilon
\end{align*}
$C < \infty$ is a constant, 
and
$\operatorname{diam}(\cX,d_0) \coloneqq \sup_{\*x,\*x^\prm \in \cX} d_0(\*x,\*x^\prm)$, which is finite because of 
$d_0(\*x,\*x^\prm) \leq L \|\*x - \*x^\prm\|_2^\alpha$. 
%
The finiteness of $J_0$ is confirmed through \eq{eq:upper-bound-log-covering-num}.
From \eq{eq:upper-bound-log-covering-num}, we have
$
    \sqrt{
        \log N(\mathcal X,d_0,\epsilon)
    }
    =
    O\left(
        \sqrt{\log(1/\epsilon)}
    \right)
$ when $\epsilon \downarrow 0$. 
There exists $\epsilon_0 > 0$ by which any $\epsilon \leq \epsilon_0$ satisfies \eq{eq:upper-bound-log-covering-num}.
\begin{align*}
 J_0
 & =
 \int_0^{\epsilon_0}
 \sqrt{\log N(\cX,d_0,\epsilon)}
 \,d\epsilon 
 +
 \int_{\epsilon_0}^{\operatorname{diam}(\cX,d_0)}
 \sqrt{\log N(\cX,d_0,\epsilon)}
 \,d\epsilon 
 \\
 & \leq 
 \int_0^{\epsilon_0}
 \sqrt{\log N(\cX,d_0,\epsilon)}
 \,d\epsilon 
 +
 \underbrace{
 (\operatorname{diam}(\cX,d_0) - \epsilon_0)
 \sqrt{\log N(\cX,d_0,\epsilon_0)}
 }_{< \infty}
\end{align*}
The inequality of the second line is from the non-increasing property of the covering number with respect to $\epsilon$.
Since $\sqrt{\log N(\cX,d_0,\epsilon)} = O(\sqrt{\log(1/\epsilon)})$ for $\epsilon \leq \epsilon_0$, we only need to show the finiteness 
$\int_0^{\epsilon_0} \sqrt{\log(1/\epsilon)} \,d\epsilon$.
This can be seen from
$
    \int_0^{\epsilon_0}
    \sqrt{\log(1/\epsilon)}\,d\epsilon
    =
    \int_{\log(1/\epsilon_0)}^\infty
    \sqrt{u}\,e^{-u}\,du
    <\infty
$.
As a result, we see
\[
 m_0 < \infty.
\]

Since 
$m_0 < \infty$, 
we see
$S_0^+ < \infty$
and
$S_0^- < \infty$ 
almost surely (cannot be $\infty$ with a positive probability) for the given $\cD_0$. 
This means that $g_0$ conditioned on $\cD_0$ is almost surely bounded.
Since $g_0$ given $\cD_0$ is almost surely bounded, we can apply Borell--TIS theorem \ref{thm:borell-tis} to $g_0$ and to $- g_0$, and obtain
\begin{align} 
    \mathbb P\left(
        \left.
        S_0^+ > m_0 + u
        \right|
        \mathcal D_0
    \right)
    & \le
    \exp\left(
        -\frac{u^2}{2\sigma_0^2}
    \right),
    \qquad \text{ for } u \ge 0.
 \label{eq:Borell-TIS-positive}
 \\
    \mathbb P\left(
        \left.
        S_0^- > m_0 + u
        \right|
        \mathcal D_0
    \right)
    & \le
    \exp\left(
        -\frac{u^2}{2\sigma_0^2}
    \right),
    \qquad \text{ for } u \ge 0,
 \label{eq:Borell-TIS-negative}
 %
\end{align}
where
\[
    \sigma_0^2
    \coloneqq
    \sup_{\*x \in \cX} \Sigma_0(\*x,\*x).
\]
Note that $\sigma_0$ is finite because we have
$0 \le \Sigma_0(\*x,\*x) \le k(\*x,\*x)$
from the posterior covariance formula,
and
$\sup_{x \in \cX} k(\*x,\*x) < \infty$
since $k$ is continuous and $\cX$ is compact.

Let
\begin{align*}
 M_0
 & \coloneqq
    \sup_{\*x \in \cX}|g_0(\*x)|
    =
    \max\{S_0^+,S_0^-\}.
\end{align*}
For every $u \ge 0$,
\[
\begin{aligned}
    &\left\{
        M_0 > m_0 + u
    \right\}
    =
    \left\{
        S_0^+ > m_0 + u
    \right\}
    \cup
    \left\{
        S_0^- > m_0 + u
    \right\}.
\end{aligned}
\]
%
%
By the union bound together with \eq{eq:Borell-TIS-positive} and \eq{eq:Borell-TIS-negative}, we see
\begin{align} 
    \mathbb P\left(
        \left.
	M_0 > m_0 + u
        \right|
        \cD_0
    \right)
    \le
    2 \exp\left(
        -\frac{u^2}{2\sigma_0^2}
    \right),
    \qquad u \ge 0.
 \label{eq:Borell-TIS-union}
\end{align}

We now use this Gaussian tail bound to establish the finiteness of the
second moment of $M_0$. For any nonnegative random variable $Z$,
\[
    \mathbb E[Z^2]
    =
    2\int_0^\infty
        t\,\mathbb P(Z>t)\,dt.
\]
Applying this identity conditionally on $\mathcal D_0$ to $M_0$ gives
\[
    \mathbb E[
        M_0^2
        \mid\mathcal D_0
    ]
    =
    2\int_0^\infty
        t\,
        \mathbb P(
            M_0>t
            \mid\mathcal D_0
        )
        \,dt.
\]
Split the integral at $m_0$:
\[
\begin{aligned}
    \mathbb E[
        M_0^2
        \mid\mathcal D_0
    ]
    &=
    2\int_0^{m_0}
        t\,
        \mathbb P(
            M_0>t
            \mid\mathcal D_0
        )
        \,dt\\
    &\quad+
    2\int_{m_0}^{\infty}
        t\,
        \mathbb P(
            M_0>t
            \mid\mathcal D_0
        )
        \,dt.
\end{aligned}
\]
The first term is bounded by
\[
    2\int_0^{m_0}t\,dt
    =
    (m_0)^2.
\]
For the second term, write $t=m_0 + u$ and use \eq{eq:Borell-TIS-union}:
\[
\begin{aligned}
    &2\int_{m_0}^{\infty}
        t\,
        \mathbb P(
            M_0>t
            \mid\mathcal D_0
        )
        \,dt\\
    &\qquad\le
    4\int_0^\infty
        (m_0+u)
        \exp\left(
            -\frac{u^2}{2\sigma_0^2}
        \right)
        \,du.
\end{aligned}
\]
Both integrals on the right-hand side are finite, since
\[
    \int_0^\infty
        \exp\left(
            -\frac{u^2}{2\sigma_0^2}
        \right)
        \,du
    <\infty
\]
and
\[
    \int_0^\infty
        u
        \exp\left(
            -\frac{u^2}{2\sigma_0^2}
        \right)
        \,du
    =
    \sigma_0^2
    <\infty.
\]
Therefore
\begin{align*}
 \EE[
 M_0^2
 \mid \cD_0
 ]
 < \infty.
\end{align*}
\end{proof}

Based on lemma~\ref{lem:second-moment}, we derive the proof for Proposition~\ref{prop:one-shot-error}
\begin{proof}

From the definition of $g_0$, 
\begin{align}
 f(\*x) = \mu_0(\*x) + g_0(\*x). 
 \label{eq:de-centering-posterior}
\end{align}
%
Under our assumptions, the posterior mean function is continuous on the compact set $\cX$. 
Consequently,
\[
    \|\mu_0\|_\infty
    \coloneqq
    \sup_{x \in \cX} |\mu_0(\*x)|
    <
    \infty.
\]
By \eq{eq:de-centering-posterior},
\[
    \sup_{\*x \in \cX} |f(\*x)|
    \le
    \|\mu_0\|_\infty + M_0.
\]
Therefore, from $(a + b)^2 \leq 2 a^2 + 2 b^2$, 
\[
    \left(
        \sup_{\*x \in \cX} |f(\*x)|
    \right)^2
    \le
    2\|\mu_0\|_\infty^2
    +
    2 M_0^2.
\]
Taking conditional expectations and using \eq{eq:second-moment-M0} gives
\[
\begin{aligned}
    C_{\mathcal D_0}
    &:=
    \mathbb E\left[
        \left.
        \left(
            \sup_{x\in\mathcal X}|f(x)|
        \right)^2
        \right|
        \mathcal D_0
    \right]\\
    &\le
    2\|\mu_0\|_\infty^2
    +
    2\mathbb E[
        M_0^2
        \mid\mathcal D_0
    ]\\
    &<
    \infty.
\end{aligned}
\]




For every realization of the posterior sample path, future observation noises, and algorithmic randomization,
\[
\begin{aligned}
    0
    \le
    f^\star-f(\*x_t)
    &\le
    |f^\star|+|f(\*x_t)|\\
    &\le
    2\sup_{\*x \in \cX}|f(\*x)|.
\end{aligned}
\]
Note that this inequality is pathwise and is independent of how $\*x_t$ was selected.
Summing over $t=1,\ldots,T$ yields
\[
    0
    \le
    R_T
    \le
    2T\sup_{\*x \in \cX}|f(\*x)|.
\]
Hence
\[
    R_T^2
    \le
    4T^2
    \left(
        \sup_{\*x \in \cX}|f(\*x)|
    \right)^2.
\]
Taking conditional expectations and using \eq{eq:second-moment-M0},
\[
    \EE[
        R_T^2
        \mid \cD_0
    ]
    \le
    4 C_{\cD_0}T^2.
\]
Therefore, we see
\[
    \operatorname{Var}
    (R_T\mid\mathcal D_0)
    \le
    4C_{\mathcal D_0}T^2.
\]




Since the MC estimator
$\widehat{\operatorname{BR}}_{T}$ 
is unbiased and each $R_T$ is independent, 
\begin{align}
\begin{aligned}
    &\mathbb E\left[
        \left.
        \left(
            \widehat{\operatorname{BR}}_{T}(\*\theta)
            -
            \operatorname{BR}_T(\*\theta)
        \right)^2
        \right|
        \mathcal D_0
    \right]\\
    &\qquad=
    \operatorname{Var}\left(
        \left.
        \widehat{\operatorname{BR}}_{T}(\*\theta)
        \right|
        \mathcal D_0
    \right)\\
    &\qquad=
    \frac{1}{N}
    \operatorname{Var}(R_T\mid\mathcal D_0)\\
    &\qquad\le
    \frac{4C_{\mathcal D_0}T^2}{N}.
\end{aligned}
\label{eq:variance-BR}
\end{align}

(Conditional) Chebyshev's inequality and \eq{eq:variance-BR} derive
\[
 \mathbb P\left(
 \left.
 \left|
 \widehat{\operatorname{BR}}_{T}(\*\theta)
 -
 \operatorname{BR}_T(\*\theta)
 \right|
 > \eta
 \right|
 \cD_0
 \right)
 \le
 \frac{4C_{\cD_0}T^2}{N\eta^2}.
\]
Taking
$\eta = 2 T \sqrt{\frac{C_{\cD_0}}{N \delta}}$
gives
\[
    \mathbb P\left(
        \left.
        \left|
            \widehat{\operatorname{BR}}_{T}(\*\theta)
            -
            \operatorname{BR}_T(\*\theta)
        \right|
        \le
        2T\sqrt{\frac{C_{\mathcal D_0}}{N\delta}}
        \right|
        \mathcal D_0
    \right)
    \ge
    1-\delta
\]
for $\delta \in (0,1)$. 
Therefore, we obtain
\[
    \widehat{\operatorname{BR}}_{T}(\*\theta)
    -
    \operatorname{BR}_T(\*\theta)
    =
    O_p\left(\frac{T}{\sqrt N}\right)
\]
conditionally on $\mathcal D_0$.


\end{proof}

\section{Bounds of Canonical Pseudo-Metric}
\label{app:bounds-canonical-metric}

Suppose first that $k$ is the Gaussian kernel
\[
    k(\*x,\*x')
    =
    \sigma_f^2
    \exp\left(
        -\frac{\|\*x-\*x'\|_2^2}{2\ell^2}
    \right),
\]
where $\ell > 0$ is the length-scale parameter and $\sigma_f > 0$ is the scale parameter.
Writing $r=\|\*x - \*x'\|_2$, we obtain
\[
\begin{aligned}
    d(\*x,\*x')^2
    &=
    2\sigma_f^2
    \left(
        1-
        \exp\left\{-\frac{r^2}{2\ell^2}\right\}
    \right)\\
    &\le
    \frac{\sigma_f^2}{\ell^2}r^2,
\end{aligned}
\]
where $1-e^{-a} \le a$ for $a \ge 0$ was used. 
Hence
\[
    d(\*x,\*x')
    \le
    \frac{\sigma_f}{\ell}\|\*x - \*x'\|_2.
\]




We define the Mat\'ern kernel by
\[
    k_\nu(\mathbf x,\mathbf x')
    =
    \sigma_f^2
    \frac{2^{1-\nu}}{\Gamma(\nu)}
    \left(
        \frac{\sqrt{2\nu}}{\ell}
        \|\mathbf x-\mathbf x'\|_2
    \right)^\nu
    K_\nu\left(
        \frac{\sqrt{2\nu}}{\ell}
        \|\mathbf x-\mathbf x'\|_2
    \right),
\]
where $\nu>0$ is the smoothness parameter, $\ell>0$ is the length-scale parameter, $\sigma_f > 0$ is the scale parameter, $\Gamma(\cdot)$ denotes the Gamma function, and $K_\nu(\cdot)$ denotes the modified Bessel function of the second kind of order $\nu$.
%
For the canonical pseudo-metric of a Mat\'ern kernel, it is known that 
\[
    \gamma(\*h)
    =
    \frac12
    \mathbb E[
        \{f(\*x + \*h) - f(\*x)\}^2
    ]
    =
    \begin{cases}
        O(\|\*h\|^{2\nu}), & 0<\nu<1,\\
        O(\|\*h\|^2\log(1/\|\*h\|)), & \nu=1,\\
        O(\|\*h\|^2), & \nu>1,
    \end{cases}
    \qquad \|\*h\|\downarrow0.
\]
See \citep{stein1999interpolation} for detail.
%
Let $r = \|\*h\|_2$.
The case of $\nu = 1$ is bounded by
\[
 O(r^2 | \log(1/r) | ) 
\]
Since any positive power dominates a logarithmic factor as 
$r \downarrow 0$, 
\[
 r^{1-\alpha}
 \sqrt{|\log r|} \rightarrow 0
\]
for any $\alpha \in (0,1)$ and $r \downarrow 0$.
This indicates for sufficiently small $r$
\[
 r \sqrt{ |\log r | } \leq r^\alpha. 
\]
Since
\[
    d(\*x,\*x')
    =
    \sqrt{
    2\gamma(\*x - \*x'),
    }
\]
it follows that, for sufficiently small $\| \*x - \*x' \|_2$ and some $L<\infty$,
\[
    d(\*x, \*x')
    \le
    L \|\*x-\*x'\|_2^\alpha,
\]
where $\alpha=\nu$ for $0<\nu<1$, any $\alpha \in (0,1)$ for $\nu=1$,
and $\alpha=1$ for $\nu>1$.

\section{Proof of Corollary~\ref{coro:optimization-error}}
\label{app:proof-optimization-error}

Define
\[
    \Delta_N
    \coloneqq
    \max_{\*\theta\in\Theta}
    \left|
        \widehat{\mr{BR}}_{T}(\*\theta)-\mr{BR}_T(\*\theta)
    \right|.
\]
Since $\Theta$ is assumed to be a finite set, we can write $\Theta=\{\*\theta_1,\ldots,\*\theta_K\}$ for the number of candidate configurations $K > 0$. 
For any $\varepsilon>0$ and each $k=1,\ldots,K$, Proposition~\ref{prop:one-shot-error} implies that there exist constants $M_k<\infty$ and $N_k$ such that, for all $N\ge N_k$,
\[
    \Pr\!\left(
        \left|
            \widehat{\mr{BR}}_{T,N}(\*\theta_k)-\mr{BR}_T(\*\theta_k)
        \right|
        >
        M_k\frac{T}{\sqrt{N}}
    \right)
    <
    \frac{\varepsilon}{K}.
\]
Let $M=\max_{1\le k\le K}M_k$ and
$N_0=\max_{1\le k\le K}N_k$. Then, by the union bound, for all
$N\ge N_0$,
\[
\begin{aligned}
    \Pr\!\left(
        \Delta_N>M\frac{T}{\sqrt{N}}
    \right)
    &\le
    \sum_{k=1}^K
    \Pr\!\left(
        \left|
            \widehat{\mr{BR}}_{T,N}(\*\theta_k)-\mr{BR}_T(\*\theta_k)
        \right|
        >
        M\frac{T}{\sqrt{N}}
    \right) \\
    &<
    \varepsilon.
\end{aligned}
\]
Therefore, we obtain $\Delta_N=O_p(T/\sqrt{N})$.

By the definition of $\widehat{\*\theta}_N$,
\[
    \widehat{\mr{BR}}_{T}(\widehat{\*\theta}_N)
    \le
    \widehat{\mr{BR}}_{T}(\*\theta^\star).
\]
Therefore,
\begin{align*}
    \mr{BR}_T(\widehat{\*\theta}_N)-\mr{BR}_T(\*\theta^\star)
    &=
    \bigl[
        \mr{BR}_T(\widehat{\*\theta}_N)
        -\widehat{\mr{BR}}_{T,N}(\widehat{\*\theta}_N)
    \bigr] \\
    &\quad+
    \bigl[
        \widehat{\mr{BR}}_{T,N}(\widehat{\*\theta}_N)
        -\widehat{\mr{BR}}_{T,N}(\*\theta^\star)
    \bigr] \\
    &\quad+
    \bigl[
        \widehat{\mr{BR}}_{T,N}(\*\theta^\star)
        -\mr{BR}_T(\*\theta^\star)
    \bigr] \\
    &\le
    \left|
        \mr{BR}_T(\widehat{\*\theta}_N)
        -\widehat{\mr{BR}}_{T,N}(\widehat{\*\theta}_N)
    \right|
    +
    \left|
        \widehat{\mr{BR}}_{T,N}(\*\theta^\star)
        -\mr{BR}_T(\*\theta^\star)
    \right| \\
    &\le
    2\Delta_N.
\end{align*}
Since $\Delta_N=O_p(T/\sqrt{N})$, it follows that
\[
    \mr{BR}_T(\widehat{\*\theta}_N)-\mr{BR}_T(\*\theta^\star)
    =
    O_p\!\left(\frac{T}{\sqrt{N}}\right),
\]
which proves the result.
